\documentclass{article}

\usepackage[main, nonatbib, preprint]{style} 
\usepackage{preamble}
\usepackage[disable]{todonotes}

\def\debug{1}
\newcommand{\ASnotes}[1]{\ifnum\debug=1{\color{red} [AS: #1]}\fi}
\newcommand{\PELnotes}[1]{\ifnum\debug=1{\color{green} [PEL: #1]}\fi}
\newcommand{\WYnotes}[1]{\ifnum\debug=1{\color{blue} [WY: #1]}\fi}

\title{How are linear representations learned? Exact solutions to the dynamics of abstraction}

\author{%
  William W. Yang$^1$
  \And Andrew M. Saxe$^{*, 1, 2}$ \\
  \And Peter E. Latham$^{*, 1}$
}

\begin{document}
\etocdepthtag.toc{main}

\maketitle

\begingroup
\renewcommand{\thefootnote}{*}
\footnotetext{Co-senior authors; equal contribution.}
\endgroup
\footnotetext[1]{Gatsby Computational Neuroscience Unit, University College London, London W1T 4JG, United Kingdom.}
\footnotetext[2]{Sainsbury Wellcome Centre, University College London, London W1T 4JG, United Kingdom.}

\begin{abstract}
  In artificial and biological neural networks, concepts are often encoded as consistent linear directions in representation space. In deep learning, this idea is known as the linear representation hypothesis and underpins many interpretability and control methods based on linear probes, from concept detection to activation steering. Yet while prior work has studied whether such directions should exist \textit{after} training, the dynamics of how they emerge \textit{during} training remain poorly understood. Here, we develop a framework to study the alignment of concept directions during training -- a process we call ``abstraction''. In a minimal linear network setting, we obtain exact solutions for the full trajectory of abstraction. These solutions reveal key analytic principles governing abstraction: (i) data and target geometry jointly determine abstraction at the end-of-learning, (ii) abstraction improves with network depth, and (iii) initialization scale controls the maximum abstraction reached during training. Extending our theory to nonlinear networks, we analyze how the choice of nonlinearity affects abstraction dynamics: erf networks approximate the linear theory, while abstraction in ReLU networks depends less on target geometry and more on input geometry. Across both, we prove a striking attenuation law: both nonlinearities weaken abstraction in activations relative to preactivations. We find evidence for this law in open models (DINOv3, Gemma 4) and apply our theory to improve linear probe generalization in LLMs. Together, our results provide a dynamical theory of abstraction with implications for interpretability and control.
\end{abstract}

\section{Introduction}

A recurring empirical observation in deep neural networks is that high-level concepts often behave like approximately linear directions in representation space \cite{mikolovLinguisticRegularities13, nandaEmergentLinear23}. This idea has recently been formalized as the linear representation hypothesis (LRH) \cite{parkLinearRepresentation24,jiangOriginsLinear24,parkGeometryCategorical25}. To use a classical example, a linear representation of the ``gender'' concept would imply that the concept vectors \(\mathbf{v}_\mathrm{king} - \mathbf{v}_\mathrm{queen} \) and \(\mathbf{v}_\mathrm{man} - \mathbf{v}_\mathrm{woman} \) are approximately parallel. A closely related idea exists in neuroscience, where a concept is said to be represented in an ``abstract'' format when concept vectors are highly aligned across contexts \cite{bernardiGeometryAbstraction20,nogueiraGeometryCortical23,shinProtocolGeometric23,fascianelliNeuralRepresentational24a,courellisAbstractRepresentations24,boyleTunedGeometries24,mishchanchukHiddenState24,oneillRepresentationalGeometry24}. Borrowing that neuroscience terminology, we use the term ``abstraction'' to refer to the alignment of concept vectors (e.g. the cosine similarity between \(\mathbf{v}_\mathrm{king} - \mathbf{v}_\mathrm{queen} \) and \(\mathbf{v}_\mathrm{man} - \mathbf{v}_\mathrm{woman} \)) during training.

Understanding how abstract representations arise is important in both deep learning and neuroscience. Many AI interpretability and control methods assume existence of abstract representations by using linear probes to extract concept directions for interpretability and detection \cite{gurneeFindingNeurons23, goldowsky-dillDetectingStrategic25} or model steering \cite{turnerSteeringLanguage24, liInferenceTimeIntervention24}. In neuroscience, abstraction is thought to support the brain's ability to adapt to changing environments via out-of-distribution and compositional generalization \cite{bernardiGeometryAbstraction20, courellisAbstractRepresentations24, itoCompositionalGeneralization22}.


However, our understanding of how abstraction arises remains limited. Previous theoretical work has argued that abstraction becomes perfect (i.e. cosine similarity = 1) once trained to convergence \cite{jiangOriginsLinear24}, or at the global minima of the loss landscape \cite{wangMathematicalTheory26}. This leaves open a simple question: what are the dynamics of abstraction \textit{during} training? Empirically, in real-world settings abstraction often exhibits nontrivial dynamics, can be non-monotonic, and seldom reach perfect abstraction as predicted in simple theories; for example, see \autoref{fig:2fs-ansatz}C (also Figure 4C of \cite{allemanTaskStructure24}). This motivates the central questions of this paper: \textit{What is the trajectory of abstraction during training? How is it affected by properties of the dataset? What is the role of depth and nonlinearity?} In this paper we address these in an analytically solvable setting to understand the learning dynamics of abstraction.

   
\begin{figure}[h]
    \centering
    \includegraphics[width=1.0\linewidth]{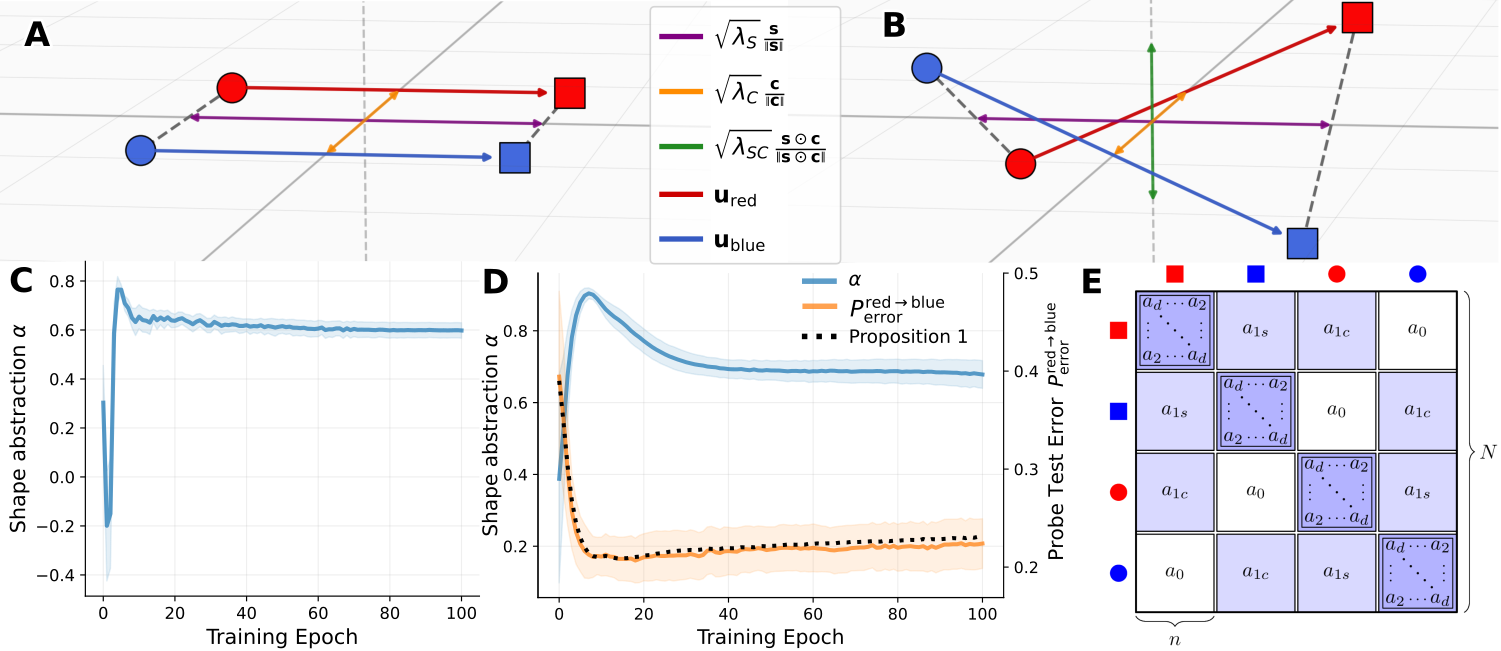}
    \caption{\textbf{Problem setting (see \autoref{section:min-model} for more details)}. We consider a minimal solvable setting with two binary concepts (shape, color) giving us four classes $\mathcal{C} \coloneqq \{\redsquare, \bluesquare, \redcircle, \bluecircle\}$. \textbf{(A)} Schematic of perfectly abstract representation where shape concept vectors $\mathbf{u}_\mathrm{red}$ and $\mathbf{u}_\mathrm{blue}$ have cosine similarity $\alpha_\mathrm{shape} = 1$. \textbf{(B)} In the minimal model, introducing an orthogonal interaction mode ($\mathrm{SC}$, green) with eigenvalue $\lambda_{SC}$ lowers the abstraction $\alpha_\mathrm{shape}$, hence we refer to it as the ``noise'' direction. \textbf{(C)} The abstraction trajectory of a deep ConvNet trained on the 3dshapes dataset \cite{kimDisentanglingFactorising18} exhibits nontrivial dynamics, which is not captured by prior theories. In particular, its trajectory is non-monotonic and it plateaus well below $\alpha = 1$ (see \autoref{appendix:feedforward-convnet} for experiment details). In this paper, our aim is a dynamical theory of abstraction which captures such behavior. \textbf{(D)} Abstraction and probe generalization of deep ReLU network trained on synthetic data (see \autoref{appendix:synthetic-details} for experimental details). \(P_\mathrm{error}^{\mathrm{red}\to\mathrm{blue}}\) (orange) is the generalization error of a shape classifier trained only on red samples and then tested on blue samples. Here, we observe \(P_\mathrm{error}^{\mathrm{red}\to\mathrm{blue}}\) trends with abstraction (blue) and is well-approximated by \autoref{proposition:ccgp} (black dotted), which illustrates how abstraction influences probe generalization error. \textbf{(E)} In our framework we assume the dataset kernels obey a block structure called ``2FS'' (see \autoref{assumption:2fs}). Any 2FS kernel $\mathbf{A}$ is described by only five entry types: $a_d$, $a_2$, $a_{1s}$, $a_{1c}$, $a_0$.}
    \label{fig:2fs-ansatz}
\end{figure}


\paragraph{Key contributions.} In \autoref{section:min-model}-\autoref{section:exact-solutions} we study a minimal linear network setting to obtain exact solutions for the training-time trajectory of abstraction and then use these to derive laws for the final and maximum abstraction reached during training (\autoref{section:fixedpoints-overshoot}). We then extend the theory to arbitrarily deep networks and derive laws for the effect of depth on abstraction (\autoref{section:depth}). In \autoref{section:nonlinear} we extend our theory to infinite-width nonlinear networks, proving that nonlinearities weaken abstraction in features compared to preactivations (\autoref{section:attenuation}) and show that ReLU reshapes dynamics to depend more on inputs than targets (\autoref{section:nonlinear-dynamics}). In \autoref{section:applications}, we apply our theory to deep learning and neuroscience respectively, by showing that ablating nonlinearities improves abstraction and linear probe generalization in DINOv3 and Gemma 4 (\autoref{fig:language-ablation}), and that abstraction improves along the macaque ventral-stream hierarchy (\autoref{fig:neuro}).

\paragraph{Related work.} \autoref{appendix:related-work} contains an extended list of related work; we particularly highlight three areas:

\textit{Abstract representations.} Recent work showed that both high-level brain regions and neural networks develop a representational geometry supporting simple forms of out-of-distribution generalization, often termed ``abstract'' \cite{johnstonAbstractRepresentations23, bernardiGeometryAbstraction20, nogueiraGeometryCortical23, fascianelliNeuralRepresentational24a, courellisAbstractRepresentations24, oneillRepresentationalGeometry24, zhuGeometricFoundation26, itoCompositionalGeneralization22}. This geometry is quantified by \textit{parallelism} \cite{bernardiGeometryAbstraction20}--the cosine similarity between counterfactual concept vectors--which is equivalent to our abstraction score. Recently, \cite{wangMathematicalTheory26} proved that at the global optimum of the loss, networks exhibit perfect abstraction in their last layer. The authors in \cite{wangMathematicalTheory26} considered an arbitrary number of concepts, and their theory's predictions remained unchanged by this number. Motivated by this, in this paper, we develop a theory for the case of two concepts that treats abstraction as a dynamical variable during training, extends to more general input-target geometries, and study its evolution across layers in features and preactivations. In addition, recent empirical work on the choice of nonlinearity showed that the sigmoid-like \texttt{tanh} nonlinearity often results in more abstract representations than using \texttt{ReLU} \cite{allemanTaskStructure24}. In \autoref{section:nonlinear-dynamics}, we provide a theoretical explanation for this phenomenon.

\textit{Linear networks.} We build on a line of previous work on linear network learning dynamics (e.g. \cite{saxeExactSolutions14, lampinenAnalyticTheory19a, kuninGetRich24}), which have been extensively studied as analytically solvable yet insightful settings \cite{namPositionSolve25a, simonThereWill26}. Like \cite{braunExactLearning22, domineLazyRich25} we adopt a matrix Riccati approach to studying learning dynamics, but whereas they study the dynamics of the block matrix 
$\tiny \begin{bmatrix}
  \mathbf{W}^\top \mathbf{W} & \mathbf{W}^\top \mathbf{W}_r^\top \\
  \mathbf{W}_r \mathbf{W} & \mathbf{W}_r \mathbf{W}_r^\top
\end{bmatrix}$ in a two-layer linear network, we employ a reduction on the readout $\mathbf{W}_r$ (see \autoref{assumption:readout}) which enables us to solve the dynamics of the hidden representation kernel in the two-layer case, and also extend to networks of arbitrary depth. 


\textit{Linear representation hypothesis.} Recent work \cite{parkLinearRepresentation24, jiangOriginsLinear24, parkGeometryCategorical25} formalized the definition of linearity for a binary concept as when the representation difference induced by changing that concept lies in a positive cone. In \autoref{appendix:lrh-parallelism}, we show that our measure of abstraction is a continuous relaxation of this condition. \cite{jiangOriginsLinear24} is closest in spirit to our work, as they prove that gradient descent asymptotically results in perfect abstraction. In contrast, we study the dynamics of abstraction during training (rather than just end-of-training), and also the evolution of abstraction across layers and through nonlinearities. 

\section{Problem setting} \label{section:min-model}

In this section, we propose a minimal linear network model which is simple enough to admit exact solutions for dynamics, yet complex enough to exhibit the phenomenon of abstraction. Consider a two-layer linear network which takes a vector input $\x \in \mathbb{R}^{D_x}$ and maps to a vector output $\y \in \mathbb{R}^{D_y}$. The network is parameterized by weight matrices $\mathbf{W} \in \mathbb{R}^{D \times D_x}$ and $\mathbf{W}_r \in \mathbb{R}^{D_y \times D}$, where $D$ is the dimension of the features. The network is trained on the dataset $\mathcal{D} = \{(\x^{(i)}, \y^{(i)})\}_{i=1}^N$ (where $N$ is the number of examples) to minimize MSE loss with ridge penalty $\gamma^{-1}$ on the readout: 
\begin{align} \label{eq:loss}
    \Lc(\mathbf{W}_r, \mathbf{Z}) &= \frac{1}{2} \| \mathbf{W}_r \mathbf{Z} - \mathbf{Y} \|_F^2 + \frac{1}{2 \gamma} \|\mathbf{W}_r\|_F^2,
\end{align}
where we have denoted the input, target and feature matrices as $\mathbf{X} \in \mathbb{R}^{D_x \times N}$, $\mathbf{Y} \in \mathbb{R}^{D_y \times N}$, $\mathbf{Z} \in \mathbb{R}^{D \times N}$ respectively, and $\mathbf{Z} = \mathbf{W} \mathbf{X}$. Following \cite{parkLinearRepresentation24,korchinskiEmergenceLinear25,wangMathematicalTheory26}, we model ``concepts'' as binary latent variables, and consider the case of two binary latents $(s^{(i)}, c^{(i)}) \in \{+1, -1\}^2$. As mnemonics, we refer to these latents and their values as ``\textbf{s}hape'' ($+1$ for square / $-1$ for circle) and ``\textbf{c}olor'' ($+1$ for red / $-1$ for blue). We assume the conditional mean of both $\x^{(i)}$ and $\y^{(i)}$ are given by different \emph{nonlinear} maps of the latents $(s^{(i)}, c^{(i)})$. We can write the generative process as:
\begin{align}
    \x^{(i)} &= f(s^{(i)}, c^{(i)}) + \boldsymbol{\epsilon}_x^{(i)} \quad \text{for} \quad f: \{-1, +1\}^2 \mapsto \mathbb{R}^{D_x}, \label{eq:input-model}\\
    \y^{(i)} &= g(s^{(i)}, c^{(i)}) + \boldsymbol{\epsilon}_y^{(i)} \quad \text{for} \quad g: \{-1, +1\}^2 \mapsto \mathbb{R}^{D_y}, \label{eq:target-model}
\end{align}
and $\boldsymbol{\epsilon}_x^{(i)}$ and $\boldsymbol{\epsilon}_y^{(i)}$ are iid zero-mean noise. This results in four classes $\mathcal{C} \coloneqq \{\redsquare, \bluesquare, \redcircle, \bluecircle\}$ (see \autoref{fig:2fs-ansatz}A) with $n$ samples each, giving us the index set \( \mathcal{I} = \{\pm 1\} \times \{\pm 1\} \times [n] \) and $N = 4n$ total samples. For simplicity, the inputs and targets are partitioned in class order, e.g. $\mathbf{Y} = [\mathbf{Y}_{\redsquare}, \mathbf{Y}_{\bluesquare}, \mathbf{Y}_{\redcircle}, \mathbf{Y}_{\bluecircle}]$. 

\subsection{Abstraction score}

Now we define our quantity of interest, the abstraction score. Let $\smash{\boldsymbol{\mu}_c \coloneqq \frac{1}{|\mathcal{I}_c|} \sum_{i \in \mathcal{I}_c} \z^{(i)}}$ be the mean representation of class $c \in \mathcal{C}$. The shape direction can be measured in two contexts: $\mathbf{u}_\mathrm{red} \coloneqq \boldsymbol{\mu}_{\redsquare} - \boldsymbol{\mu}_{\redcircle}$ and $\mathbf{u}_\mathrm{blue} \coloneqq \boldsymbol{\mu}_{\bluesquare} - \boldsymbol{\mu}_{\bluecircle}$. We define the abstraction of the shape concept as the cosine similarity $\alpha_\mathrm{shape}
  \coloneqq
  \frac{\mathbf{u}_\mathrm{red} \cdot \mathbf{u}_\mathrm{blue}}
  {\|\mathbf{u}_\mathrm{red}\| \|\mathbf{u}_\mathrm{blue}\|}$ (see \autoref{fig:2fs-ansatz}A). 
We can similarly define $\alpha_\mathrm{color}$ by symmetry, but for brevity this paper will focus only on shape abstraction and so hereafter we will drop the subscript, i.e. $\alpha \equiv \alpha_\mathrm{shape}$. The case $\alpha = 1$ is equivalent to the formal definition of linearity in linear representation hypothesis literature \cite{parkLinearRepresentation24, jiangOriginsLinear24} (see \autoref{appendix:lrh-parallelism}). An important motivation for studying abstraction is that it influences the generalization of linear probes. Under mild assumptions, changes in probe generalization error are highly predictable from $\alpha$ (see \autoref{fig:2fs-ansatz}D, orange and black). We can state the following proposition:

\begin{proposition}[Abstraction controls probe transfer] \label{proposition:ccgp}
Let $\widehat{\mathbf{w}} \coloneqq \mathbf{u}_\mathrm{red} / \|\mathbf{u}_\mathrm{red}\|$ be the probe which classifies shape from the red samples and let $\Phi$ be the normal CDF. Transfer it to a blue test sample $\mathbf{z}_\mathrm{blue}$, centered as $\widetilde{\mathbf{z}}_\mathrm{blue} \coloneqq \mathbf{z}_\mathrm{blue} - \smash{\frac{1}{2}} (\boldsymbol{\mu}_{\bluesquare} + \boldsymbol{\mu}_{\bluecircle})$.
If the projected blue scores $\widehat{\mathbf{w}}^\top \widetilde{\mathbf{z}}_\mathrm{blue}$ are approximately Gaussian with variance 
$
\sigma_\mathrm{blue}^{2}
    =
    \widehat{\mathbf{w}}^\top
    \boldsymbol{\Sigma}_\mathrm{blue}
    \widehat{\mathbf{w}},
$
then the test error is approximately:
\begin{align}
P_{\mathrm{error}}^{\mathrm{red}\to\mathrm{blue}}
    \approx
    \Phi
    \left(
        -
        \frac{
            \|\mathbf{u}_\mathrm{blue}\|
        }{
            2\sigma_\mathrm{blue}
        } \cdot \alpha
    \right),
    \quad 
    \text{and in particular,} ~~ 
    \frac{
        \partial P_{\mathrm{error}}^{\mathrm{red}\to\mathrm{blue}}
    }{
        \partial \alpha
    }
    <
    0.
\end{align}
See \autoref{appendix:probe-generalization} for proof and discussion. In neuroscience, this form of generalization is sometimes referred to as the cross-condition generalization performance (CCGP) \cite{bernardiGeometryAbstraction20}. 
\end{proposition}

\subsection{Assumptions}

Our aim is to study $\alpha(t)$ during gradient flow. At every learning time $t$, $\alpha$ is determined by the kernel $\mathbf{Q}(t) \coloneqq \mathbf{Z}(t)^\top \mathbf{Z}(t)$, which depends on the input and target kernels $\boldsymbol{\Sigma}_x \coloneqq \mathbf{X}^\top \mathbf{X}$ and $\boldsymbol{\Sigma}_y \coloneqq \mathbf{Y}^\top \mathbf{Y}$. Generically, the gradient flow dynamics of $\mathbf{Q}$ are intractable without further assumptions, so to obtain exact solutions, we make two key simplifying assumptions which we briefly discuss in this section. 

The first assumption concerns the readout $\mathbf{W}_r$. Note that in \autoref{eq:loss}, we deliberately regularize just $\mathbf{W}_r$ (not $\mathbf{W}$) since since our quantity of interest $\alpha$ is invariant to the scale of $\mathbf{W}$ (notice in \autoref{fig:2fs-ansatz}A, the angle between the \red{red} and \blue{blue} vectors is invariant to their norms). We assume the following:

\begin{assumption}[Variable-Projected Readout] \label{assumption:readout}
We assume the readout is always at its optimum,
\begin{align*}
    \mathbf{W}_r(t) = \mathbf{W}_r^*(\mathbf{Z}(t)) = \gamma \mathbf{Y}(\mathbf{I}_N + \gamma \mathbf{Q}(t))^{-1} \mathbf{Z}(t)^\top, \forall t
\end{align*}
\end{assumption}

Using \autoref{assumption:readout}, we obtain a simplified form of \autoref{eq:loss}, $\smash{\mathcal{L}^* (\mathbf{Z})} = \smash{\frac{1}{2} \mathrm{Tr} \pbracket{ \mathbf{Y}^\top \mathbf{Y} (\mathbf{I}_N + \gamma \mathbf{Q}(t))^{-1}}}$. We argue this is more realistic than freezing \(\mathbf{W}_r\), which is commonly assumed in prior work on learning dynamics, e.g. \cite{saadOnlineLearning95a,aroraFineGrainedAnalysis19}. This corresponds to the limit in which the learning rate for \(\mathbf{W}_r\) is much larger than for \(\mathbf{W}\) \cite{barboniUltrafastFeature25,marionLeveragingTwo23}. In \autoref{appendix:vp-readout} we do a full analysis of \autoref{assumption:readout}, where we show that it still captures the key qualitative behavior of the naturalistic case where \(\mathbf{W}_r\) has a finite learning rate.

The second assumption concerns the structure of $\boldsymbol{\Sigma}_x$ and $\boldsymbol{\Sigma}_y$. By \autoref{proposition:walsh-expansion}, we know that any function of the latents $(s^{(i)}, c^{(i)})$ admits a unique Fourier expansion in the basis $\{1, s^{(i)}, c^{(i)}, s^{(i)}c^{(i)}\}$. Because $g$ in \autoref{eq:target-model} is a function of $(s^{(i)}, c^{(i)})$, we can expand $\mathbf{y}^{(i)}$ as follows:
\begin{align}
  \mathbf{y}^{(i)} 
  = g(s^{(i)}, c^{(i)}) + \boldsymbol{\epsilon}_y^{(i)}
  = \hat{\mathbf{g}}_\varnothing + s^{(i)} \hat{\mathbf{g}}_S + c^{(i)} \hat{\mathbf{g}}_C + s^{(i)} c^{(i)} \hat{\mathbf{g}}_{SC} + \boldsymbol{\epsilon}_y^{(i)}
\end{align}

If we impose pairwise orthogonality between the Fourier components, $\hat{\mathbf{g}}_S \perp \hat{\mathbf{g}}_C \perp \hat{\mathbf{g}}_{SC} \perp \hat{\mathbf{g}}_{\varnothing}$, then the target kernel $\boldsymbol{\Sigma}_y$ diagonalizes in a natural, interpretable basis for studying abstraction. By similar arguments, we assume the same structure holds for $\boldsymbol{\Sigma}_x$ and $\mathbf{Q}$. We refer the reader to \autoref{appendix:2fs-walsh-hadamard} for a full discussion on this assumption, which we formalize in terms of a symmetry group as follows: 

\begin{assumption}[Two-Factor Symmetry (2FS)] \label{assumption:2fs}
Define the symmetry group \(\mathcal{G} \cong (S_n)^4 \rtimes (\mathbb{Z}_2)^2\) which acts on the index set $\mathcal{I} = \{\pm 1\} \times \{\pm 1\} \times [n]$ by permutations of samples within each class together with global relabelings $s \mapsto -s$ and $c \mapsto -c$. Let \( \boldsymbol{\Pi}_g \) be the permutation matrix representing the permutation of \(\mathcal{I}\) induced by \(g \in \mathcal{G}\). We assume that $\boldsymbol{\Sigma}_x, \boldsymbol{\Sigma}_y, \mathbf{Q}(0)$ are invariant under this action:
\begin{align}
  \boldsymbol{\Pi}_g \mathbf{A} \boldsymbol{\Pi}_g^\top = \mathbf{A} \quad \forall g \in \mathcal{G}, \quad \mathbf{A} \in \{ \boldsymbol{\Sigma}_x, \boldsymbol{\Sigma}_y, \mathbf{Q}(0) \}
\end{align}
Any 2FS kernel $\mathbf{A}$ is characterized by a five-entry block structure shown in \autoref{fig:2fs-ansatz}E and admits the eigenmode decomposition $\mathbf{A} = \sum_m \lambda_m^{\smash{(\mathbf{A})}} \mathbf{P}_m$ where $\mathbf{P}_m$ is the orthogonal projector onto mode \(m \in \{I, S, C, SC, G\}\) and $\lambda_m^{\smash{(\mathbf{A})}}$ is the corresponding eigenvalue. $G, S, C, SC$ denote the global mean, shape, color and shape-color interaction modes, and $I$ captures residual within-class variation. 
\end{assumption}

2FS is a minimal symmetry class which preserves the distinction between abstract vs non-abstract representations. Importantly, \autoref{assumption:2fs} affords us simultaneous diagonalization of the kernels in a natural interpretable basis, reducing the dynamics of $\mathbf{Q}$ to scalar ODEs for each eigenmode. 

\paragraph{Signal vs. noise interpretation.} The two modes relevant for shape abstraction are \(S\) and \(SC\). As illustrated in \autoref{fig:2fs-ansatz}A-B, \(S\) controls separability of square vs circle, while \(SC\) controls separability of the XOR dichotomy \(\{\redsquare, \bluecircle\}\) vs \(\{\redcircle, \bluesquare\}\). As $\alpha$ can only be reduced via \(SC\), we refer to it as the ``noise'' mode, and \(S\) as the ``signal'' mode. Motivated by this, for any 2FS kernel $\mathbf{A}$ we can write the abstraction induced by it, $\alpha_\mathbf{A}$, in terms of an ``inverse SNR'' (signal-to-noise ratio) $\smash{\nu(\mathbf{A})}$:
\begin{align}
  \alpha_{\mathbf{A}} = \frac{1 - \nu(\mathbf{A})}{1 + \nu(\mathbf{A})}, \quad \text{where we define} ~~
  \nu(\mathbf{A}) \coloneqq \lambda_{SC}^{(\mathbf{A})} / {\lambda_S^{(\mathbf{A})}} \label{eq:alpha-from-nu}
\end{align}

\section{Solutions to the minimal model} \label{section:exact-solutions}

In \autoref{appendix:reduced-kernel-dynamics}, we derive the gradient flow equation $\dot{\mathbf{Z}} = \mathbf{Z} \mathbf{M}(\mathbf{Q}) \boldsymbol{\Sigma}_x$ and the Riccati equation $\dot{\mathbf{Q}} = \boldsymbol{\Sigma}_x \mathbf{M} \mathbf{Q} + \mathbf{Q} \mathbf{M} \boldsymbol{\Sigma}_x$, where we define $\mathbf{M}(\mathbf{Q}) \coloneqq \gamma (\mathbf{I}_N + \gamma \mathbf{Q})^{-1} \boldsymbol{\Sigma}_y (\mathbf{I}_N + \gamma \mathbf{Q})^{-1}$ as the \textit{effective target kernel}. The targets influence dynamics through $\mathbf{M}$, which is a spectrally filtered version of $\boldsymbol{\Sigma}_y$. In this section, for the abstraction induced by $\mathbf{Q}$, we suppress the subscript such that $\alpha \equiv \alpha_{\mathbf{Q}}$. 

\begin{proposition}[Linear network dynamics] \label{proposition:linear-odes}
Let $\lambda_m$ be the eigenvalue of $\mathbf{Q}$ for mode $m \in \{I, S, C, SC, G\}$. Then $\lambda_m$ and $\alpha$ evolve according to the following respective ODEs:
\begin{align} \label{eq:eigenvalue-odes}
  \dot \lambda_m = 2 \gamma \lambda_m \lambda_m^{(\boldsymbol{\Sigma}_x)} \frac{\lambda_m^{(\boldsymbol{\Sigma}_y)}}{(1 + \gamma \lambda_m)^2}, \qquad 
  \dot \alpha = (1 - \alpha^2) \lambda_S^{(\boldsymbol{\Sigma}_x)} \lambda_S^{(\mathbf{M})} \pbracket{1 - \nu(\boldsymbol{\Sigma}_x) \nu(\mathbf{M})} 
\end{align}
\end{proposition}

We can then solve \autoref{proposition:linear-odes} to obtain scalar solutions for each eigenvalue (\autoref{theorem:implicit-solution-eigenvalue}), which in turn gives us an exact solution for the trajectory of abstraction as follows (see \autoref{appendix:implicit-solutions} for proofs):

\begin{theorem}[Exact implicit solution] \label{theorem:implicit-solution}
Define the ridge-normalized eigenvalues $\smash{\tilde{\lambda}_m(t)} \coloneqq \gamma \lambda_m(t)$. Let $F(a) \coloneqq \frac{1}{2} a^2 + 2 a + \log a$ for $a > 0$ and $\bar{\nu} \coloneqq \nu(\boldsymbol{\Sigma}_x)\nu(\boldsymbol{\Sigma}_y)$. Then we have the following exact implicit solution for $\alpha(t)$ in parametric form using $\smash{\tilde{\lambda}_S}$ as a monotone clock variable:
\begin{align}
  F \pbracket{\tilde{\lambda}_S \, \frac{1 - \alpha}{1 + \alpha}} - \bar{\nu} F (\tilde{\lambda}_S) = F (\tilde{\lambda}_{SC,0}) - \bar{\nu} F (\tilde{\lambda}_{S,0}), \quad 
  t = \frac{F (\tilde{\lambda}_S) - F (\tilde{\lambda}_{S,0})}{2 \gamma \lambda_S^{(\boldsymbol{\Sigma}_x)}\lambda_S^{(\boldsymbol{\Sigma}_y)}}
\end{align}
\end{theorem}

Using \autoref{theorem:implicit-solution}, we derive analytic laws governing the trajectory of $\alpha(t)$ in \autoref{section:fixedpoints-overshoot}-\autoref{section:depth}.

\subsection{What determines the endpoint and trajectory of abstraction?} \label{section:fixedpoints-overshoot}

\begin{theorem}[Terminal abstraction law] \label{thm:fixed-point}
Define the terminal abstraction $\alpha_{\infty} \coloneqq \displaystyle \lim_{t \to \infty} \alpha(t)$. For any non-zero initialization $\lambda_{S,0}, \lambda_{SC,0}>0$, the terminal abstraction follows the simple law:
\begin{align}
  \alpha_{\infty} = \frac{1 - \sqrt{\nu(\boldsymbol{\Sigma}_x)\nu(\boldsymbol{\Sigma}_y)}}{1 + \sqrt{\nu(\boldsymbol{\Sigma}_x)\nu(\boldsymbol{\Sigma}_y)}}
\end{align}
\end{theorem}

The significance of this is that $\alpha_\infty$ is determined solely by the geometric mean of the input and target inverse-SNRs. In particular, terminal abstraction will be perfect when \textit{either} the inputs \textit{or} targets are noiseless. In the usual supervised classification setting, using centered binary (\{0, 1\} or $\pm 1$) multi-output labels to predict the individual latents is equivalent to choosing $\nu(\boldsymbol{\Sigma}_y) = 0$, giving $\alpha_\infty = 1$. Thus \autoref{thm:fixed-point} recovers previous results for perfect abstraction in multi-class classification \cite{wangMathematicalTheory26} while allowing more general input and output statistics. To understand the \textit{trajectory} rather than just the endpoint, we focus on a common setting of the dataset geometry:

\begin{setting}[Signal-dominant dataset geometry] \label{setting:signal-dominated}
Denote $\smash{\alpha_{\boldsymbol{\Sigma}_x}}$ and $\smash{\alpha_{\boldsymbol{\Sigma}_y}}$ as the abstraction levels of the raw inputs and targets respectively. We say that a dataset is in the signal-dominant regime if $\nu(\boldsymbol{\Sigma}_x) \nu(\boldsymbol{\Sigma}_y) < 1$. Equivalently, $\lambda_S^{(\boldsymbol{\Sigma}_x)} \lambda_S^{(\boldsymbol{\Sigma}_y)} > \lambda_{SC}^{(\boldsymbol{\Sigma}_x)} \lambda_{SC}^{(\boldsymbol{\Sigma}_y)} \Longleftrightarrow \alpha_{\boldsymbol{\Sigma}_x} + \alpha_{\boldsymbol{\Sigma}_y} > 0$. 
\end{setting}

\autoref{setting:signal-dominated} is satisfied when the sum of the abstraction of the raw inputs and targets is positive. This is a mild condition in typical supervised learning settings where targets are handcrafted to be low-noise. However, for self-supervised objectives such as next token prediction, $\nu(\boldsymbol{\Sigma}_y)$ can be non-negligible because the target vector is sampled from a distribution. More generally, noisy supervision like next token prediction can still be in this regime if the input's SNR is sufficiently high.

In \autoref{appendix:overshoot-initialization} we show that in \autoref{setting:signal-dominated}, $\alpha(t)$ overshoots $\alpha_\infty$. Unlike terminal abstraction, the \textit{maximum abstraction} reached during training depends on the initialization scale of weights. In \autoref{appendix:initialization-scale} we analyze how initialization affects the maximum abstraction and link it to rich and lazy learning \cite{domineLazyRich25, kuninGetRich24}. 

\begin{theorem}[Initialization scale controls maximum abstraction] \label{theorem:initialization-overshoot}
Assume \autoref{setting:signal-dominated} and $\alpha_{0} < \alpha_{\infty}$. Define the maximum abstraction $\smash{\alpha_\mathrm{max}(\kappa) \coloneqq \sup_{t > 0} \alpha(t)}$. Fix $\nu_{\mathbf{Q}, 0} \coloneqq \lambda_{SC,0} / \lambda_{S,0}$ and consider the family of initializations $\tilde{\lambda}_{S}(0) = \kappa$, $\tilde{\lambda}_{SC,0} = \kappa \, \nu_{\mathbf{Q},0}$ for $\kappa > 0$. Then $\alpha_\mathrm{max}(\kappa)$ is \textbf{strictly decreasing in the initialization scale $\kappa$} and in particular, $\lim_{\kappa \downarrow 0} \alpha_\mathrm{max}(\kappa) = 1$ and $\lim_{\kappa \to \infty} \alpha_\mathrm{max}(\kappa) = \alpha_{\infty}$.
\end{theorem}

Although the dataset controls terminal abstraction, the initialization scale controls the maximum abstraction reached over its trajectory (see \autoref{fig:linear-theory}A). In the rich limit, the model comes arbitrarily close to perfect abstraction even if the terminal value is low. In contrast, in the lazy limit, abstraction never exceeds its terminal value. How long does abstraction stay perfect in the rich limit? Strikingly, the time spent near perfect abstraction diverges as $\kappa \downarrow 0$. In \autoref{appendix:rich-limit-residence} we prove the following result:

\begin{proposition}[Rich-limit metastability] 
\label{proposition:rich-limit-metastability}
Assume \autoref{setting:signal-dominated} and $\alpha_0 < \alpha_\infty$. Consider the initialization family from \autoref{theorem:initialization-overshoot}. Fix any near-perfect abstraction threshold $1-\varepsilon$ strictly above the terminal value, i.e. $0<\varepsilon<1-\alpha_\infty$. Let $T_\varepsilon^{\mathrm{in}}(\kappa)$ be the first time at which $\alpha(t;\kappa)\ge 1-\varepsilon$, and let $R_\varepsilon(\kappa)$ be the total time spent in this band. Then, as $\kappa\downarrow 0$, we have:
\begin{align}
T_\varepsilon^{\mathrm{in}}(\kappa)=O(1),
\quad
R_\varepsilon(\kappa)=\Theta\left(\log\frac{1}{\kappa}\right), \quad \text{and in particular,} ~~ R_\varepsilon(\kappa)\to\infty
\end{align}
\end{proposition}

Perfect abstraction is \textit{\textbf{metastable}} in the rich limit since it lasts for a parametrically long time. In other words, in the rich limit \textit{\textbf{perfect abstraction will look effectively permanent on any fixed training horizon}}, even if the terminal value (\autoref{thm:fixed-point}) is lower. Another way to view this paradox is to see that the following limits do not commute: $\smash{\displaystyle \lim_{\kappa \downarrow 0}\lim_{t \to \infty}\alpha(t;\kappa) = \alpha_{\infty}}$, whereas $\smash{\displaystyle \lim_{t \to \infty}\lim_{\kappa \downarrow 0}\alpha(t;\kappa) = 1}$. 

\begin{figure}
    \centering
    \includegraphics[width=1.0\linewidth]{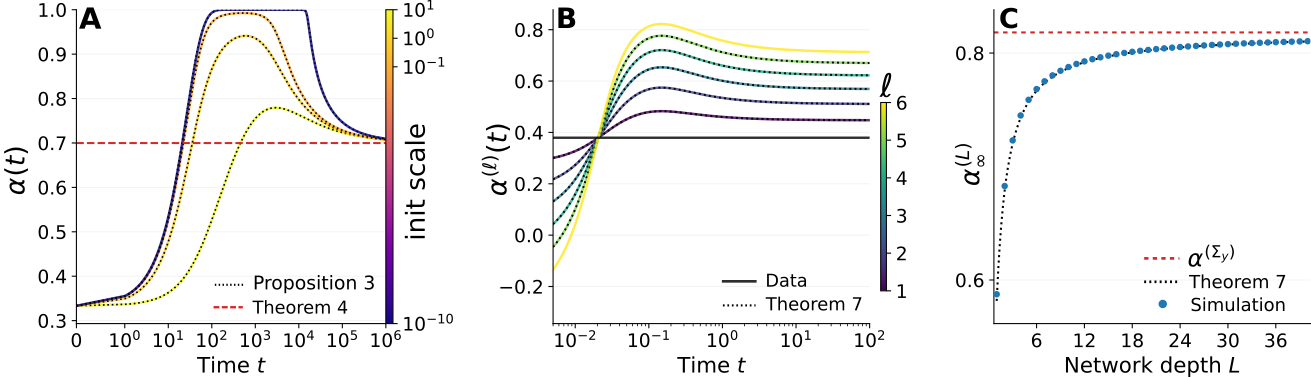}
    \caption{\textbf{Analytic laws governing abstraction dynamics in the minimal model}. All solid colored lines correspond to numerical simulations. \textbf{(A)} shows trajectory of $\alpha(t)$ in an $L=1$ hidden layer network for varying weight initialization. For all initializations, $\alpha(t)$ converges to the red terminal value \autoref{thm:fixed-point}. Meanwhile, the maximum abstraction reached over the trajectory decreases with initialization scale (\autoref{theorem:initialization-overshoot}). For very small initialization, $\alpha(t)$ reaches and stays near $1$ for a long time. \textbf{(B)} Abstraction interpolates layerwise in $L=6$ layer network, where $\ell$ is the layer index; lighter colors correspond to deeper layers. Abstraction improves in deeper layers. \textbf{(C)} Terminal abstraction in final layer increases with network depth $L$, where red limit is given by $\alpha_{\boldsymbol{\Sigma}_y}$.}
    \label{fig:linear-theory}
\vspace{-10pt}
\end{figure}

\subsection{How does depth impact abstraction?} \label{section:depth}

We now extend our model to a deep linear network with $L$ hidden layers. Let the superscript $\ell \in \{1, \ldots, L\}$ denote the layer, so $\smash{\lambda_m^{(\ell)}(t)}$ is the eigenvalue of $\smash{\mathbf{Q}^{(\ell)}}$ for mode $m$ and $\smash{\alpha^{(\ell)}(t)}$ is the abstraction. To derive the dynamics for arbitrary depth, we assume a mild setting similar to \cite{aroraConvergenceAnalysis19, kuninGetRich24}:

\begin{setting}[Layerwise balancing] \label{setting:balanced-depth}
For each mode $m \in \{ I, S, C, SC, G \}$, the layer gains remain equal along training, so $\lambda_m^{\smash{(\ell)}}=\lambda_m^{\smash{(0)}} u_m^{2\ell}$ for a scalar gain $u_m$ shared across layers.
\end{setting}

Under \autoref{setting:balanced-depth}, in \autoref{appendix:depth-dynamics} we derive the following ODE for the eigenvalues of the final hidden layer:
\begin{align}
  \dot \lambda_m = 2L \gamma \lambda^{(\boldsymbol{\Sigma}_y)}_m \pbracket{\lambda^{(\boldsymbol{\Sigma}_x)}_m}^{1/L} \frac{\pbracket{\lambda_m}^{2 - \frac{1}{L}}}{(1 + \gamma \lambda_m)^2}
\end{align}

The first question we can ask using this ODE is: for a deep network with $L$ hidden layers in this setting, how does abstraction differ across layers? In \autoref{appendix:depth-dynamics}, we derive the following answer:

\begin{theorem}[Layerwise interpolation and depth]
\label{theorem:layerwise-interp}
Suppose we are in \autoref{setting:balanced-depth}. Then the abstraction at layer $\ell \in \{1, \ldots, L \}$ obeys the following interpolation in $\mathrm{arctanh}$ space (also see \autoref{fig:linear-theory}B):
\begin{align} \label{eq:layerwise-interp}
\arctanh\alpha^{(\ell)}(t)=\pbracket{1-\frac{\ell}{L}} \arctanh\alpha_{\boldsymbol{\Sigma}_x}+\frac{\ell}{L}\arctanh\alpha^{(L)}(t)
\end{align}
Moreover, the terminal abstraction of the final layer $L$ satisfies:
\begin{align} \label{eq:deep-terminal-abstraction}
\arctanh\alpha_{\infty}^{(L)}=\frac{1}{L+1}\arctanh\alpha_{\boldsymbol{\Sigma}_x}+\frac{L}{L+1}\arctanh\alpha_{\boldsymbol{\Sigma}_y}
\end{align} 
\end{theorem}

Next we ask: how does $L$ affect the final layer's terminal abstraction? In \autoref{corollary:depth-monotonicity}, we analytically continue $L \in \mathbb{R}_+$ and differentiate \autoref{eq:deep-terminal-abstraction} to show that terminal abstraction is monotonically increasing in depth if $\nu(\boldsymbol{\Sigma}_y) < \nu(\boldsymbol{\Sigma}_x) \Longleftrightarrow \smash{\alpha_{\boldsymbol{\Sigma}_y}} > \smash{\alpha_{\boldsymbol{\Sigma}_x}}$. So if targets are more abstract than inputs, then \textbf{\textit{terminal abstraction will improve monotonically with network depth}} (see \autoref{fig:linear-theory}C).

\subsection{Empirical validation of the minimal model} \label{section:empirical}

To test whether our theory's predictions apply beyond the minimal setting, we train small ResNets on a task derived from the 3dshapes dataset \cite{kimDisentanglingFactorising18}. Despite major departures from our assumptions, the results qualitatively agree with key predictions of our theory (see \autoref{fig:empirical-validation}). See \autoref{appendix:3dshapes-experiments} for further details.

\begin{figure}[h]
    \centering
    \includegraphics[width=1.0\linewidth]{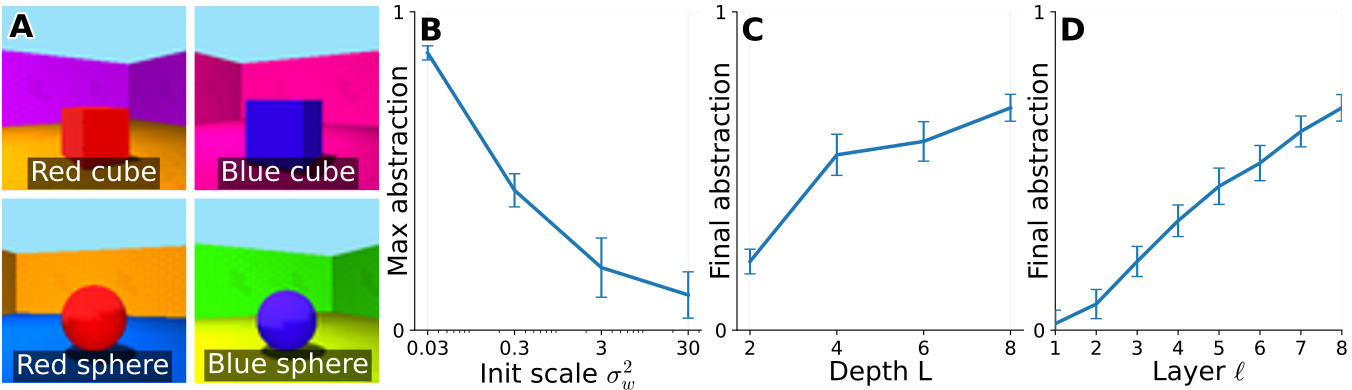}
    \caption{\textbf{Empirical validation of the minimal model in convolutional networks.} \textbf{(A)} We train small ResNets on a task constructed from the 3dshapes dataset \cite{kimDisentanglingFactorising18} using binary shape and color latents, with \(n=1000\) examples per class while nuisance variables (floor color, wall color, size, orientation) vary randomly. For panels B-D, error bars show SEM across 20 seeds. \textbf{(B)} Maximum abstraction in last layer decreases with weight initialization scale \(\sigma_w^2\), in agreement with \autoref{theorem:initialization-overshoot}. \textbf{(C)} Final abstraction after training in last layer increases with network depth \(L\), in agreement with \autoref{theorem:layerwise-interp}. \textbf{(D)} Final abstraction after training increases layerwise in a depth \(L=8\) network, in agreement with \autoref{theorem:layerwise-interp} (i.e. deeper layers are more abstract than early layers).}
    \label{fig:empirical-validation}
\vspace{-10pt}
\end{figure}

\section{Extending the theory to nonlinear networks} \label{section:nonlinear}

In this section we extend our theoretical framework to two-layer nonlinear networks in the infinite-width limit and study which aspects of the linear theory qualitatively survive or change. We now distinguish between preactivations $\mathbf{Z}$ and features $\mathbf{H} = \phi(\mathbf{Z})$ where we consider two types of nonlinearities for $\phi(\cdot)$: the error function $\phi_\beta(z) = \mathrm{erf}(\beta z)$ with $\beta > 0$, and the leaky ReLU (L-ReLU) $\phi_\omega(z) = \max\{z,\omega z\}$ with $\omega \in [0,1]$. We define the 2FS preactivation kernel $\mathbf{Q} \coloneqq \smash{\frac{1}{D}} \mathbf{Z}^\top \mathbf{Z}$, feature kernel $\mathbf{K} \coloneqq \smash{\frac{1}{D}} \mathbf{H}^\top \mathbf{H}$ and effective target kernel $\mathbf{M}(\mathbf{K}) \coloneqq \gamma (\mathbf{I} + \gamma \mathbf{K})^{-1} \boldsymbol{\Sigma}_y (\mathbf{I} + \gamma \mathbf{K})^{-1}$. 

As $D \to \infty$, we assume that the rows of $\mathbf{Z}$ are iid draws $\mathbf{z}^{(a)} \stackrel{\mathrm{iid}}{\sim} \mathcal{N}(0,\mathbf{Q})$ for $a=1,\ldots,D$. Then by law of large numbers, $\mathbf{K}_{ij} \to \mathbb{E}_{z\sim\mathcal{N}(0,\mathbf{Q})} [\phi(z_i)\phi(z_j)]$ and in \autoref{appendix:2fs-preserved} we prove that $\mathbf{K}$ inherits 2FS structure from $\mathbf{Q}$. Thus both $\mathbf{Q},\mathbf{K}$ are described by five types of entries we denote $\{q_\mu, k_\mu\}$, where the subscript $\mu \in \{d, 2, 1s, 1c, 0\}$ indexes the 2FS entries (see \autoref{fig:2fs-ansatz}E). Using canonical results on NNGP kernels \cite{choKernelMethods09, williamsComputingInfinite96} in \autoref{appendix:2fs-preserved} we derive the following kernel maps $\psi: (q_\mu, q_d) \mapsto k_\mu$:
\begin{align} \label{eq:kernel-maps}
  k_\mu = 
  \begin{cases}
    \psi_\beta(q_\mu, q_d) = \frac{2}{\pi} \arcsin \pbracket{\frac{2 \beta^2 q_\mu}{1 + 2 \beta^2 q_d}} & \text{(For erf)} \\
    \psi_\omega(q_\mu, q_d) = \frac{1+\omega^2}{2} q_\mu + \frac{(1-\omega)^2}{2 \pi} \bbracket{\sqrt{q_d^2 - q_\mu^2} - q_\mu \arccos \pbracket{\frac{q_\mu}{q_d}}} & \text{(For L-ReLU)}
  \end{cases}
\end{align}

Using this, we can study two aspects of the effect of nonlinearities on abstraction: (i) the instantaneous effect of nonlinearity on abstraction at any fixed time $t$ (\autoref{section:attenuation}), and (ii) the dynamical effect of nonlinearity on the trajectory of abstraction (\autoref{section:nonlinear-dynamics}).

\subsection{Instantaneous effect of nonlinearities on abstraction} \label{section:attenuation}

A crucial difference to the linear network case is that we can now measure abstraction at both preactivations ($\alpha_{\mathbf{Q}}$) and features ($\alpha_{\mathbf{K}}$). In this section we study the instantaneous effect of nonlinearity on abstraction by comparing $\alpha_{\mathbf{K}}$ to $\alpha_{\mathbf{Q}}$. Here, we ask a simple question: do the erf and L-ReLU nonlinearities (and hence their kernel maps $\psi_\omega$, $\psi_\beta$) improve abstraction? \textbf{\textit{Strikingly, the answer is neither}}. We prove the following result, which is stated informally here and formally in \autoref{theorem:attenuation-law-formal}:

\begin{theorem}[Attenuation law, informal]
\label{thm:nonlinear-attenuation}
Under \autoref{assumption:2fs} and as $D \to \infty$, at any time $t$, the abstraction in the features and preactivations obeys $\alpha_{\mathbf{K}}(t) = \mathcal{A}_{\phi}(\mathbf{Q}) \, \alpha_{\mathbf{Q}}(t)$ where $\mathcal{A}_{\phi}(\mathbf{Q})$ is an attenuation factor $0 \leq \mathcal{A}_{\phi}(\mathbf{Q}) \leq 1$ and \(\phi \in \{\phi_\beta,\phi_\omega\}\). Therefore $|\alpha_{\mathbf{K}}(t)| \leq |\alpha_{\mathbf{Q}}(t)|$ and if $\alpha_{\mathbf{Q}}(t) > 0$, the nonlinearity cannot improve abstraction, and generally weakens it.
\end{theorem}

In \autoref{appendix:attenuation} we derive the exact expression for $\mathcal{A}_{\phi}(\mathbf{Q})$, verify the bound $\mathcal{A}_{\phi}(\mathbf{Q}) \leq 1$ for erf and leaky ReLU, and discuss sufficient conditions for the law to hold for any other nonlinearities. We conjecture this result holds for a wider family of nonlinearities, and leave explicit proofs of this for future work. \autoref{thm:nonlinear-attenuation} immediately motivates applications for the extraction of more generalizable probing/steering vectors, which we study in \autoref{section:applications} in realistic finite-width settings.

\subsection{Dynamical effect of nonlinearities on abstraction} \label{section:nonlinear-dynamics}

In \autoref{section:attenuation} we showed the instantaneous effect of nonlinearity on $\alpha_{\mathbf{Q}}(t)$ at any fixed time $t$. A natural next step is to understand how nonlinearities affect the \textit{dynamics} of abstraction. In particular, we ask: how does the choice of nonlinearity $\phi$ affect the terminal abstraction $\alpha_{\mathbf{Q}, \infty}$ compared to the linear case? To answer this question, we first derive the ODEs for the nonlinear case:

\begin{theorem}[Nonlinear network abstraction dynamics] \label{thm:nonlinear-dynamics}
Let $\lambda_m$ be the eigenvalue of $\mathbf{Q}$ for mode $m \in \{I, S, C, SC, G\}$. As $D \to \infty$, $\lambda_m$ and $\alpha$ evolve according to the following ODEs respectively:
\begin{align} \label{eq:nonlinear-odes}
\dot{\lambda}_m = 2 \lambda_m \lambda_m^{(\boldsymbol{\Sigma}_x)} \lambda_m^{(\mathbf{R}_\phi)},
\qquad \dot{\alpha}_{\mathbf{Q}}
=
(1-\alpha_{\mathbf{Q}}^2)\lambda_S^{(\boldsymbol{\Sigma}_x)} \lambda_S^{(\mathbf{R}_\phi)} (1-\nu(\boldsymbol{\Sigma}_x) \nu(\mathbf{R}_\phi)),
\end{align}
where \(\mathbf{R}_\phi\) is a symmetric matrix with the following closed-form expressions:
\begin{align}
\mathbf{R}_\phi
=
\begin{cases}
\displaystyle
\mathbf{M}(\mathbf{K}) \odot \mathbf{G}_\beta
-
\zeta_\beta
\operatorname{Diag}\!\left(
\big(
\mathbf{M}(\mathbf{K})
\odot
\mathbf{G}_\beta
\odot
\mathbf{Q}
\big)
\mathbf{1}
\right),
&
\text{(For erf)}
\\[1em]
\displaystyle
\mathbf{M}(\mathbf{K}) \odot \mathbf{E}_\omega
+
\operatorname{Diag}\!\left(
\big(
\mathbf{M}(\mathbf{K})
\odot
\mathbf{D}_\omega
\big)
\mathbf{1}
\right),
&
\text{(For L-ReLU)}.
\end{cases}
\end{align}
Where $\zeta_\beta$, $\mathbf{G}_\beta$, \(\mathbf{D}_\omega\) and \(\mathbf{E}_\omega\) depend on the nonlinearity (full expressions in \autoref{appendix:nonlinear-riccati})
\end{theorem}

Notice that in \autoref{eq:nonlinear-odes} if we replace $\mathbf{R}_\phi$ with $\mathbf{M}(\mathbf{Q})$, we recover the linear theory (\autoref{proposition:linear-odes}). Recall that $\mathbf{M}(\mathbf{Q})$ was a spectrally filtered version of the target kernel $\boldsymbol{\Sigma}_y$. So by introducing $\phi$, we change how targets influence the dynamics via the more complicated \textit{nonlinear gain matrix} $\mathbf{R}_\phi$.  

\begin{figure}[h]
    \centering
    \includegraphics[width=1.0\linewidth]{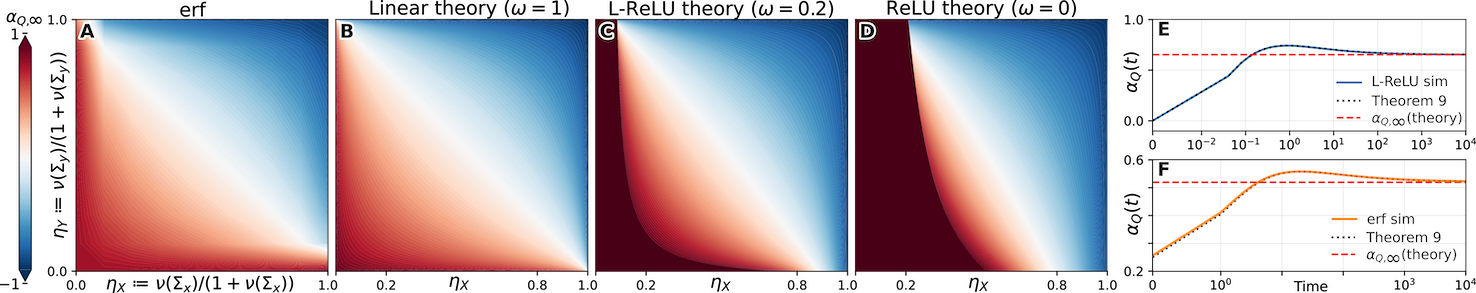}
    \caption{\textbf{Effect of nonlinearities on terminal abstraction}. Panels A-D show phase diagrams of $\alpha_{\mathbf{Q}, \infty}$ as a function of noise strength $\eta_\bullet \coloneqq \nu(\boldsymbol{\Sigma}_\bullet) / (1 + \nu(\boldsymbol{\Sigma}_\bullet))$ for inputs and targets. \textbf{(A)} Erf terminal abstraction with weak regularization ($\beta=1, \gamma=1000$) appears qualitatively well-described by the linear theory. \textbf{(B)} Linear theory from \autoref{thm:fixed-point}. \textbf{(C-D)} L-ReLU terminal abstraction under \autoref{setting:simplified-targets} evaluated by solving \autoref{eq:relufixedpointequation}. Decreasing $\omega$ warps the shape to make $\alpha_{\mathbf{Q}, \infty}$ more sensitive to inputs than targets and creates a band of near-perfect abstraction at low $\eta_X$. \textbf{(E)} Abstraction trajectory of L-ReLU network ($\omega = 0.5$), with red line evaluated by solving \autoref{eq:relufixedpointequation}. \textbf{(F)} Abstraction trajectory of erf network ($\beta = 0.1, \gamma=1000$) where red line is from linear theory (\autoref{thm:fixed-point}).}
    \label{fig:nonlinear}
\vspace{-5pt}
\end{figure}

\paragraph{(a) Erf.}
Consider the effectively ridgeless case $\gamma \to \infty$. Then the erf phase diagram for $\alpha_{\mathbf{Q}, \infty}$ appears qualitatively well-described by the linear network theory (\autoref{fig:nonlinear}A-B). Denoting $\alpha_{\infty}^{\mathrm{linear}}$ as the linear terminal value, in \autoref{appendix:erf-fixed-point} we show that the nonlinear correction is perturbative:

\begin{proposition}[Erf terminal abstraction]
\label{proposition:erf-fixed-point}
There exists a correction factor $ E_{\mathrm{erf}} > 0 $ such that
\begin{align} \label{eq:erf-linear-limit}
\alpha_{\mathbf{Q}, \infty}
=
\frac{1-\sqrt{\nu_{\boldsymbol{\Sigma}_x} \nu_{\boldsymbol{\Sigma}_y} E_{\mathrm{erf}}}}
{1+\sqrt{\nu_{\boldsymbol{\Sigma}_x} \nu_{\boldsymbol{\Sigma}_y} E_{\mathrm{erf}}}}
=
\alpha_{\infty}^{\mathrm{linear}}
+
O\!\left(
\beta^2 + \frac{1}{\gamma \beta^2}
\right).
\end{align}
In particular, as $\beta \to 0$ and $\gamma \beta^2 \to \infty$, the $\mathrm{erf}$ terminal abstraction converges to the linear one.
\end{proposition}

\paragraph{(b) Leaky ReLU.}
The more interesting case is ReLU given its dramatic difference from the linear theory (\autoref{fig:nonlinear}D). To study $\alpha_{\mathbf{Q},\infty}$ analytically we introduce the following simplifying assumptions:

\begin{setting}[Centered, signal-balanced regime] \label{setting:simplified-targets}
We assume $\lambda_S^{(\mathbf{A})} = \lambda_C^{(\mathbf{A})}$ for $\mathbf{A} \in \{ \boldsymbol{\Sigma}_x, \boldsymbol{\Sigma}_y, \mathbf{Q}(0) \}$ (the kernels are signal-balanced) and $\lambda_G^{(\boldsymbol{\Sigma}_y)} = \lambda_G^{(\mathbf{Q})}(0) = 0$ ($\boldsymbol{\Sigma}_y$ and $\mathbf{Q}(0)$ are centered).
\end{setting}

Define the shorthands: $\smash{\nu_X \coloneqq \nu(\boldsymbol{\Sigma}_x), \nu_Y \coloneqq \nu(\boldsymbol{\Sigma}_y), \nu_Q \coloneqq \nu(\mathbf{Q})}$. Under \autoref{setting:simplified-targets}, we can make a useful change of variables to the angle $\theta$ which parameterizes the terminal abstraction $\alpha_{\mathbf{Q}, \infty}(\theta)$: 
\begin{align}
\theta \coloneqq \arccos \pbracket{\frac{\nu_Q - 2}{\nu_Q + 2}} \in [0, \pi] \quad 
\implies \alpha_{\mathbf{Q}, \infty}(\theta) = -\frac{1 + 3 \cos \theta}{3 + \cos \theta}
\end{align}
Evaluating $\alpha_{\mathbf{Q}, \infty}$ under our theory then reduces to solving a scalar fixed point equation for $\theta$\footnote{Note: for each initial condition, evaluating $\alpha_{\mathbf{Q}, \infty}$ requires only a single scalar root solve without numerical integration.}:
\begin{align} \label{eq:relufixedpointequation}
  \nu_X \nu_R(\theta; \omega, \nu_Y) = 1
\end{align}
Where $\nu_R$ is the inverse SNR of $\mathbf{R}_\omega$ with explicit expression given in \autoref{appendix:relu-fixedpoint}, \autoref{eq:inverse-snr-relu}. To understand why the ReLU phase diagram differs so much, the most instructive part of \autoref{fig:nonlinear}D to study is the bottom edge corresponding to noiseless targets $\nu_Y = \eta_Y = 0$. This edge captures the strongest difference: in linear networks, $\nu_Y = 0$ forces $\alpha_{\mathbf{Q}, \infty} = 1$ regardless of the input, whereas for ReLU, this is no longer true. In \autoref{appendix:relu-fixedpoint} we derive the following analytic characterization for this edge:

\begin{proposition}[L-ReLU terminal abstraction with noiseless targets] \label{proposition:relu-fixed-point}
Suppose \autoref{setting:simplified-targets} holds and $\nu_Y = \eta_Y = 0$. The terminal abstraction $\alpha_{\mathbf{Q}, \infty}(\theta)$ is parameterized by the equation:
\begin{align}
  \nu_X = \frac{4 \pi \omega + (1 - \omega)^2 (2 \pi - \theta - \sin \theta)}{(1 - \omega)^2 (\theta - \sin \theta)}, \quad \text{or equivalently,} ~~ 
  \eta_X = \frac{\frac{4 \pi \omega}{(1 - \omega)^2} + 2 \pi - \theta - \sin \theta}{\frac{4 \pi \omega}{(1 - \omega)^2} + 2 \pi - 2 \sin \theta}
\end{align}
Moreover, $\alpha_{\mathbf{Q}, \infty} = 1$ iff $\nu_X \leq \nu_{\mathrm{crit}}(\omega)$ or $\eta_X \leq \eta_{\mathrm{crit}}(\omega)$, where $\nu_{\mathrm{crit}}(\omega)$ and $\eta_{\mathrm{crit}}(\omega)$ are:
\begin{align}
\nu_{\mathrm{crit}}(\omega) = 1 + \frac{4 \omega}{(1 - \omega)^2}, \quad \text{or equivalently}, ~ \eta_{\mathrm{crit}}(\omega) = \frac{4 \omega + (1 - \omega)^2}{4 \omega + 2 (1 - \omega)^2}
\end{align}
In particular, for pure ReLU (i.e. $\omega = 0$) we have $\nu_{\mathrm{crit}}(0) = 1$ and $\eta_{\mathrm{crit}}(0) = \frac{1}{2}$.
\end{proposition}

Thus for pure ReLU, noiseless targets result in $\alpha_{\mathbf{Q}, \infty} = 1$ only on the left half of the bottom edge of \autoref{fig:nonlinear}D where $\eta_X \leq 1/2$. Furthermore, we can also analytically solve the bottom right corner of \autoref{fig:nonlinear}D which is the sharpest contrast from the linear network case. Here, we take the limit $\nu_X \to \infty$, which forces $\theta \downarrow 0$ and $\alpha_{\mathbf{Q}, \infty} \to -1$, resulting in \textit{complete anti-abstraction}. This analytically characterizes previously observed empirical differences between ReLU and sigmoid-like nonlinearities \cite{allemanTaskStructure24}: \textbf{\textit{abstraction dynamics under ReLU are much less sensitive to target geometry, but much more sensitive to input geometry}}. We conjecture this may be a reason why ReLU-like nonlinearities excel in self-supervised settings like next token prediction where we cannot handcraft low-noise targets, but can rely on high-SNR structure in the input data. 

\section{Applications to deep learning and neuroscience} \label{section:applications}

Here we demonstrate two applications of our theory in deep learning and neuroscience respectively. 

\begin{figure}[h]
    \centering
    \includegraphics[width=1.0\linewidth]{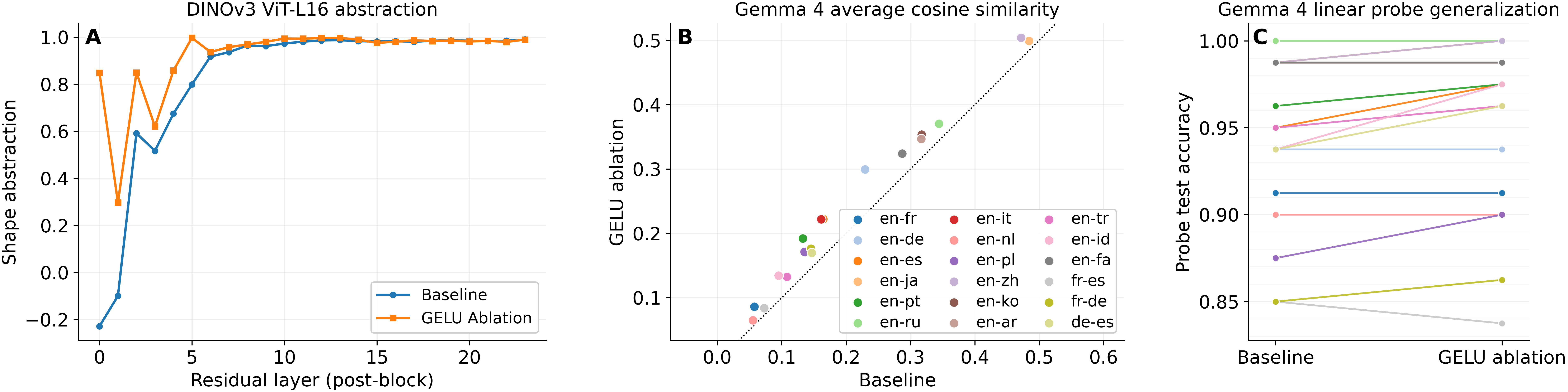}
    \caption{\textbf{Local GELU ablation improves concept abstraction and linear probe generalization in transformers}. We apply \autoref{thm:nonlinear-attenuation} to transformers by using the local GELU ablation technique described in \autoref{section:applications}. \textbf{(A)} For visual concepts, we evaluate DINOv3 ViT-L/16 \cite{simeoniDINOv325a} on the 3dshapes task described in \autoref{section:empirical}. Local GELU ablation improves or preserves shape abstraction in \textbf{\textit{every layer}} compared to baseline. Also, baseline abstraction improves layerwise which agrees with \autoref{theorem:layerwise-interp}. \textbf{(B)} For language concepts, we analyze the final residual stream of Gemma 4 E2B \cite{teamGemmaOpen24} evaluated on 18 bilingual concepts (e.g. \texttt{English-Spanish}), each instantiated by 80 word pairs (e.g. \texttt{Man, Hombre}). Local GELU ablation \textbf{\textit{improves abstraction for all 18 bilingual concepts}} (i.e. all points lie above diagonal). \textbf{(C)} The same intervention improves or preserves linear probe generalization for 17/18 concepts, with an average accuracy improvement of $\sim$1 percentage point.}
    \label{fig:language-ablation}
\end{figure}

\paragraph{Application to transformers.} A concept direction is useful only if it generalizes. Combining \autoref{thm:nonlinear-attenuation} with \autoref{proposition:ccgp} motivates a practical prediction: concept/steering vectors extracted via linear probes should be more abstract and generalize better after ablating nonlinearities. To operationalize this, we introduce the \textbf{local GELU ablation} procedure: (1) Measure abstraction in residual stream activations at layer \(\ell\) to get a baseline; (2) in the MLP block immediately preceding the activations, replace GELU with the identity; (3) rerun the forward pass and measure abstraction in the modified activations. We test local GELU ablation in both visual and language concepts in frontier open transformer models (see \autoref{fig:language-ablation}). Refer to \autoref{appendix:transformer-app-details} for further experiment details. 

\paragraph{Application to neural population codes.} \autoref{theorem:layerwise-interp} suggests a testable neuroscience prediction: representations should be more abstract in higher/deeper brain regions which are closer to downstream behaviour and learning signals. We test this on data from macaque visual cortex and find agreement with this prediction (see \autoref{fig:neuro}). See \autoref{appendix:majaj-app-details} for discussion and experimental details.

\begin{figure}[h] 
    \centering
    \includegraphics[width=1.0\linewidth]{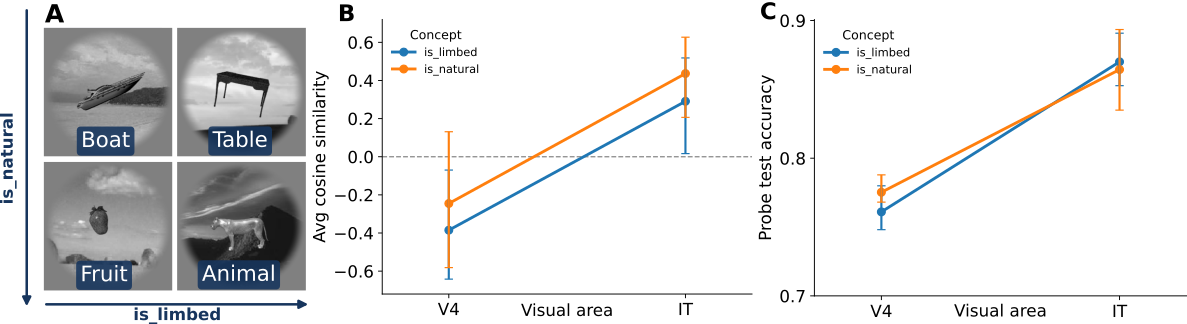}
    \caption{\textbf{Abstraction increases along macaque ventral visual stream.} We test the prediction of \autoref{theorem:layerwise-interp} that deeper cortical layers should exhibit greater abstraction using a subset of recordings in macaque V4 and IT from the public Majaj-Hong dataset \cite{majajSimpleLearned15}. \textbf{(A)} We construct a \(2 \times 2\) factor design around the concepts of \textit{naturalness} and \textit{limbedness}. \textbf{(B)} In a pooled site-matched analysis, both concepts are significantly more abstract in IT than in V4 (error bars are 95\% bootstrap intervals). \textbf{(C)} Linear probes trained to decode both concepts generalize significantly better in IT than in V4.}
    \label{fig:neuro}
\end{figure}

\section{Discussion} 

In this paper we developed a dynamical theory of abstraction in deep networks. We started in a simple, solvable linear network setting (\autoref{section:min-model}), which we fully solved (\autoref{section:exact-solutions}) and extended to nonlinear networks (\autoref{section:nonlinear}). Throughout, we provided evidence that our theory captures nontrivial phenomena in realistic settings (\autoref{section:empirical}) and demonstrated applications of our theory to both deep learning and neuroscience (\autoref{section:applications}). \textbf{Limitations.} Key limitations of our theory include (i) the simplifying assumptions made in \autoref{section:min-model} to obtain exact solutions and (ii) the difficulty of extending our theory to deep nonlinear networks. Despite these, experiments show that key qualitative predictions of our theory still hold when departing from our assumptions. \textbf{Future work.} Promising directions include (i) relaxation of our theoretical assumptions; (ii) broader application and analysis of the local GELU ablation method; and (iii) in-depth study of the theory's implications for interpretability and control of LLMs. 

\clearpage \bibliographystyle{unsrt}
\bibliography{bib}

\clearpage 
\section*{Appendix}
\appendix
\etocdepthtag.toc{appendix}

\begingroup
\etocsettagdepth{main}{none}
\etocsettagdepth{appendix}{subsection}
\etocsettocstyle{\section*{Appendix Contents}}{}
\tableofcontents
\endgroup

\clearpage

\counterwithin{theorem}{section}
\renewcommand{\thetheorem}{\thesection.\arabic{theorem}}

\renewcommand{\thelemma}{\thetheorem}
\renewcommand{\theproposition}{\thetheorem}
\renewcommand{\thecorollary}{\thetheorem}

\section{Supplementary discussion}

\subsection{Further related work} \label{appendix:related-work}

In this section we present a more comprehensive discussion of related work.

\paragraph{Abstract representations.} Recent work showed that both high-level brain regions and neural networks develop a representational geometry supporting simple forms of out-of-distribution generalization, often termed ``abstract'' \cite{johnstonAbstractRepresentations23, bernardiGeometryAbstraction20, nogueiraGeometryCortical23, fascianelliNeuralRepresentational24a, courellisAbstractRepresentations24, oneillRepresentationalGeometry24, zhuGeometricFoundation26, itoCompositionalGeneralization22}. This geometry is quantified by \textit{parallelism} \cite{bernardiGeometryAbstraction20}--the cosine similarity between counterfactual concept vectors--which is equivalent to our abstraction score. Recent empirical work on the choice of nonlinearity showed that the sigmoid-like \texttt{tanh} nonlinearity often results in more abstract representations than using \texttt{ReLU} \cite{allemanTaskStructure24}. In this paper, we provide a theoretical explanation for this phenomenon in \autoref{section:nonlinear-dynamics}. Recently, \cite{wangMathematicalTheory26} proved that at the global optimum of the loss, networks exhibit perfect abstraction in their last layer. The authors in \cite{wangMathematicalTheory26} considered an arbitrary number of concepts, and their theory's predictions remained unchanged by this number. Motivated by this, we develop a theory for the case of two concepts that treats abstraction as a dynamical variable during training, extends to more general input-target geometries, and study its evolution across layers in features and preactivations. 

\paragraph{Linear networks.} We build on a line of previous work on linear network learning dynamics (e.g. \cite{saxeExactSolutions14, lampinenAnalyticTheory19a, kuninGetRich24}), which have been extensively studied as analytically solvable yet insightful settings \cite{namPositionSolve25a, simonThereWill26}. Like \cite{braunExactLearning22, domineLazyRich25} we adopt a matrix Riccati approach to studying learning dynamics, but whereas they study the dynamics of the block matrix 
$\tiny \begin{bmatrix}
  \mathbf{W}^\top \mathbf{W} & \mathbf{W}^\top \mathbf{W}_r^\top \\
  \mathbf{W}_r \mathbf{W} & \mathbf{W}_r \mathbf{W}_r^\top
\end{bmatrix}$ in a two-layer linear network, we employ a variable projection reduction on the readout $\mathbf{W}_r$ (see \autoref{assumption:readout}) which enables us to solve the continuous-time dynamics of the hidden representation kernel $\mathbf{Q} = (\mathbf{W} \mathbf{X})^\top \mathbf{WX}$ in the two-layer case, and also extend our analysis to networks of arbitrary depth. 

\paragraph{Linear word analogies in word embedding models.} The idea of abstraction has strong parallels with the linear analogy structures discovered in word embedding models such as Word2Vec \cite{mikolovLinguisticRegularities13}. The origin of such structure has been the subject of much recent theoretical work \cite{allenAnalogiesExplained19, karkadaClosedFormTraining25, korchinskiEmergenceLinear25, karkadaSymmetryLanguage26}. \cite{karkadaClosedFormTraining25} is closest in spirit to our work, as they take a learning dynamics approach by analytically solving the gradient flow dynamics of a Word2Vec-like embedding model. However, their theory focuses on sequential learning of the eigenmodes of the approximate pointwise mutual information matrix, whereas we directly study the dynamics of cosine similarity between counterfactual concept directions. More broadly, these works study abstraction specifically arising from natural language co-occurrence statistics, whereas we study abstraction in a domain-agnostic setting, and then specialize the theory to vision and language concepts empirically.

\paragraph{Linear representation hypothesis.} Recent theoretical work has made the linear representation hypothesis precise for LLMs \cite{parkLinearRepresentation24, jiangOriginsLinear24, parkGeometryCategorical25}. \cite{parkLinearRepresentation24} formalized the notion of linearity for binary concepts such that a concept is said to be linearly embedded when every representation difference induced by changing only that concept lies in a common positive cone. In \autoref{appendix:lrh-parallelism}, we show that our measure of abstraction is a continuous extension of this exact cone-linearity condition. \cite{jiangOriginsLinear24} is closest in spirit to our work, as they prove that gradient descent asymptotically results in perfect abstraction in a simplified setting. In contrast, we study the dynamics of abstraction during training (rather than just end-of-training), and also study the evolution of abstraction across layers and through nonlinearities. 

\subsection{Relationship between abstraction and linear representation hypothesis (LRH)} \label{appendix:lrh-parallelism}

In this section we discuss the claim made in \autoref{section:min-model} that the abstraction or ``parallelism score'' from theoretical neuroscience or abstraction score is a graded extension of the exact linear representation condition adopted in the LRH literature, which we will call ``cone-linearity''. In particular, we will show thatwhen abstraction is perfect (meaning \(\alpha = 1\)), the context-specific concept vectors all lie in one common positive cone, which is exactly the cone-linearity condition used in \cite{parkLinearRepresentation24, jiangOriginsLinear24}. To avoid overloading notation, we write generic representation vectors as \(\mathbf{r}\) throughout this section.

\paragraph{Preliminaries.} First we unify the notation used in this paper and \cite{parkLinearRepresentation24, jiangOriginsLinear24}. Let \(\mathcal{V}\) be a representation space and denote the representation associated with shape value \(s\) and color value \(c\) with:
\begin{align}
    \mathbf{r}_{s,c}\in\mathcal{V},
    \qquad
    (s,c)\in\{+1,-1\}^{2},
\end{align}
In the main text, \(\mathbf{r}_{s,c}\) is the class centroid \(\boldsymbol{\mu}_{s,c}\) of the learned features. We take \(s=+1\) to mean square, \(s=-1\) to mean circle, \(c=+1\) to mean red, and \(c=-1\) to mean blue. The two context-specific shape vectors can be written as:
\begin{align}
    \boldsymbol{\delta}^{S}_{+}
    &\coloneqq
    \mathbf{r}_{+1,+1}-\mathbf{r}_{-1,+1}
    \equiv
    \mathbf{u}_{\mathrm{red}}
    \\
    \boldsymbol{\delta}^{S}_{-}
    &\coloneqq
    \mathbf{r}_{+1,-1}-\mathbf{r}_{-1,-1}
    \equiv
    \mathbf{u}_{\mathrm{blue}}
\end{align}
Thus the abstraction score from the main text is
\begin{align}
    \alpha
    =
    \frac{
        \left\langle
            \boldsymbol{\delta}^{S}_{+},
            \boldsymbol{\delta}^{S}_{-}
        \right\rangle
    }{
        \left\|
            \boldsymbol{\delta}^{S}_{+}
        \right\|
        \left\|
            \boldsymbol{\delta}^{S}_{-}
        \right\|
    }
    \label{eq:appendix-alpha-as-cosine}
\end{align}

And for a nonzero vector \(\mathbf{v}\), define its positive cone by
\begin{align}
    \operatorname{Cone}(\mathbf{v})
    \coloneqq
    \left\{
        a\mathbf{v}: a>0
    \right\}
\end{align}
Note that the positivity matters: perfect linearity is \(\alpha=1\), not merely \(|\alpha|=1\). If \(\alpha=-1\), then
\begin{align}
    \boldsymbol{\delta}^{S}_{-}
    =
    -a\boldsymbol{\delta}^{S}_{+}
    \qquad
    \text{for some } a>0
\end{align}
The two context-specific vectors then span the same line, but they point in opposite directions. This would mean that the square-minus-circle direction in one color context is the circle-minus-square direction in the other color context. \cite{parkLinearRepresentation24,jiangOriginsLinear24} use positive cones precisely to rule this out.

\begin{proposition}[Perfect abstraction is exactly cone-linearity]
\label{proposition:perfect-abstraction-lrh}
Assume
\(\boldsymbol{\delta}^{S}_{+}\neq 0\) and \(\boldsymbol{\delta}^{S}_{-}\neq 0\). Then the following are equivalent:
\begin{enumerate}[label=(\roman*)]
    \item The shape abstraction score is perfect:
    \begin{align}
        \alpha = 1
    \end{align}
    \item There exists a nonzero vector \(\overline{\mathbf{r}}_{S}\in\mathcal{V}\) such that
    \begin{align}
        \boldsymbol{\delta}^{S}_{+}
        \in
        \operatorname{Cone}(\overline{\mathbf{r}}_{S}),
        \qquad
        \boldsymbol{\delta}^{S}_{-}
        \in
        \operatorname{Cone}(\overline{\mathbf{r}}_{S})
        \label{eq:appendix-two-context-cone-linearity}
    \end{align}
    \item At the level of the representations \(\mathbf{r}_{s,c}\), the shape concept has a linear representation in the cone sense used by \cite{parkLinearRepresentation24, jiangOriginsLinear24}.
\end{enumerate}
\end{proposition}

\begin{proof}
The equivalence between (i) and (ii) is exactly the equality condition for Cauchy--Schwarz. Starting from \autoref{eq:appendix-alpha-as-cosine}, we have
\begin{align}
    \left\|
        \frac{\boldsymbol{\delta}^{S}_{+}}{\left\|\boldsymbol{\delta}^{S}_{+}\right\|}
        -
        \frac{\boldsymbol{\delta}^{S}_{-}}{\left\|\boldsymbol{\delta}^{S}_{-}\right\|}
    \right\|^{2}
    &=
    \left\|
        \frac{\boldsymbol{\delta}^{S}_{+}}{\left\|\boldsymbol{\delta}^{S}_{+}\right\|}
    \right\|^{2}
    +
    \left\|
        \frac{\boldsymbol{\delta}^{S}_{-}}{\left\|\boldsymbol{\delta}^{S}_{-}\right\|}
    \right\|^{2}
    -
    2
    \frac{
        \left\langle
            \boldsymbol{\delta}^{S}_{+},
            \boldsymbol{\delta}^{S}_{-}
        \right\rangle
    }{
        \left\|
            \boldsymbol{\delta}^{S}_{+}
        \right\|
        \left\|
            \boldsymbol{\delta}^{S}_{-}
        \right\|
    }
    \\
    &=
    2(1-\alpha)
\end{align}
Therefore \(\alpha=1\) if and only if
\begin{align}
    \frac{\boldsymbol{\delta}^{S}_{+}}{\left\|\boldsymbol{\delta}^{S}_{+}\right\|}
    =
    \frac{\boldsymbol{\delta}^{S}_{-}}{\left\|\boldsymbol{\delta}^{S}_{-}\right\|}
\end{align}
Equivalently,
\begin{align}
    \boldsymbol{\delta}^{S}_{-}
    =
    a\boldsymbol{\delta}^{S}_{+},
    \qquad
    a
    =
    \frac{
        \left\|\boldsymbol{\delta}^{S}_{-}\right\|
    }{
        \left\|\boldsymbol{\delta}^{S}_{+}\right\|
    }
    >
    0
\end{align}
Thus both context-specific shape vectors lie in the same positive cone. Taking
\begin{align}
    \overline{\mathbf{r}}_{S}
    =
    \boldsymbol{\delta}^{S}_{+}
\end{align}
proves (ii). Conversely, if (ii) holds, then there exist \(a_{+},a_{-}>0\) such that
\begin{align}
    \boldsymbol{\delta}^{S}_{+}
    =
    a_{+}\overline{\mathbf{r}}_{S},
    \qquad
    \boldsymbol{\delta}^{S}_{-}
    =
    a_{-}\overline{\mathbf{r}}_{S}
\end{align}
Substituting into \autoref{eq:appendix-alpha-as-cosine} gives
\begin{align}
    \alpha
    =
    \frac{
        a_{+}a_{-}
        \left\|
            \overline{\mathbf{r}}_{S}
        \right\|^{2}
    }{
        a_{+}
        \left\|
            \overline{\mathbf{r}}_{S}
        \right\|
        a_{-}
        \left\|
            \overline{\mathbf{r}}_{S}
        \right\|
    }
    =
    1
\end{align}
This proves (i) \(\Longleftrightarrow\) (ii).

Next we connect (ii) to the LRH definitions (iii). In \cite{parkLinearRepresentation24,jiangOriginsLinear24} the authors define an embedding representation of a target binary concept \(s\) by requiring that every counterfactual pair of context embeddings which changes \(s\) (while leaving off-target concepts unchanged) has difference in a common cone. In our notation, the target concept is shape \(s\), and the off-target concept is color \(c\). For each color fixed at \(+1\) or \(-1\), this condition can be written as:
\begin{align}
    \boldsymbol{\delta}^{S}_{+}, \boldsymbol{\delta}^{S}_{-}
    \in
    \operatorname{Cone}(\overline{\mathbf{r}}_{s})
\end{align}
Thus cone-linearity says precisely that both \(\boldsymbol{\delta}^{S}_{+}\) and \(\boldsymbol{\delta}^{S}_{-}\) lie in one common positive cone, which is (ii). This proves (ii) \(\Longleftrightarrow\) (iii), which completes the proof.
\end{proof}

\paragraph{How \cite{parkLinearRepresentation24, jiangOriginsLinear24} relates to our results.}
The results of \cite{parkLinearRepresentation24, jiangOriginsLinear24} are exact LRH statements: a concept is linearly represented when all of its counterfactual concept vectors lie in a common positive cone. The abstraction score \(\alpha\) turns this exact binary property into a continuous dynamical quantity:
\begin{align}
    \alpha=1
    \quad
    \Longleftrightarrow
    \quad
    \text{exact LRH cone-linearity},
\end{align}
while \(\alpha<1\) measures the angular deviation from exact cone-linearity.

This is why \(\alpha(t)\) is the natural dynamical variable for the present paper. Jiang et al. \cite{jiangOriginsLinear24} prove that, in their simplified gradient-descent setting, $\lim_{t \to \infty} \alpha(t) = 1$. In our notation, we say that the terminal abstraction is perfect. By contrast, \autoref{thm:fixed-point} characterizes when perfect abstraction occurs in the minimal linear network studied here:
\begin{align}
    \alpha_{\infty}
    =
    \frac{
        1-\sqrt{\nu(\boldsymbol{\Sigma}_x)\nu(\boldsymbol{\Sigma}_y)}
    }{
        1+\sqrt{\nu(\boldsymbol{\Sigma}_x)\nu(\boldsymbol{\Sigma}_y)}
    }
\end{align}
In particular, if either the inputs or targets have no interaction noise, then
\begin{align}
    \nu(\boldsymbol{\Sigma}_x)\nu(\boldsymbol{\Sigma}_y)=0
    \qquad
    \Longrightarrow
    \qquad
    \alpha_{\infty}=1
\end{align}
Thus our terminal law recovers exact cone-linearity in the noiseless cases, while also describing partially abstract representations when input and target geometries contain interaction noise. 

\clearpage
\subsection{Two-factor symmetry (\autoref{assumption:2fs})} \label{appendix:2fs-walsh-hadamard}

This section provides an unabridged explanation and discussion on \autoref{assumption:2fs} in \autoref{section:min-model}. The goal is to motivate 2FS kernels in more detail and show they arise naturally from a simple and interpretable picture in which both inputs and targets are functions of two binary latent variables.

\paragraph{Walsh-Hadamard coordinates for two binary latents.}
The natural coordinate system for functions of the two binary latent variables \(s, c\) is the Walsh-Hadamard basis, an analogue of the Fourier basis.

\begin{proposition}[Walsh-Hadamard expansion] \label{proposition:walsh-expansion}
Let \((s,c) \in \{-1,+1\}^2\). Then any vector-valued function \(\mathbf{f}: (s, c) \to \mathbb{R}^d\) admits a unique expansion
\begin{align}
  \mathbf{f}(s,c)
  =
  \mathbf{f}_\varnothing
  +
  s\mathbf{f}_S
  +
  c\mathbf{f}_C
  +
  sc\mathbf{f}_{SC},
\end{align}
where the Fourier coefficients are
\begin{align}
  \mathbf{f}_\varnothing
  &=
  \frac{1}{4}\sum_{s,c}\mathbf{f}(s,c)
  &
  \mathbf{f}_S
  &=
  \frac{1}{4}\sum_{s,c}s\mathbf{f}(s,c)
  \\
  \mathbf{f}_C
  &=
  \frac{1}{4}\sum_{s,c}c\mathbf{f}(s,c)
  &
  \mathbf{f}_{SC}
  &=
  \frac{1}{4}\sum_{s,c}sc\mathbf{f}(s,c)
\end{align}
\end{proposition}

\begin{proof}
Let
\begin{align}
\chi_\varnothing(s,c)=1,
\qquad
\chi_S(s,c)=s,
\qquad
\chi_C(s,c)=c,
\qquad
\chi_{SC}(s,c)=sc
\end{align}
These four scalar functions are orthonormal under the uniform inner product on \(\{-1,+1\}^2\):
\begin{align}
\frac{1}{4}\sum_{s,c}\chi_a(s,c)\chi_b(s,c)
=
\begin{cases}
1, & a=b,\\
0, & a\neq b
\end{cases}
\end{align}
For example, \(\frac{1}{4}\sum_{s,c}s=0\), \(\frac{1}{4}\sum_{s,c}sc=0\), and \(\frac{1}{4}\sum_{s,c}s^2=1\). Since there are four functions on a four-point domain, this orthonormal set is a basis for all scalar functions on the domain. Applying the scalar expansion coordinatewise to \(\mathbf{f}\) gives the displayed vector-valued expansion. The coefficient formulas are exactly the orthogonal projections of \(\mathbf{f}\) onto each basis function, and uniqueness follows from basis uniqueness.
\end{proof}

\paragraph{Sample indexing and the five 2FS modes.}
We next connect the Walsh-Hadamard basis to the sample-level kernel symmetry in \autoref{assumption:2fs}. Index samples by triples
\begin{align}
 i=(s_i,c_i,a_i),
 \qquad
 s_i,c_i\in\{-1,+1\},
 \qquad
 a_i\in\{1,\ldots,n\},
\end{align}
where \(a_i\) is the within-class replicate index. Thus \(N=4n\). Define four normalized class-constant vectors in \(\mathbb{R}^N\):
\begin{align}
(\boldsymbol{\chi}_G)_i
&=
\frac{1}{\sqrt{N}},
&
(\boldsymbol{\chi}_S)_i
&=
\frac{s_i}{\sqrt{N}},
&
(\boldsymbol{\chi}_C)_i
&=
\frac{c_i}{\sqrt{N}},
&
(\boldsymbol{\chi}_{SC})_i
&=
\frac{s_i c_i}{\sqrt{N}}
\end{align}
Balanced sampling implies that these four vectors are orthonormal. Let
\begin{align}
\mathbf{P}_G
&=
\boldsymbol{\chi}_G\boldsymbol{\chi}_G^\top,
&
\mathbf{P}_S
&=
\boldsymbol{\chi}_S\boldsymbol{\chi}_S^\top,
&
\mathbf{P}_C
&=
\boldsymbol{\chi}_C\boldsymbol{\chi}_C^\top,
&
\mathbf{P}_{SC}
&=
\boldsymbol{\chi}_{SC}\boldsymbol{\chi}_{SC}^\top,
\end{align}
and
\begin{align}
\mathbf{P}_I
=
\I_N-\mathbf{P}_G-\mathbf{P}_S-\mathbf{P}_C-\mathbf{P}_{SC}
\end{align}
The projector \(\mathbf{P}_I\) is the within-class residual subspace: it contains vectors whose entries sum to zero within each of the four fine classes. Thus \(G,S,C,SC\) describe the four-dimensional class-centroid geometry, while \(I\) describes variation left over after class averaging.

The next proposition makes precise the equivalence between the group-invariance definition of 2FS in \autoref{assumption:2fs}, the five entry types shown in \autoref{fig:2fs-ansatz}D, and the simultaneous diagonalization used throughout the main text.

\begin{proposition}[Equivalent forms of a 2FS kernel] \label{proposition:2fs-projector-decomposition}
Let \(\mathbf{A}\in\mathbb{R}^{N\times N}\) be a kernel matrix, with \(N=4n\) and \(n>1\). The following descriptions are equivalent.

\begin{enumerate}[label=(\roman*)]
\item \(\mathbf{A}\) is invariant under \(\mathcal{G}\cong (S_n)^4\rtimes (\mathbb{Z}_2)^2\), meaning that for permutation matrix \(\boldsymbol{\Pi}_g\)
\begin{align}
\boldsymbol{\Pi}_g \mathbf{A} \boldsymbol{\Pi}_g^\top=\mathbf{A}
\qquad
\text{for all }g\in\mathcal{G}
\end{align}
\item \(\mathbf{A}\) has five entry types:
\begin{align}
\mathbf{A}_{ij}
=
\begin{cases}
a_d,
& i=j,
\\
a_2,
& i\neq j \text{ and } (s_i,c_i)=(s_j,c_j),
\\
a_{1s},
& s_i=s_j \text{ and } c_i\neq c_j,
\\
a_{1c},
& s_i\neq s_j \text{ and } c_i=c_j,
\\
a_0,
& s_i\neq s_j \text{ and } c_i\neq c_j
\end{cases}
\end{align}
Here \(a_{1s}\) denotes an entry between samples with the same shape and different color, while \(a_{1c}\) denotes an entry between samples with the same color and different shape.
\item \(\mathbf{A}\) is diagonal in the five projectors:
\begin{align}
\mathbf{A}
=
\lambda_G^{(\mathbf{A})}\mathbf{P}_G
+
\lambda_S^{(\mathbf{A})}\mathbf{P}_S
+
\lambda_C^{(\mathbf{A})}\mathbf{P}_C
+
\lambda_{SC}^{(\mathbf{A})}\mathbf{P}_{SC}
+
\lambda_I^{(\mathbf{A})}\mathbf{P}_I
\end{align}
\end{enumerate}
For a kernel with entry types as in (ii), define
\begin{align}
\bar a
\coloneqq
\frac{a_d+(n-1)a_2}{n},
\end{align}
the inner product between two class centroids from the same fine class. Then the eigenvalues in (iii) are
\begin{align}
\lambda_G^{(\mathbf{A})}
&=
 n(\bar a+a_{1s}+a_{1c}+a_0),
\\
\lambda_S^{(\mathbf{A})}
&=
 n(\bar a+a_{1s}-a_{1c}-a_0),
\\
\lambda_C^{(\mathbf{A})}
&=
 n(\bar a-a_{1s}+a_{1c}-a_0),
\\
\lambda_{SC}^{(\mathbf{A})}
&=
 n(\bar a-a_{1s}-a_{1c}+a_0),
\\
\lambda_I^{(\mathbf{A})}
&=
 a_d-a_2
\end{align}
Conversely, if \(\mathbf{A}\) has the projector decomposition in (iii), then its five entry types are
\begin{align}
 a_d
 &=
 \frac{\lambda_G^{(\mathbf{A})}+\lambda_S^{(\mathbf{A})}+\lambda_C^{(\mathbf{A})}+\lambda_{SC}^{(\mathbf{A})}}{4n}
 +
 \left(1-\frac{1}{n}\right)\lambda_I^{(\mathbf{A})},
\\
 a_2
 &=
 \frac{\lambda_G^{(\mathbf{A})}+\lambda_S^{(\mathbf{A})}+\lambda_C^{(\mathbf{A})}+\lambda_{SC}^{(\mathbf{A})}}{4n}
 -
 \frac{1}{n}\lambda_I^{(\mathbf{A})},
\\
 a_{1s}
 &=
 \frac{\lambda_G^{(\mathbf{A})}+\lambda_S^{(\mathbf{A})}-\lambda_C^{(\mathbf{A})}-\lambda_{SC}^{(\mathbf{A})}}{4n},
\\
 a_{1c}
 &=
 \frac{\lambda_G^{(\mathbf{A})}-\lambda_S^{(\mathbf{A})}+\lambda_C^{(\mathbf{A})}-\lambda_{SC}^{(\mathbf{A})}}{4n},
\\
 a_0
 &=
 \frac{\lambda_G^{(\mathbf{A})}-\lambda_S^{(\mathbf{A})}-\lambda_C^{(\mathbf{A})}+\lambda_{SC}^{(\mathbf{A})}}{4n}
\end{align}
\end{proposition}

\begin{proof}
First suppose \(\mathbf{A}\) is invariant under \(\mathcal{G}\). Arbitrary permutations within each fine class force all diagonal entries within the four classes to be equal, all off-diagonal entries within each fine class to be equal, and all entries between any fixed pair of fine classes to be constant. The global relabelings \(s\mapsto -s\) and \(c\mapsto -c\) then identify class pairs with the same agreement pattern: same shape and different color, same color and different shape, or different in both factors. Hence \(\mathbf{A}\) has exactly the five entry types in (ii).

Next assume (ii). We show that the five-entry form is diagonal in the projectors above. Consider the action of \(\mathbf{A}\) on \(\boldsymbol{\chi}_S\). Fix a row \(i\) with latent values \((s_i,c_i)\). Summing over the four classes gives
\begin{align}
(\mathbf{A}\boldsymbol{\chi}_S)_i
&=
\frac{1}{\sqrt{N}}
\Big[
\big(a_d+(n-1)a_2\big)s_i
+
 n a_{1s}s_i
-
 n a_{1c}s_i
-
 n a_0s_i
\Big]
\\
&=
 n(\bar a+a_{1s}-a_{1c}-a_0)(\boldsymbol{\chi}_S)_i
\end{align}
Thus \(\boldsymbol{\chi}_S\) is an eigenvector with eigenvalue \(n(\bar a+a_{1s}-a_{1c}-a_0)\). The same calculation with \(1\), \(c_i\), and \(s_i c_i\) gives the displayed formulas for \(\lambda_G^{(\mathbf{A})}\), \(\lambda_C^{(\mathbf{A})}\), and \(\lambda_{SC}^{(\mathbf{A})}\).

It remains to check the residual subspace. Let \(\mathbf{v}\in\operatorname{range}(\mathbf{P}_I)\), so for every fine class \((s,c)\),
\begin{align}
\sum_{i:(s_i,c_i)=(s,c)}v_i=0
\end{align}
For a row \(i\), all contributions from other classes vanish because \(\mathbf{A}_{ij}\) is constant over each other fine class and the sum of \(v_j\) in that class is zero. Within the same class as \(i\),
\begin{align}
(\mathbf{A}\mathbf{v})_i
&=
a_d v_i
+
a_2\sum_{j\neq i:(s_j,c_j)=(s_i,c_i)}v_j
\\
&=
a_d v_i-a_2v_i
=
(a_d-a_2)v_i
\end{align}
Thus every vector in the residual subspace is an eigenvector with eigenvalue \(a_d-a_2\). This proves (iii) and the eigenvalue formulas.

Conversely, each projector \(\mathbf{P}_G,\mathbf{P}_S,\mathbf{P}_C,\mathbf{P}_{SC},\mathbf{P}_I\) is invariant under within-class permutations and global sign flips. Any linear combination of them is therefore invariant, proving (iii) \(\Rightarrow\) (i). The entry formulas follow by writing out the entries of the projectors. For instance, \((\mathbf{P}_S)_{ij}=s_i s_j/N\), \((\mathbf{P}_{SC})_{ij}=s_i c_i s_j c_j/N\), while \((\mathbf{P}_I)_{ij}=1-1/n\) on the diagonal, \((\mathbf{P}_I)_{ij}=-1/n\) for distinct samples in the same fine class, and \((\mathbf{P}_I)_{ij}=0\) for samples in different fine classes.
\end{proof}

\paragraph{Abstraction in the 2FS eigenbasis.}
In the main text we show that shape abstraction has a closed form in just the \(S\) and \(SC\) eigenvalues. We now derive this identity explicitly.

\begin{proposition}[Shape abstraction from 2FS eigenvalues] \label{proposition:2fs-abstraction-eigenvalues}
Let \(\mathbf{Q}=\mathbf{Z}^\top \mathbf{Z}\) be a 2FS kernel. If \(\lambda_S^{(\mathbf{Q})}+\lambda_{SC}^{(\mathbf{Q})}>0\), then the shape abstraction score defined in \autoref{section:min-model} is
\begin{align}
\alpha
=
\frac{\lambda_S^{(\mathbf{Q})}-\lambda_{SC}^{(\mathbf{Q})}}{\lambda_S^{(\mathbf{Q})}+\lambda_{SC}^{(\mathbf{Q})}}
=
\frac{1-\nu(\mathbf{Q})}{1+\nu(\mathbf{Q})},
\qquad
\nu(\mathbf{Q})
=
\frac{\lambda_{SC}^{(\mathbf{Q})}}{\lambda_S^{(\mathbf{Q})}}
\end{align}
\end{proposition}

\begin{proof}
Let \(q_d,q_2,q_{1s},q_{1c},q_0\) be the five entry types of \(\mathbf{Q}\), and define
\begin{align}
\bar q
=
\frac{q_d+(n-1)q_2}{n}
\end{align}
The quantity \(\bar q\) is the squared norm of a fine-class centroid, because a centroid averages one diagonal kernel entry and \(n-1\) off-diagonal within-class entries per sample.

Write \(\boldsymbol{\mu}_{s,c}\) for the centroid of class \((s,c)\). The two context-specific shape directions are
\begin{align}
\mathbf{u}_{+}
&=
\boldsymbol{\mu}_{+,+}-\boldsymbol{\mu}_{-,+},
&
\mathbf{u}_{-}
&=
\boldsymbol{\mu}_{+,-}-\boldsymbol{\mu}_{-,-}
\end{align}
Here \(c=+1\) and \(c=-1\) are the two color contexts. By the five-entry structure,
\begin{align}
\|\mathbf{u}_{+}\|^2
&=
\|\mathbf{u}_{-}\|^2
=
2(\bar q-q_{1c}),
\\
\langle \mathbf{u}_{+},\mathbf{u}_{-}\rangle
&=
2(q_{1s}-q_0)
\end{align}
Therefore
\begin{align}
\alpha
=
\frac{q_{1s}-q_0}{\bar q-q_{1c}}
\end{align}
Using the eigenvalue formulas from \autoref{proposition:2fs-projector-decomposition},
\begin{align}
q_{1s}-q_0
&=
\frac{\lambda_S^{(\mathbf{Q})}-\lambda_{SC}^{(\mathbf{Q})}}{2n},
\\
\bar q-q_{1c}
&=
\frac{\lambda_S^{(\mathbf{Q})}+\lambda_{SC}^{(\mathbf{Q})}}{2n}
\end{align}
Substituting these two identities gives the result.
\end{proof}

\subsubsection{Generative model for 2FS kernels} \label{appendix:2fs-generative-model}

Here we demonstrate one sufficient generative model for \(\boldsymbol{\Sigma}_x\) and \(\boldsymbol{\Sigma}_y\) to satisfy \autoref{assumption:2fs}. The same argument applies to inputs and targets, so write a generic vector \(\mathbf{r}(s,c)\in\mathbb{R}^d\), which can be either \(\mathbf{x}(s,c)\) or \(\mathbf{y}(s,c)\). Suppose first that samples in the same fine class share the same vector:
\begin{align}
\mathbf{r}^{(i)}
=
\mathbf{r}(s_i,c_i)
\end{align}
By \autoref{proposition:walsh-expansion},
\begin{align}
\mathbf{r}(s,c)
=
\mathbf{r}_\varnothing
+s\mathbf{r}_S
+c\mathbf{r}_C
+sc\mathbf{r}_{SC}
\end{align}
The term \(\mathbf{r}_S\) is the additive shape component, \(\mathbf{r}_C\) is the additive color component, and \(\mathbf{r}_{SC}\) is the shape--color interaction component. The interaction component is exactly the part that makes the shape direction context-dependent: for a fixed color \(c\),
\begin{align}
\mathbf{r}(+1,c)-\mathbf{r}(-1,c)
=
2\mathbf{r}_S+2c\mathbf{r}_{SC}
\end{align}
Thus, if \(\mathbf{r}_{SC}=0\), the shape direction is the same in both colors. If \(\mathbf{r}_{SC}\neq0\), the shape direction changes with color. Define the kernel \(\mathbf{R}\in\mathbb{R}^{N\times N}\) with entries \(\mathbf{R}_{ij}=\langle \mathbf{r}^{(i)},\mathbf{r}^{(j)}\rangle\):
\begin{align}
\mathbf{R}_{ij}
=
\sum_{a,b\in\{\varnothing,S,C,SC\}}
\chi_a(s_i,c_i)\chi_b(s_j,c_j)
\langle \mathbf{r}_a,\mathbf{r}_b\rangle
\label{eq:appendix-general-walsh-gram}
\end{align}
Here \(\chi_\varnothing=1\), \(\chi_S=s\), \(\chi_C=c\), and \(\chi_{SC}=sc\). If the four Walsh components are mutually orthogonal,
\begin{align}
\langle \mathbf{r}_a,\mathbf{r}_b\rangle=0
\qquad
\text{for }a\neq b,
\end{align}
then all cross-terms in \autoref{eq:appendix-general-walsh-gram} vanish, and
\begin{align}
\mathbf{R}_{ij}
=
\|\mathbf{r}_\varnothing\|^2
+s_i s_j\|\mathbf{r}_S\|^2
+c_i c_j\|\mathbf{r}_C\|^2
+s_i c_i s_j c_j\|\mathbf{r}_{SC}\|^2
\label{eq:appendix-orthogonal-walsh-gram}
\end{align}
This depends only on whether the two samples agree in shape, agree in color, agree in both, or agree in neither. Equivalently,
\begin{align}
\mathbf{R}
=
N\|\mathbf{r}_\varnothing\|^2\mathbf{P}_G
+
N\|\mathbf{r}_S\|^2\mathbf{P}_S
+
N\|\mathbf{r}_C\|^2\mathbf{P}_C
+
N\|\mathbf{r}_{SC}\|^2\mathbf{P}_{SC}
\label{eq:appendix-latent-function-2fs-decomposition}
\end{align}
Thus \(\mathbf{R}\) is 2FS, with
\begin{align}
\lambda_G^{(\mathbf{R})}
&=
N\|\mathbf{r}_\varnothing\|^2,
&
\lambda_S^{(\mathbf{R})}
&=
N\|\mathbf{r}_S\|^2,
&
\lambda_C^{(\mathbf{R})}
&=
N\|\mathbf{r}_C\|^2,
&
\lambda_{SC}^{(\mathbf{R})}
&=
N\|\mathbf{r}_{SC}\|^2,
&
\lambda_I^{(\mathbf{R})}
&=
0
\end{align}
Applying this construction with \(\mathbf{r}=\mathbf{x}\) motivates a 2FS input kernel \(\boldsymbol{\Sigma}_x=\mathbf{X}^\top \mathbf{X}\). Applying it with \(\mathbf{r}=\mathbf{y}\) motivates a 2FS target kernel \(\boldsymbol{\Sigma}_y=\mathbf{Y}^\top \mathbf{Y}\). This derivation also explains the signal/noise language used in the main text. The inverse SNR \(\nu(\mathbf{R})\) is given by:
\begin{align}
\nu(\mathbf{R})
=
\frac{\lambda_{SC}^{(\mathbf{R})}}{\lambda_S^{(\mathbf{R})}}
=
\frac{\|\mathbf{r}_{SC}\|^2}{\|\mathbf{r}_S\|^2}
\end{align}
The \(S\)-mode measures the strength of the context-independent shape direction. The \(SC\)-mode measures the strength of the context-dependent correction to that direction. Therefore large \(\lambda_S^{(\mathbf{R})}\) improves shape abstraction, while large \(\lambda_{SC}^{(\mathbf{R})}\) reduces it. Indeed, if the class vectors themselves are used as representations, then
\begin{align}
\alpha_R
=
\frac{\|\mathbf{r}_S\|^2-\|\mathbf{r}_{SC}\|^2}
{\|\mathbf{r}_S\|^2+\|\mathbf{r}_{SC}\|^2},
\end{align}
which is the same formula as \autoref{proposition:2fs-abstraction-eigenvalues}.

\paragraph{Within-class variation and the \(I\) mode.}
The class-deterministic construction above has \(\lambda_I=0\). The full 2FS assumption is more permissive. It allows samples within the same fine class to vary, as long as the residual variation is exchangeable within each fine class and respects the same global relabeling symmetries. For example, suppose that via zero-mean iid noise terms $\boldsymbol{\epsilon}^{(s,c,a)}$, we have:
\begin{align}
\mathbf{r}^{(s,c,a)}
=
\mathbf{f}(s,c)+\boldsymbol{\epsilon}^{(s,c,a)},
\end{align}
where the latent component \(\mathbf{f}\) has the orthogonal Walsh structure above, and the residual Gram contributions
\begin{align}
\langle \boldsymbol{\epsilon}^{(s_i,c_i,a_i)},\boldsymbol{\epsilon}^{(s_j,c_j,a_j)}\rangle,
\qquad
\langle \mathbf{f}(s_i,c_i),\boldsymbol{\epsilon}^{(s_j,c_j,a_j)}\rangle
+
\langle \boldsymbol{\epsilon}^{(s_i,c_i,a_i)},\mathbf{f}(s_j,c_j)\rangle
\end{align}
have five entry types after class balancing or averaging. Then the total Gram kernel is still 2FS. The \(I\)-mode captures precisely the part of this variation that is invisible to class centroids: directions that fluctuate within a fine class but average to zero within that class. This is why \autoref{assumption:2fs} includes \(I\), even though abstraction itself depends only on the class-level \(S\) and \(SC\) modes.

\paragraph{The initial kernel is approximately 2FS under random initialization.}
Previously we motivated how \(\boldsymbol{\Sigma}_x\) and \(\boldsymbol{\Sigma}_y\) can be 2FS. Next we motivate how the feature/preactivation kernel \(\mathbf{Q}(0)\) might also also 2FS at initialization. In the linear network of \autoref{section:min-model},
\begin{align}
\mathbf{Z}(0)=\mathbf{W}_0\mathbf{X},
\qquad
\mathbf{Q}(0)=\mathbf{Z}(0)^\top \mathbf{Z}(0)=\mathbf{X}^\top \mathbf{W}_0^\top \mathbf{W}_0\mathbf{X}
\end{align}
If \(\mathbf{W}_0\) is isotropic and zero mean, then \(\mathbf{W}_0^\top \mathbf{W}_0\) is close to a scalar multiple of the identity at large width. Consequently \(\mathbf{Q}(0)\) is close to the same scalar multiple of \(\mathbf{X}^\top \mathbf{X}=\boldsymbol{\Sigma}_x\). The following proposition makes this concentration statement explicit in a simple Gaussian case.

\begin{proposition}[Inheritance of 2FS in expectation] \label{proposition:random-init-2fs}
Assume \(\mathbf{Z}(0)=\mathbf{W}_0\mathbf{X}\), where \(\mathbf{W}_0\in\mathbb{R}^{D\times D_x}\) has independent Gaussian rows
\begin{align}
\mathbf{w}_r
\sim
\mathcal{N}\left(0,\frac{\sigma_w^2}{D}\I_{D_x}\right),
\qquad
r=1,\ldots,D
\end{align}
Let \(\boldsymbol{\Sigma}_x=\mathbf{X}^\top \mathbf{X}\) and \(\mathbf{Q}(0)=\mathbf{Z}(0)^\top \mathbf{Z}(0)\). Conditional on \(\mathbf{X}\),
\begin{align}
\mathbb{E}\left[\mathbf{Q}(0)\mid \mathbf{X}\right]
=
\sigma_w^2\boldsymbol{\Sigma}_x
\end{align}
Therefore, if \(\boldsymbol{\Sigma}_x\) is 2FS, then \(\mathbb{E}[\mathbf{Q}(0)\mid \mathbf{X}]\) is 2FS. Moreover, for every pair \((i,j)\),
\begin{align}
\mathrm{Var}\left(\mathbf{Q}_{ij}(0)\mid \mathbf{X}\right)
=
\frac{\sigma_w^4}{D}
\left(
(\boldsymbol{\Sigma}_x)_{ii}(\boldsymbol{\Sigma}_x)_{jj}
+
(\boldsymbol{\Sigma}_x)_{ij}^2
\right)
\label{eq:appendix-q0-entry-variance}
\end{align}
If \(B\coloneqq \max_i(\boldsymbol{\Sigma}_x)_{ii}\), then for a universal constant \(c>0\),
\begin{align}
\mathbb{P}\left(
\max_{i,j}\left|\mathbf{Q}_{ij}(0)-\sigma_w^2(\boldsymbol{\Sigma}_x)_{ij}\right|
\ge t
\,\middle|\,
\mathbf{X}
\right)
\le
2N^2
\exp\left[
-c\min\left(
\frac{Dt^2}{\sigma_w^4B^2},
\frac{Dt}{\sigma_w^2B}
\right)
\right]
\label{eq:appendix-q0-entry-concentration}
\end{align}
Consequently, whenever \(\log N/D\) is small,
\begin{align}
\max_{i,j}\left|\mathbf{Q}_{ij}(0)-\sigma_w^2(\boldsymbol{\Sigma}_x)_{ij}\right|
=
O_p\left(\sigma_w^2B\sqrt{\frac{\log N}{D}}\right)
\end{align}
In particular, if \(\boldsymbol{\Sigma}_x\) is 2FS, then \(\mathbf{Q}(0)\) is approximately 2FS at large width.
\end{proposition}

\begin{proof}
Let \(\mathbf{z}_r\in\mathbb{R}^N\) be row \(r\) of \(\mathbf{Z}(0)\), so
\begin{align}
(\mathbf{z}_r)_i
=
\mathbf{w}_r^\top\mathbf{x}^{(i)}
\end{align}
Conditional on \(\mathbf{X}\), each \(\mathbf{z}_r\) is a centered Gaussian vector with covariance
\begin{align}
\mathbb{E}\left[\mathbf{z}_r\mathbf{z}_r^\top\mid \mathbf{X}\right]
=
\frac{\sigma_w^2}{D}\boldsymbol{\Sigma}_x
\end{align}
Since
\begin{align}
\mathbf{Q}(0)
=
\sum_{r=1}^D\mathbf{z}_r\mathbf{z}_r^\top,
\end{align}
linearity of expectation gives
\begin{align}
\mathbb{E}\left[\mathbf{Q}(0)\mid \mathbf{X}\right]
=
D\cdot\frac{\sigma_w^2}{D}\boldsymbol{\Sigma}_x
=
\sigma_w^2\boldsymbol{\Sigma}_x
\end{align}

For the variance, write
\begin{align}
\mathbf{C}
=
\frac{\sigma_w^2}{D}\boldsymbol{\Sigma}_x
\end{align}
For a single row \(\mathbf{z}_r\), Isserlis' theorem gives
\begin{align}
\mathbb{E}\left[(\mathbf{z}_r)_i(\mathbf{z}_r)_j(\mathbf{z}_r)_i(\mathbf{z}_r)_j\mid \mathbf{X}\right]
=
\mathbf{C}_{ii}\mathbf{C}_{jj}+2\mathbf{C}_{ij}^2
\end{align}
Therefore
\begin{align}
\mathrm{Var}\left((\mathbf{z}_r)_i(\mathbf{z}_r)_j\mid \mathbf{X}\right)
=
\mathbf{C}_{ii}\mathbf{C}_{jj}+\mathbf{C}_{ij}^2
\end{align}
The rows are independent, so variances add across \(r=1,\ldots,D\), yielding \autoref{eq:appendix-q0-entry-variance}. Finally, each centered product
\begin{align}
(\mathbf{z}_r)_i(\mathbf{z}_r)_j
-
\mathbb{E}\left[(\mathbf{z}_r)_i(\mathbf{z}_r)_j\mid \mathbf{X}\right]
\end{align}
is sub-exponential with scale at most a constant multiple of \(\sigma_w^2B/D\), since \(\mathbf{C}_{ii}\le \sigma_w^2B/D\) for all \(i\). Bernstein's inequality gives the displayed tail bound for each fixed \((i,j)\), and a union bound over the \(N^2\) pairs gives \autoref{eq:appendix-q0-entry-concentration}.
\end{proof}

There is also a useful distributional way to state the same point. If \(\boldsymbol{\Sigma}_x\) is 2FS, then for every \(g\in\mathcal{G}\), \(\boldsymbol{\Pi}_g\Sigma_x \boldsymbol{\Pi}_g^\top=\boldsymbol{\Sigma}_x\). Therefore each Gaussian row \(\mathbf{z}_r\) has the same distribution as \(\boldsymbol{\Pi}_g\mathbf{z}_r\), and hence
\begin{align}
\boldsymbol{\Pi}_g \mathbf{Q}(0)\boldsymbol{\Pi}_g^\top
\stackrel{d}{=}
\mathbf{Q}(0)
\end{align}
Thus random initialization does not prefer any particular element of a symmetry orbit. A finite-width draw can still break the symmetry by sampling noise, but the preceding concentration bound shows that this symmetry-breaking component is small at large width.

To make the last sentence precise, let \(\Pi_{\mathrm{2FS}}\) denote group averaging over \(\mathcal{G}\):
\begin{align}
\Pi_{\mathrm{2FS}}(A)
=
\frac{1}{|\mathcal{G}|}\sum_{g\in\mathcal{G}}\boldsymbol{\Pi}_gA\boldsymbol{\Pi}_g^\top
\end{align}
This is the orthogonal projection onto the 2FS subspace of matrices. If \(\boldsymbol{\Sigma}_x\) is 2FS, then \(\Pi_{\mathrm{2FS}}(\sigma_w^2\boldsymbol{\Sigma}_x)=\sigma_w^2\boldsymbol{\Sigma}_x\). Hence
\begin{align}
\|\mathbf{Q}(0)-\Pi_{\mathrm{2FS}}(\mathbf{Q}(0))\|_{\max}
&\le
\|\mathbf{Q}(0)-\sigma_w^2\boldsymbol{\Sigma}_x\|_{\max}
+
\|\Pi_{\mathrm{2FS}}(\mathbf{Q}(0)-\sigma_w^2\boldsymbol{\Sigma}_x)\|_{\max}
\\
&\le
2\|\mathbf{Q}(0)-\sigma_w^2\boldsymbol{\Sigma}_x\|_{\max}
\end{align}
The concentration bound above therefore implies
\begin{align}
\|\mathbf{Q}(0)-\Pi_{\mathrm{2FS}}(\mathbf{Q}(0))\|_{\max}
=
O_p\left(\sigma_w^2B\sqrt{\frac{\log N}{D}}\right)
\end{align}
Thus the non-2FS component of \(\mathbf{Q}(0)\) vanishes in the large-width limit, provided \(N\) is fixed or grows slowly enough relative to \(D\). In that limit,
\begin{align}
\mathbf{Q}(0)
\approx
\sigma_w^2\boldsymbol{\Sigma}_x, 
\qquad
\lambda_m^{(\mathbf{Q}(0))}
\approx
\sigma_w^2\lambda_m^{(\boldsymbol{\Sigma}_x)},
\qquad
\nu(\mathbf{Q}(0))
\approx
\nu(\boldsymbol{\Sigma}_x)
\end{align}

\clearpage
\subsection{Variable-projected readout (\autoref{assumption:readout})} \label{appendix:vp-readout}

\begin{figure}[h]
    \centering
    \includegraphics[width=1.0\linewidth]{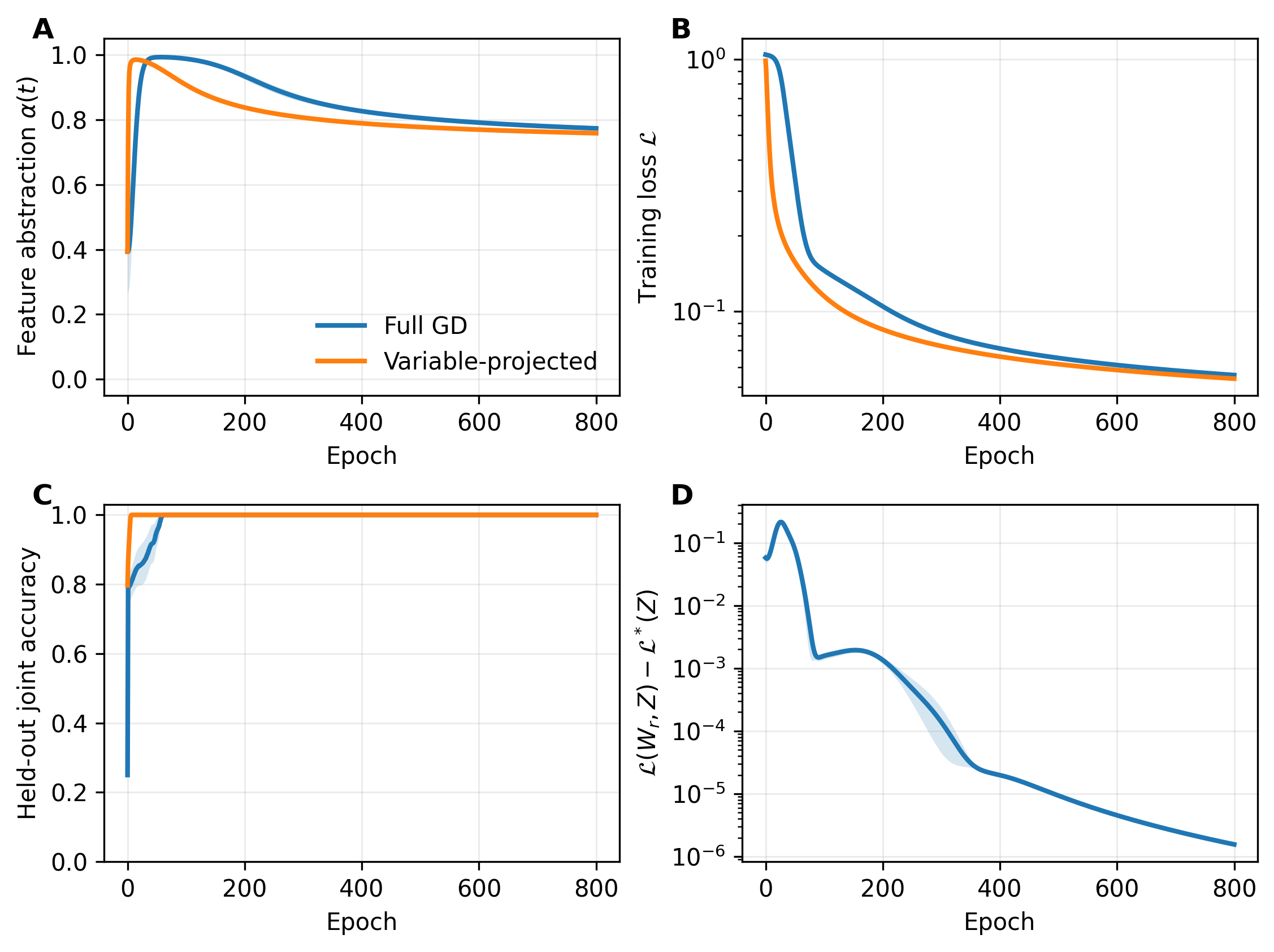}
    \caption{
    \textbf{Variable-projected readout versus ordinary equal-timescale readout training.}
    We compare full-batch gradient descent on the full loss \(\mathcal{L}(\mathbf{W}_r,\mathbf{Z})\) against gradient descent on the reduced loss \(\mathcal{L}^*(\mathbf{Z}) = \min_{\mathbf{W}_r}\mathcal{L}(\mathbf{W}_r,\mathbf{Z})\), corresponding to \autoref{assumption:readout}. Curves show mean \(\pm\) SEM over \(12\) random seeds.
    \textbf{A.} Shape abstraction \(\alpha(t)\).
    \textbf{B.} Training loss.
    \textbf{C.} Test accuracy.
    \textbf{D.} Loss gap for the full-gradient-descent baseline. The variable-projected method has zero loss gap by construction.
    }
    \label{fig:vp-readout-dynamics}
\end{figure}

In this section we analyze \autoref{assumption:readout} and compare it with the case of full gradient descent in which both \(\mathbf{W}\) and \(\mathbf{W}_r\) are learned. \autoref{assumption:readout} was motivated from canonical work on variable projection \cite{golubSeparableNonlinear03, kimTrainingTwoLayered08} as well as recent work on two-timescale training which study regimes where the readout weights and feature weights have different learning rates \cite{barboniUltrafastFeature25,marionLeveragingTwo23}. We argue that this is a less restrictive and more realistic assumption than the fixed-readout simplification often used in classical statistical mechanical analyses of neural network dynamics \cite{saadOnlineLearning95a} or recent work on dynamics of two-layer ReLU networks \cite{aroraFineGrainedAnalysis19}. Still, this approximation is not exact and so we compare it against realistic dynamics. Specifically, we compare two training procedures. The first is full gradient descent on the complete loss
\[
    \mathcal{L}(\mathbf{W}_r,\mathbf{Z})
    =
    \frac{1}{2}\|\mathbf{W}_r \mathbf{Z} - \mathbf{Y}\|_F^2
    +
    \frac{1}{2\gamma}\|\mathbf{W}_r\|_F^2,
    \qquad \mathbf{Z} = WX,
\]
where both \(\mathbf{W}\) and \(\mathbf{W}_r\) are updated with the same learning rate. The second is gradient descent on the simplified objective where at every training step the readout is set to its ridge optimum
\[
    \mathcal{L}^*(\mathbf{Z})
    =
    \frac{1}{2} \mathrm{Tr} \pbracket{ \mathbf{Y}^\top \mathbf{Y} (\I_N + \gamma \mathbf{Q}(t))^{-1}}
\]
We perform simulations use a balanced synthetic two-factor dataset with latents \((s,c) \in \{\pm 1\}^2\). For each class, inputs are generated as
\[
    \mathbf{x}
    =
    a_s s \mathbf{u}_s
    +
    a_c c \mathbf{u}_c
    +
    a_{sc} sc\, \mathbf{u}_{sc}
    +
    \sigma \epsilon / \sqrt{D_x},
\]
where \(\mathbf{u}_s,\mathbf{u}_c,\mathbf{u}_{sc}\) are orthonormal directions in \(\mathbb{R}^{D_x}\), generated independently for each seed. Targets are
\[
    \mathbf{y} = [s,c,\rho sc]^\top
\]
We used \(D_x = 20\), \(a_s = 1.0\), \(a_c = 0.7\), \(a_{sc} = 0.6\), \(\sigma = 0.6\), and \(\rho = 0.3\). The training set contained \(64\) examples per class, for \(N_{\mathrm{train}} = 256\), and the held-out test set contained \(256\) examples per class, for \(N_{\mathrm{test}} = 1024\). The test set used the same latent directions as the corresponding training set but fresh noise. We scaled the data matrices by the square root of the number of samples so that the unnormalized loss remained \(O(1)\).

The network had hidden dimension \(d = 4\), with \(\mathbf{W} \in \mathbb{R}^{4 \times 20}\) and \(\mathbf{W}_r \in \mathbb{R}^{3 \times 4}\). Feature weights were initialized as Gaussian entries with standard deviation \(0.2/\sqrt{20}\), while the readout was initialized at zero. We used ridge parameter \(\gamma = 10\), learning rate \(0.1\), and trained for \(800\) full-batch updates. Both methods used the same dataset and the same initialization of \(\mathbf{W}\) for each seed. Results are averaged over \(12\) random seeds, with shaded regions denoting the standard error of the mean. For each method, we report the feature abstraction
\[
    \alpha(t)
    =
    \frac{
        \mathbf{u}_{\mathrm{red}}(t)^\top \mathbf{u}_{\mathrm{blue}}(t)
    }{
        \|\mathbf{u}_{\mathrm{red}}(t)\|\,
        \|\mathbf{u}_{\mathrm{blue}}(t)\|
    },
\]
where \(\mathbf{u}_{\mathrm{red}}\) and \(\mathbf{u}_{\mathrm{blue}}\) are the shape directions in the two color contexts. We also report the training loss, held-out joint accuracy on the shape and color coordinates, and the loss gap
\[
    \mathcal{L}(\mathbf{W}_r(t),\mathbf{Z}(t))
    -
    \mathcal{L}^*(\mathbf{Z}(t))
\]
For the variable-projected dynamics this gap is zero by construction, so the loss gap is shown only for the full-gradient-descent baseline. The results in \autoref{fig:vp-readout-dynamics} show that \autoref{assumption:readout} mainly differs from full gradient descent earlier in training. At initialization, the full-gradient-descent model has a zero readout and must spend early training steps learning a useful linear map from features to targets. In contrast, the variable-projected model immediately uses the best ridge readout for the current features. Consequently, the loss for the simplified dynamics improves faster in the earlier part of training. The abstraction trajectory also differs early on: abstraction peaks for variable projection. However, the discrepancy becomes much smaller once the full-gradient-descent readout starts to equilibrate, and both loss and abstraction levels become close later on in training. These simulations support \autoref{assumption:readout} as a useful approximation for later feature dynamics, but not as a faithful model of the earliest readout transient.

\clearpage
\section{Derivations and proofs for results in \autoref{section:min-model} (linear networks)}

\subsection{Probe generalization error} \label{appendix:probe-generalization}

This appendix proves and discusses \autoref{proposition:ccgp}. The proposition is a formal version of the following geometric statement: a shape probe trained in one color context transfers to the other color context precisely to the extent that the shape directions in the two contexts are aligned. In the notation of \autoref{section:min-model}, the red-context shape direction is
\begin{align}
    \mathbf{u}_\mathrm{red}
    \coloneqq
    \boldsymbol{\mu}_{\redsquare}
    -
    \boldsymbol{\mu}_{\redcircle},
\end{align}
and the blue-context shape direction is
\begin{align}
    \mathbf{u}_\mathrm{blue}
    \coloneqq
    \boldsymbol{\mu}_{\bluesquare}
    -
    \boldsymbol{\mu}_{\bluecircle}
\end{align}
Their cosine similarity is the abstraction score
\begin{align}
    \alpha
    =
    \frac{
        \mathbf{u}_\mathrm{red}^{\top}
        \mathbf{u}_\mathrm{blue}
    }{
        \|\mathbf{u}_\mathrm{red}\|
        \|\mathbf{u}_\mathrm{blue}\|
    }
\end{align}
Thus \(\alpha = 1\) means that the shape direction is perfectly shared across red and blue contexts, \(\alpha = 0\) means that the red-context probe is orthogonal to the blue-context shape direction, and \(\alpha < 0\) means that the red-context probe points in the wrong direction when transferred to blue samples.

\begin{proof}[Proof of \autoref{proposition:ccgp}]
We prove the result for the non-degenerate case
\begin{align}
    \|\mathbf{u}_\mathrm{red}\| > 0,
    \qquad
    \|\mathbf{u}_\mathrm{blue}\| > 0,
    \qquad
    \sigma_\mathrm{blue} > 0
\end{align}
Degenerate cases are obtained as limits. The probe trained on the red samples is the unit vector
\begin{align}
    \widehat{\mathbf{w}}
    =
    \frac{
        \mathbf{u}_\mathrm{red}
    }{
        \|\mathbf{u}_\mathrm{red}\|
    }
\end{align}
We transfer this probe to the blue samples. Let \(y \in \{+1,-1\}\) denote the blue shape label, with \(y=+1\) corresponding to \(\bluesquare\) and \(y=-1\) corresponding to \(\bluecircle\). As in \autoref{proposition:ccgp}, we center each blue sample by the midpoint of the two blue class means:
\begin{align}
    \widetilde{\mathbf{z}}_\mathrm{blue}
    =
    \mathbf{z}_\mathrm{blue}
    -
    \frac{1}{2}
    \left(
        \boldsymbol{\mu}_{\bluesquare}
        +
        \boldsymbol{\mu}_{\bluecircle}
    \right).
\end{align}

\paragraph{Step 1: compute the centered blue class means.}
For a blue square sample, the mean of the centered representation is
\begin{align}
    \boldsymbol{\mu}_{\bluesquare}
    -
    \frac{1}{2}
    \left(
        \boldsymbol{\mu}_{\bluesquare}
        +
        \boldsymbol{\mu}_{\bluecircle}
    \right)
    &=
    \frac{1}{2}
    \left(
        \boldsymbol{\mu}_{\bluesquare}
        -
        \boldsymbol{\mu}_{\bluecircle}
    \right)
    \\
    &=
    \frac{1}{2}
    \mathbf{u}_\mathrm{blue}
\end{align}
Similarly, for a blue circle sample,
\begin{align}
    \boldsymbol{\mu}_{\bluecircle}
    -
    \frac{1}{2}
    \left(
        \boldsymbol{\mu}_{\bluesquare}
        +
        \boldsymbol{\mu}_{\bluecircle}
    \right)
    &=
    -
    \frac{1}{2}
    \left(
        \boldsymbol{\mu}_{\bluesquare}
        -
        \boldsymbol{\mu}_{\bluecircle}
    \right)
    \\
    &=
    -
    \frac{1}{2}
    \mathbf{u}_\mathrm{blue}
\end{align}
Therefore the centered blue samples have class means
\begin{align}
    \mathbb{E}
    \left[
        \widetilde{\mathbf{z}}_\mathrm{blue}
        \mid
        y
    \right]
    =
    y
    \frac{
        \mathbf{u}_\mathrm{blue}
    }{
        2
    }
\end{align}

\paragraph{Step 2: project the blue class means onto the red probe.}
Define the one-dimensional transferred probe score
\begin{align}
    T
    \coloneqq
    \widehat{\mathbf{w}}^{\top}
    \widetilde{\mathbf{z}}_\mathrm{blue}
\end{align}
The projected blue class mean is
\begin{align}
    \mathbb{E}
    \left[
        T
        \mid
        y
    \right]
    &=
    \widehat{\mathbf{w}}^{\top}
    \mathbb{E}
    \left[
        \widetilde{\mathbf{z}}_\mathrm{blue}
        \mid
        y
    \right]
    \\
    &=
    y
    \frac{
        \widehat{\mathbf{w}}^{\top}
        \mathbf{u}_\mathrm{blue}
    }{
        2
    }
\end{align}
Now substitute the definition of the red probe:
\begin{align}
    \widehat{\mathbf{w}}^{\top}
    \mathbf{u}_\mathrm{blue}
    &=
    \frac{
        \mathbf{u}_\mathrm{red}^{\top}
        \mathbf{u}_\mathrm{blue}
    }{
        \|\mathbf{u}_\mathrm{red}\|
    }
    \\
    &=
    \|\mathbf{u}_\mathrm{blue}\|
    \frac{
        \mathbf{u}_\mathrm{red}^{\top}
        \mathbf{u}_\mathrm{blue}
    }{
        \|\mathbf{u}_\mathrm{red}\|
        \|\mathbf{u}_\mathrm{blue}\|
    }
    \\
    &=
    \|\mathbf{u}_\mathrm{blue}\| \alpha
\end{align}
Thus the projected blue class means are
\begin{align}
    \mathbb{E}
    \left[
        T
        \mid
        y
    \right]
    =
    y
    \frac{
        \|\mathbf{u}_\mathrm{blue}\|\alpha
    }{
        2
    }
\end{align}
The key point is that \(\alpha\) directly multiplies the transferred margin. Holding \(\|\mathbf{u}_\mathrm{blue}\|\) fixed, increasing \(\alpha\) increases the separation of the two blue classes along the red probe direction.

\paragraph{Step 3: apply the Gaussian score approximation.}
Write a centered blue sample as
\begin{align}
    \widetilde{\mathbf{z}}_\mathrm{blue}
    =
    y
    \frac{
        \mathbf{u}_\mathrm{blue}
    }{
        2
    }
    +
    \boldsymbol{\varepsilon},
    \qquad
    \mathbb{E}
    \left[
        \boldsymbol{\varepsilon}
        \mid
        y
    \right]
    =
    0
\end{align}
The variance term in \autoref{proposition:ccgp} should be read as the projected within-class, or pooled within-class, blue variance:
\begin{align}
    \sigma_\mathrm{blue}^{2}
    =
    \widehat{\mathbf{w}}^{\top}
    \boldsymbol{\Sigma}_\mathrm{blue}
    \widehat{\mathbf{w}}
\end{align}
Equivalently,
\begin{align}
    \mathrm{Var}
    \left(
        \widehat{\mathbf{w}}^{\top}
        \boldsymbol{\varepsilon}
        \mid
        y
    \right)
    =
    \sigma_\mathrm{blue}^{2}
\end{align}
Under the Gaussian score approximation, the projected score is therefore
\begin{align}
    T
    \mid
    y
    \approx
    \mathcal{N}
    \left(
        y
        \frac{
            \|\mathbf{u}_\mathrm{blue}\|\alpha
        }{
            2
        },
        \sigma_\mathrm{blue}^{2}
    \right),
    \qquad
    y \in \{+1,-1\}
\end{align}

\paragraph{Step 4: compute the zero-threshold transfer error.}
The transferred classifier predicts
\begin{align}
    \widehat{y}
    =
    \operatorname{sign}(T)
\end{align}
For the \(y=+1\) blue class, an error occurs when \(T < 0\). Therefore
\begin{align}
    \mathbb{P}
    \left(
        \widehat{y} \neq y
        \mid
        y = +1
    \right)
    &=
    \mathbb{P}
    \left(
        T < 0
        \mid
        y = +1
    \right)
    \\
    &\approx
    \Phi
    \left(
        \frac{
            0
            -
            \frac{1}{2}
            \|\mathbf{u}_\mathrm{blue}\|\alpha
        }{
            \sigma_\mathrm{blue}
        }
    \right)
    \\
    &=
    \Phi
    \left(
        -
        \frac{
            \|\mathbf{u}_\mathrm{blue}\|\alpha
        }{
            2\sigma_\mathrm{blue}
        }
    \right)
\end{align}
For the \(y=-1\) blue class, an error occurs when \(T > 0\). Therefore
\begin{align}
    \mathbb{P}
    \left(
        \widehat{y} \neq y
        \mid
        y = -1
    \right)
    &=
    \mathbb{P}
    \left(
        T > 0
        \mid
        y = -1
    \right)
    \\
    &\approx
    1
    -
    \Phi
    \left(
        \frac{
            0
            +
            \frac{1}{2}
            \|\mathbf{u}_\mathrm{blue}\|\alpha
        }{
            \sigma_\mathrm{blue}
        }
    \right)
    \\
    &=
    \Phi
    \left(
        -
        \frac{
            \|\mathbf{u}_\mathrm{blue}\|\alpha
        }{
            2\sigma_\mathrm{blue}
        }
    \right),
\end{align}
where we used \(1-\Phi(a)=\Phi(-a)\). The two conditional errors are the same. Hence the total blue transfer error is
\begin{align}
    P_{\mathrm{error}}^{\mathrm{red}\to\mathrm{blue}}
    &\approx
    \Phi
    \left(
        -
        \frac{
            \|\mathbf{u}_\mathrm{blue}\|\alpha
        }{
            2\sigma_\mathrm{blue}
        }
    \right)
    \\
    &=
    \Phi
    \left(
        -
        \frac{
            \|\mathbf{u}_\mathrm{blue}\|
        }{
            2\sigma_\mathrm{blue}
        }
        \cdot
        \alpha
    \right)
\end{align}
This is the expression claimed in \autoref{proposition:ccgp}.

\paragraph{Step 5: show monotonic improvement with abstraction.}
Now hold \(\|\mathbf{u}_\mathrm{blue}\|\) and \(\sigma_\mathrm{blue}\) fixed, and define the positive constant
\begin{align}
    c
    \coloneqq
    \frac{
        \|\mathbf{u}_\mathrm{blue}\|
    }{
        2\sigma_\mathrm{blue}
    }
    >
    0
\end{align}
Then
\begin{align}
    P_{\mathrm{error}}^{\mathrm{red}\to\mathrm{blue}}
    \approx
    \Phi(-c\alpha)
\end{align}
Let \(\varphi = \Phi'\) denote the standard normal density. Differentiating with respect to \(\alpha\) gives
\begin{align}
    \frac{
        \partial
        P_{\mathrm{error}}^{\mathrm{red}\to\mathrm{blue}}
    }{
        \partial \alpha
    }
    &=
    -c
    \varphi(-c\alpha)
    \\
    &=
    -c
    \varphi(c\alpha)
    \\
    &<
    0,
\end{align}
because \(c>0\) and \(\varphi(r)>0\) for all \(r \in \mathbb{R}\). Thus, holding all other quantities fixed, increasing abstraction strictly decreases the transferred probe error. This proves \autoref{proposition:ccgp}.
\end{proof}

\paragraph{Interpretation.}
The proof shows that abstraction controls the transferred probe margin. The relevant normalized margin is
\begin{align}
    \frac{
        \text{projected blue class separation}
    }{
        \text{projected blue noise}
    }
    =
    \frac{
        \|\mathbf{u}_\mathrm{blue}\|\alpha
    }{
        2\sigma_\mathrm{blue}
    }
\end{align}
Thus \(\alpha\) directly influences the probability of probe failure. When \(\alpha = 0\), the red probe is orthogonal to the blue shape direction, so the two blue classes have the same projected mean and the error is approximately \(\Phi(0)=1/2\). When \(\alpha > 0\), the red probe points at least partially in the correct blue shape direction, producing better-than-chance transfer. When \(\alpha < 0\), the red probe points in the opposite direction after transfer, producing worse-than-chance error because the sign of the classifier is inherited from the red context.

Note that the transferred error can also change if the blue class separation \(\|\mathbf{u}_\mathrm{blue}\|\) changes, or if the projected within-class variance \(\sigma_\mathrm{blue}^{2}\) changes. \autoref{proposition:ccgp} isolates the effect of abstraction itself: for fixed blue-context signal strength and fixed projected noise, larger \(\alpha\) gives a larger positive margin and therefore smaller probe generalization error.

\clearpage
\subsection{Reduced kernel dynamics} \label{appendix:reduced-kernel-dynamics}

In this subsection we derive the reduced feature dynamics used in
\autoref{section:exact-solutions} and then prove \autoref{proposition:linear-odes}.
The main idea is that, after optimizing out the readout, the targets enter the
feature dynamics through a single \(N\times N\) matrix
\[
\mathbf{M}(\mathbf{Q})
\coloneqq
\gamma
(\I_N+\gamma \mathbf{Q})^{-1}
\boldsymbol{\Sigma}_y
(\I_N+\gamma \mathbf{Q})^{-1},
\qquad
\mathbf{Q} \coloneqq \mathbf{Z}^\top \mathbf{Z},
\qquad
\boldsymbol{\Sigma}_y \coloneqq \mathbf{Y}^\top \mathbf{Y}
\]
This matrix is the effective target kernel. It is ``effective'' because it is
not the raw target kernel \(\boldsymbol{\Sigma}_y\), but rather a spectrally filtered version of \(\boldsymbol{\Sigma}_y\) that depends on the current feature kernel \(\mathbf{Q}\). The dynamics of \(\mathbf{Q}\) are given by a Riccati equation driven by \(\mathbf{M}(\mathbf{Q})\).

\paragraph{Variable projection over the readout.}

Recall the original objective:
\[
\Lc(\mathbf{W}_r,\mathbf{Z})
=
\frac{1}{2}\|\mathbf{W}_r\mathbf{Z}-\mathbf{Y}\|_F^2
+
\frac{1}{2\gamma}\|\mathbf{W}_r\|_F^2
\]
Now we apply \autoref{assumption:readout}. For fixed \(\mathbf{Z}\), the first-order condition for \(\mathbf{W}_r\) is
\begin{align}
0
&=
(\mathbf{W}_r\mathbf{Z}-\mathbf{Y})\mathbf{Z}^\top+\gamma^{-1}\mathbf{W}_r
\\
&=
\mathbf{W}_r(\mathbf{Z}\mathbf{Z}^\top+\gamma^{-1}\I_D)-\mathbf{Y}\mathbf{Z}^\top
\end{align}
Therefore
\begin{align}
\mathbf{W}_r^*(\mathbf{Z})
&=
\mathbf{Y}\mathbf{Z}^\top(\mathbf{Z}\mathbf{Z}^\top+\gamma^{-1}\I_D)^{-1}
\end{align}
Using
\[
\mathbf{Z}^\top(\I_D+\gamma \mathbf{Z}\mathbf{Z}^\top)^{-1}
=
(\I_N+\gamma \mathbf{Z}^\top \mathbf{Z})^{-1}\mathbf{Z}^\top,
\]
we can also write this optimum in the sample-space form
\begin{align}
\mathbf{W}_r^*(\mathbf{Z})
&=
\gamma
\mathbf{Y}(\I_N+\gamma \mathbf{Q})^{-1}\mathbf{Z}^\top,
\qquad
\mathbf{Q}=\mathbf{Z}^\top \mathbf{Z}
\label{eq:appendix-wr-star}
\end{align}
This is the expression stated in \autoref{assumption:readout}.

Let
\[
\mathbf{B}(\mathbf{Q})\coloneqq(\I_N+\gamma \mathbf{Q})^{-1}
\]
Then, at the readout optimum,
\begin{align}
\mathbf{W}_r^*\mathbf{Z}
&=
\gamma \mathbf{Y}\mathbf{B}(\mathbf{Q})\mathbf{Q}
\end{align}
Since \(\mathbf{B}(\mathbf{Q})\) is a function of \(\mathbf{Q}\), it commutes with \(\mathbf{Q}\), and
\[
\mathbf{B}(\mathbf{Q})+\gamma \mathbf{B}(\mathbf{Q})\mathbf{Q}=\I_N
\]
Thus
\begin{align}
\mathbf{W}_r^*\mathbf{Z}-\mathbf{Y}
&=
\gamma \mathbf{Y}\mathbf{B}(\mathbf{Q})\mathbf{Q}-\mathbf{Y}
=
-\mathbf{Y}\mathbf{B}(\mathbf{Q})
\end{align}
The data-fit term becomes
\begin{align}
\frac{1}{2}\|\mathbf{W}_r^*\mathbf{Z}-\mathbf{Y}\|_F^2
&=
\frac{1}{2}
\Tr\!\left(
\mathbf{Y}\mathbf{B}(\mathbf{Q})^2\mathbf{Y}^\top
\right)
=
\frac{1}{2}
\Tr\!\left(
\boldsymbol{\Sigma}_y \mathbf{B}(\mathbf{Q})^2
\right)
\end{align}
The ridge term is
\begin{align}
\frac{1}{2\gamma}\|\mathbf{W}_r^*\|_F^2
&=
\frac{1}{2\gamma}
\Tr\!\left(
\gamma^2 \mathbf{Y}\mathbf{B}(\mathbf{Q})\mathbf{Z}^\top \mathbf{Z}\mathbf{B}(\mathbf{Q})\mathbf{Y}^\top
\right)
\\
&=
\frac{\gamma}{2}
\Tr\!\left(
\boldsymbol{\Sigma}_y \mathbf{B}(\mathbf{Q})\mathbf{Q}\mathbf{B}(\mathbf{Q})
\right)
\end{align}
Adding the two terms gives
\begin{align}
\Lc^*(\mathbf{Z})
&\coloneqq
\Lc(\mathbf{W}_r^*(\mathbf{Z}),\mathbf{Z})
\\
&=
\frac{1}{2}
\Tr\!\left(
\boldsymbol{\Sigma}_y \mathbf{B}(\mathbf{Q})^2
\right)
+
\frac{\gamma}{2}
\Tr\!\left(
\boldsymbol{\Sigma}_y \mathbf{B}(\mathbf{Q})\mathbf{Q}\mathbf{B}(\mathbf{Q})
\right)
\\
&=
\frac{1}{2}
\Tr\!\left(
\boldsymbol{\Sigma}_y \mathbf{B}(\mathbf{Q})\bigl(\mathbf{B}(\mathbf{Q})+\gamma \mathbf{Q}\mathbf{B}(\mathbf{Q})\bigr)
\right)
\\
&=
\frac{1}{2}
\Tr\!\left(
\boldsymbol{\Sigma}_y \mathbf{B}(\mathbf{Q})
\right)
\\
&=
\frac{1}{2}
\Tr\!\left(
\mathbf{Y}^\top \mathbf{Y}(\I_N+\gamma \mathbf{Q})^{-1}
\right).
\label{eq:appendix-reduced-objective}
\end{align}
This proves the reduced objective stated in the main text.

\paragraph{Effective target kernel.}

We now differentiate \(\Lc^*(\mathbf{Z})\). Since
\[
\mathbf{B}(\mathbf{Q})=(\I_N+\gamma \mathbf{Q})^{-1},
\]
its differential is
\[
d\mathbf{B}
=
-\gamma \mathbf{B}(d\mathbf{Q})\mathbf{B}
\]
Therefore
\begin{align}
d\Lc^*
&=
\frac{1}{2}
\Tr\!\left(
\boldsymbol{\Sigma}_y\,d\mathbf{B}
\right)
\\
&=
-\frac{\gamma}{2}
\Tr\!\left(
\boldsymbol{\Sigma}_y \mathbf{B}(d\mathbf{Q})\mathbf{B}
\right)
\\
&=
-\frac{1}{2}
\Tr\!\left(
\gamma \mathbf{B}\boldsymbol{\Sigma}_y \mathbf{B}\,d\mathbf{Q}
\right)
\end{align}
This motivates the definition of the effective target kernel from the main text:
\begin{align}
\mathbf{M}(\mathbf{Q})
&\coloneqq
\gamma \mathbf{B}(\mathbf{Q})\boldsymbol{\Sigma}_y \mathbf{B}(\mathbf{Q})
=
\gamma
(\I_N+\gamma \mathbf{Q})^{-1}
\boldsymbol{\Sigma}_y
(\I_N+\gamma \mathbf{Q})^{-1}
\label{eq:appendix-effective-target}
\end{align}
Equivalently,
\[
\nabla_{\mathbf{Q}} \Lc^*(\mathbf{Q})
=
-\frac{1}{2}\mathbf{M}(\mathbf{Q})
\]
Because \(\mathbf{Q}=\mathbf{Z}^\top \mathbf{Z}\),
\[
d\mathbf{Q}=d\mathbf{Z}^\top \mathbf{Z}+\mathbf{Z}^\top d\mathbf{Z}
\]
Using the symmetry of \(\mathbf{M}(\mathbf{Q})\), we obtain
\begin{align}
d\Lc^*
&=
-\frac{1}{2}
\Tr\!\left(
\mathbf{M}(\mathbf{Q})(d\mathbf{Z}^\top \mathbf{Z}+\mathbf{Z}^\top d\mathbf{Z})
\right)
\\
&=
-\Tr\!\left(
(\mathbf{Z}\mathbf{M}(\mathbf{Q}))^\top d\mathbf{Z}
\right)
\end{align}
Hence
\[
\nabla_{\mathbf{Z}}\Lc^*(\mathbf{Z})=-\mathbf{Z}\mathbf{M}(\mathbf{Q})
\]

The learned feature matrix is \(\mathbf{Z}=\mathbf{W}\mathbf{X}\). Therefore
\[
\nabla_{\mathbf{W}}\Lc^*
=
\nabla_{\mathbf{Z}}\Lc^*\,\mathbf{X}^\top
=
-\mathbf{Z}\mathbf{M}(\mathbf{Q})\mathbf{X}^\top
\]
Gradient flow on \(\mathbf{W}\) gives
\[
\dot{\mathbf{W}}=-\nabla_{\mathbf{W}}\Lc^*
=
\mathbf{Z}\mathbf{M}(\mathbf{Q})\mathbf{X}^\top
\]
Multiplying by \(\mathbf{X}\), the feature dynamics are
\begin{align}
\dot{\mathbf{Z}}
&=
\dot{\mathbf{W}} \mathbf{X}
=
\mathbf{Z}\mathbf{M}(\mathbf{Q})\mathbf{X}^\top \mathbf{X}
=
\mathbf{Z}\mathbf{M}(\mathbf{Q})\boldsymbol{\Sigma}_x,
\label{eq:appendix-feature-flow}
\end{align}
where \(\boldsymbol{\Sigma}_x\coloneqq \mathbf{X}^\top \mathbf{X}\). Differentiating \(\mathbf{Q}=\mathbf{Z}^\top \mathbf{Z}\), we get the following matrix Riccati equation for the dynamics of the kernel \(\mathbf{Q}\) which is mentioned at the beginning of \autoref{section:exact-solutions}:
\begin{align}
\dot{\mathbf{Q}}
&=
\dot{\mathbf{Z}}^\top \mathbf{Z}+\mathbf{Z}^\top \dot{\mathbf{Z}}
\\
&=
\boldsymbol{\Sigma}_x \mathbf{M}(\mathbf{Q})\mathbf{Q}+\mathbf{Q}\mathbf{M}(\mathbf{Q})\boldsymbol{\Sigma}_x
\label{eq:appendix-linear-riccati}
\end{align}

\paragraph{Spectral filtering by the effective target kernel.}

Under \autoref{assumption:2fs}, the matrices \(\mathbf{Q}\), \(\boldsymbol{\Sigma}_x\), and
\(\boldsymbol{\Sigma}_y\) are diagonalizable in the same 2FS basis. Let
\(\{\mathbf{P}_m\}_{m\in\{I,S,C,SC,G\}}\) denote the common orthogonal projectors onto
the 2FS modes. Then
\[
\mathbf{Q}
=
\sum_m \lambda_m \mathbf{P}_m,
\qquad
\boldsymbol{\Sigma}_y
=
\sum_m \lambda_m^{(\boldsymbol{\Sigma}_y)}\mathbf{P}_m
\]
Because \(B(\mathbf{Q})=(\I_N+\gamma \mathbf{Q})^{-1}\) is a function of \(\mathbf{Q}\),
\[
B(\mathbf{Q})
=
\sum_m
\frac{1}{1+\gamma\lambda_m}\mathbf{P}_m
\]
Substituting into \autoref{eq:appendix-effective-target} gives
\begin{align}
\mathbf{M}(\mathbf{Q})
&=
\gamma
\left(
\sum_m
\frac{1}{1+\gamma\lambda_m}\mathbf{P}_m
\right)
\left(
\sum_m
\lambda_m^{(\boldsymbol{\Sigma}_y)}\mathbf{P}_m
\right)
\left(
\sum_m
\frac{1}{1+\gamma\lambda_m}\mathbf{P}_m
\right)
\\
&=
\sum_m
\gamma
\frac{\lambda_m^{(\boldsymbol{\Sigma}_y)}}{(1+\gamma\lambda_m)^2}
\mathbf{P}_m
\end{align}
Therefore the eigenvalues of \(\mathbf{M}(\mathbf{Q})\) are
\begin{align}
\lambda_m^{(\mathbf{M})}
=
\gamma
\frac{\lambda_m^{(\boldsymbol{\Sigma}_y)}}{(1+\gamma\lambda_m)^2}
\label{eq:appendix-M-filter}
\end{align}
Thus \(\mathbf{M}(\mathbf{Q})\) is exactly a spectrally filtered version of \(\boldsymbol{\Sigma}_y\): each
target eigenvalue \(\lambda_m^{(\boldsymbol{\Sigma}_y)}\) is multiplied by the scalar filter
\[
h_\gamma(\lambda_m)
=
\frac{\gamma}{(1+\gamma\lambda_m)^2}
\]

The inverse-SNR of \(\mathbf{M}\) is also filtered:
\begin{align}
\nu(\mathbf{M})
&=
\frac{\lambda_{SC}^{(\mathbf{M})}}{\lambda_S^{(\mathbf{M})}}
\\
&=
\frac{\lambda_{SC}^{(\boldsymbol{\Sigma}_y)}}{\lambda_S^{(\boldsymbol{\Sigma}_y)}}
\left(
\frac{1+\gamma\lambda_S}{1+\gamma\lambda_{SC}}
\right)^2
\\
&=
\nu(\boldsymbol{\Sigma}_y)
\left(
\frac{1+\gamma\lambda_S}{1+\gamma\lambda_{SC}}
\right)^2
\label{eq:appendix-nu-M-filter}
\end{align}
Thus the target noise seen by the feature dynamics is itself a dynamical quantity:
it depends on the current signal and noise eigenvalues of \(\mathbf{Q}\).

Note also that if \(\mathbf{Q}(t)\) is 2FS, then \(B(\mathbf{Q}(t))\) is 2FS, hence \(\mathbf{M}(\mathbf{Q}(t))\) is 2FS, and the right-hand side of \autoref{eq:appendix-linear-riccati} is again 2FS. Since \(\mathbf{Q}(0)\) is 2FS by \autoref{assumption:2fs}, the ODE vector field is tangent to the 2FS algebra, so \(\mathbf{Q}(t)\) remains 2FS for all \(t\) for which the flow exists. So the 2FS projectors diagonalize the dynamics at every time.

\paragraph{Mode formula for abstraction.} Here we derive the closed form expression for abstraction in terms of the 2FS mode eigenvalues used throughout the main text. At the class-centroid level, a 2FS representation admits a Walsh decomposition of the form
\[
\boldsymbol{\mu}_{s,c}
=
\boldsymbol{\mu}_G
+
s\boldsymbol{\mu}_S
+
c\boldsymbol{\mu}_C
+
sc\boldsymbol{\mu}_{SC},
\qquad
(s,c)\in\{-1,+1\}^2
\]
The two context-specific shape directions are
\[
\mathbf{u}_{c=+1}
=
\boldsymbol{\mu}_{+,+}-\boldsymbol{\mu}_{-,+}
=
2(\boldsymbol{\mu}_S+\boldsymbol{\mu}_{SC}),
\]
and
\[
\mathbf{u}_{c=-1}
=
\boldsymbol{\mu}_{+,-}-\boldsymbol{\mu}_{-,-}
=
2(\boldsymbol{\mu}_S-\boldsymbol{\mu}_{SC})
\]
The 2FS modes are orthogonal, so
\[
\langle \boldsymbol{\mu}_S,\boldsymbol{\mu}_{SC}\rangle=0
\]
Therefore
\begin{align}
\alpha
&=
\frac{
\langle \mathbf{u}_{c=+1},\mathbf{u}_{c=-1}\rangle
}{
\|\mathbf{u}_{c=+1}\|\,\|\mathbf{u}_{c=-1}\|
}
\\
&=
\frac{
\|\boldsymbol{\mu}_S\|^2-\|\boldsymbol{\mu}_{SC}\|^2
}{
\|\boldsymbol{\mu}_S\|^2+\|\boldsymbol{\mu}_{SC}\|^2
}
\end{align}
The eigenvalues \(\lambda_S\) and \(\lambda_{SC}\) of \(\mathbf{Q}\) are proportional to
\(\|\boldsymbol{\mu}_S\|^2\) and \(\|\boldsymbol{\mu}_{SC}\|^2\), with the same
positive class-size factor. Thus
\begin{align}
\alpha
&=
\frac{\lambda_S-\lambda_{SC}}{\lambda_S+\lambda_{SC}}
=
\frac{1-\nu(\mathbf{Q})}{1+\nu(\mathbf{Q})},
\qquad
\nu(\mathbf{Q})\coloneqq \frac{\lambda_{SC}}{\lambda_S}
\label{eq:appendix-alpha-mode-form}
\end{align}

\paragraph{Proof of \autoref{proposition:linear-odes}.}

First, we derive the ODE for the eigenvalues of $\mathbf{Q}$. Since all matrices are diagonal in the same 2FS basis, write
\[
\boldsymbol{\Sigma}_x
=
\sum_m
\lambda_m^{(\boldsymbol{\Sigma}_x)}\mathbf{P}_m,
\qquad
\mathbf{M}(\mathbf{Q})
=
\sum_m
\lambda_m^{(\mathbf{M})}\mathbf{P}_m,
\qquad
\mathbf{Q}
=
\sum_m
\lambda_m \mathbf{P}_m
\]
Substituting these decompositions into
\autoref{eq:appendix-linear-riccati} gives
\begin{align}
\dot{\mathbf{Q}}
&=
\boldsymbol{\Sigma}_x \mathbf{M} \mathbf{Q}+\mathbf{Q} \mathbf{M}\boldsymbol{\Sigma}_x
\\
&=
2
\sum_m
\lambda_m^{(\boldsymbol{\Sigma}_x)}
\lambda_m^{(\mathbf{M})}
\lambda_m
\mathbf{P}_m
\end{align}
Writing $\lambda_m$ as the eigenvalue of $\mathbf{Q}$ in mode $m$, we have
\[
\lambda_m = \langle \mathbf{P}_m, \mathbf{Q} \rangle
\]
Therefore each eigenvalue evolves independently as
\begin{align}
\dot\lambda_m
&=
2
\lambda_m^{(\boldsymbol{\Sigma}_x)}
\lambda_m^{(\mathbf{M})}
\lambda_m
\end{align}
Using the spectral filter formula \autoref{eq:appendix-M-filter}, we obtain
\begin{align}
\dot \lambda_m
&=
2\gamma
\lambda_m
\lambda_m^{(\boldsymbol{\Sigma}_x)}
\frac{\lambda_m^{(\boldsymbol{\Sigma}_y)}}{(1+\gamma\lambda_m)^2}
\label{eq:appendix-eigenvalue-ode-derived}
\end{align}
This is the first expression in \autoref{proposition:linear-odes}.

Now we derive the ODE for abstraction. Let
\[
s\coloneqq \lambda_S,
\qquad
n\coloneqq \lambda_{SC},
\qquad
q\coloneqq \nu(\mathbf{Q})=\frac{n}{s}
\]
Then
\[
\alpha=\frac{1-q}{1+q}
\]
From the eigenvalue ODE,
\[
\dot s
=
2s\lambda_S^{(\boldsymbol{\Sigma}_x)}\lambda_S^{(\mathbf{M})},
\qquad
\dot n
=
2n\lambda_{SC}^{(\boldsymbol{\Sigma}_x)}\lambda_{SC}^{(\mathbf{M})}
\]
Hence
\begin{align}
\frac{\dot q}{q}
&=
\frac{\dot n}{n}
-
\frac{\dot s}{s}
\\
&=
2\lambda_{SC}^{(\boldsymbol{\Sigma}_x)}\lambda_{SC}^{(\mathbf{M})}
-
2\lambda_S^{(\boldsymbol{\Sigma}_x)}\lambda_S^{(\mathbf{M})}
\\
&=
2\lambda_S^{(\boldsymbol{\Sigma}_x)}\lambda_S^{(\mathbf{M})}
\left(
\nu(\boldsymbol{\Sigma}_x)\nu(\mathbf{M})-1
\right)
\end{align}
Therefore
\begin{align}
\dot q
=
2q\lambda_S^{(\boldsymbol{\Sigma}_x)}\lambda_S^{(\mathbf{M})}
\left(
\nu(\boldsymbol{\Sigma}_x)\nu(\mathbf{M})-1
\right)
\label{eq:appendix-q-ode}
\end{align}
Since
\[
\frac{d\alpha}{dq}
=
-\frac{2}{(1+q)^2},
\]
we get
\begin{align}
\dot\alpha
&=
-\frac{2}{(1+q)^2}\dot q
\\
&=
\frac{4q}{(1+q)^2}
\lambda_S^{(\boldsymbol{\Sigma}_x)}\lambda_S^{(\mathbf{M})}
\left(
1-\nu(\boldsymbol{\Sigma}_x)\nu(\mathbf{M})
\right)
\end{align}
Finally,
\[
1-\alpha^2
=
1-\left(\frac{1-q}{1+q}\right)^2
=
\frac{4q}{(1+q)^2}
\]
Thus
\begin{align}
\dot\alpha
&=
(1-\alpha^2)
\lambda_S^{(\boldsymbol{\Sigma}_x)}
\lambda_S^{(\mathbf{M})}
\left(
1-\nu(\boldsymbol{\Sigma}_x)\nu(\mathbf{M})
\right)
\label{eq:appendix-alpha-ode-derived}
\end{align}
This is the second expression in \autoref{proposition:linear-odes}.

\clearpage
\subsection{Exact solutions for eigenvalues and abstraction} \label{appendix:implicit-solutions}

In this section we derive exact implicit expressions for the trajectories of the eigenvalues and abstraction. 

\subsubsection{Exact solutions for eigenvalues} \label{appendix:implicit-solutions-eigenvalues}

Here we solve the eigenvalue ODE in \autoref{proposition:linear-odes} to prove the following statement:

\begin{theorem}[Exact implicit solution for eigenvalues] \label{theorem:implicit-solution-eigenvalue}
Define the ridge-normalized eigenvalues $\tilde{\lambda}_m(t) \coloneqq \gamma \lambda_m(t)$. Let $F(a) \coloneqq \frac{1}{2} a^2 + 2 a + \log a$ for $a > 0$. Then the trajectory of the ridge-normalized eigenvalues obeys the following exact relation:
\begin{align}
    F\pbracket{\tilde{\lambda}_m(t)} = 2 \gamma \lambda_m^{(\boldsymbol{\Sigma}_x)} \lambda_m^{(\boldsymbol{\Sigma}_y)} t + F\pbracket{\tilde{\lambda}_{m, 0}}
\end{align}
\end{theorem}

\begin{proof}
For any mode \(m\in\{I,S,C,SC,G\}\), define
\[
b_m
\coloneqq
\lambda_m^{(\boldsymbol{\Sigma}_x)}
\lambda_m^{(\boldsymbol{\Sigma}_y)}
\]
The eigenvalue ODE from \autoref{proposition:linear-odes} is
\begin{align}
\dot\lambda_m
=
2\gamma b_m
\frac{\lambda_m}{(1+\gamma\lambda_m)^2}
\label{eq:appendix-scalar-lambda-ode}
\end{align}
The natural dimensionless variable for this ODE is
\begin{align}
\tilde\lambda_m(t)
\coloneqq
\gamma\lambda_m(t)
\label{eq:appendix-ridge-variable}
\end{align}
Then
\[
\dot{\tilde\lambda}_m
=
\gamma\dot\lambda_m
=
2\gamma b_m
\frac{\tilde\lambda_m}{(1+\tilde\lambda_m)^2}
\]
Let
\[
F(a)
\coloneqq
\frac{1}{2}a^2+2a+\log a,
\qquad
a>0
\]
Then
\[
F'(a)
=
a+2+\frac{1}{a}
=
\frac{(1+a)^2}{a}
\]
Therefore
\begin{align}
\frac{d}{dt}F(\tilde\lambda_m(t))
&=
F'(\tilde\lambda_m(t))\dot{\tilde\lambda}_m(t)
\\
&=
\frac{(1+\tilde\lambda_m(t))^2}{\tilde\lambda_m(t)}
\left(
2\gamma b_m
\frac{\tilde\lambda_m(t)}{(1+\tilde\lambda_m(t))^2}
\right)
\\
&=
2\gamma b_m
\end{align}
Integrating from \(0\) to \(t\) gives the exact expression in \autoref{theorem:implicit-solution-eigenvalue}:
\begin{align}
F(\tilde\lambda_m(t))
=
2\gamma
\lambda_m^{(\boldsymbol{\Sigma}_x)}
\lambda_m^{(\boldsymbol{\Sigma}_y)}
t
+
F(\tilde{\lambda}_{m, 0}),
\qquad
\tilde{\lambda}_{m, 0}\coloneqq \gamma\lambda_{m}(0)
\label{eq:appendix-exact-implicit-eigenvalue}
\end{align}
This concludes the proof 
\end{proof}

\paragraph{An exact inverse-function solution.}

The function \(F\) is strictly increasing on \((0,\infty)\), because
\[
F'(a)=\frac{(1+a)^2}{a}>0
\]
Also,
\[
\lim_{a\downarrow 0}F(a)=-\infty,
\qquad
\lim_{a\to\infty}F(a)=+\infty
\]
Since \(F\) is a bijection from \((0,\infty)\) to \(\mathbb{R}\), it is invertible. Define:
\[
\Xi \coloneqq F^{-1}
\]
Then the exact eigenvalue solution can be written as
\begin{align}
\tilde\lambda_m(t)
=
\Xi
\left(
F(\tilde{\lambda}_{m, 0})
+
2\gamma b_m t
\right)
\label{eq:appendix-exact-eigenvalue-inverse}
\end{align}
Equivalently,
\[
\lambda_m(t)
=
\frac{1}{\gamma}
\Xi
\left(
F(\gamma\lambda_{m,0})
+
2\gamma
\lambda_m^{(\boldsymbol{\Sigma}_x)}
\lambda_m^{(\boldsymbol{\Sigma}_y)}
t
\right)
\]
This is an exact solution in terms of the inverse of \(F\). However, \(\Xi\) cannot be expressed in terms of elementary functions because \(F(a)\) contains both a quadratic term and a logarithmic term.

\subsubsection{Exact solutions for abstraction} \label{appendix:exact-solutions-abstraction}

In this section we derive expressions for the abstraction trajectory in terms of both the inverse-function \(\Xi\) and an implicit parametric form. The shape abstraction depends only on the \(S\) and \(SC\) modes:
\[
\alpha(t)
=
\frac{\lambda_S(t)-\lambda_{SC}(t)}
{\lambda_S(t)+\lambda_{SC}(t)}
=
\frac{\tilde{\lambda}_S(t)-\tilde{\lambda}_{SC}(t)}
{\tilde{\lambda}_S(t)+\tilde{\lambda}_{SC}(t)}
\]
Define
\[
b_S
\coloneqq
\lambda_S^{(\boldsymbol{\Sigma}_x)}
\lambda_S^{(\boldsymbol{\Sigma}_y)},
\qquad
b_{SC}
\coloneqq
\lambda_{SC}^{(\boldsymbol{\Sigma}_x)}
\lambda_{SC}^{(\boldsymbol{\Sigma}_y)}
\]
Using \autoref{eq:appendix-exact-eigenvalue-inverse}, we obtain an exact expression for $\alpha(t)$ in terms of the inverse \(\Xi\):
\begin{align}
\alpha(t)
=
\frac{
\Xi\!\left(F(\tilde{\lambda}_{S, 0})+2\gamma b_S t\right)
-
\Xi\!\left(F(\tilde{\lambda}_{SC, 0})+2\gamma b_{SC} t\right)
}{
\Xi\!\left(F(\tilde{\lambda}_{S, 0})+2\gamma b_S t\right)
+
\Xi\!\left(F(\tilde{\lambda}_{SC, 0})+2\gamma b_{SC} t\right)
}
\label{eq:appendix-exact-alpha-inverse}
\end{align}
There is also a useful implicit parametric form. Assume \(b_S>0\), and define
\[
\bar{\nu}
\coloneqq
\frac{b_{SC}}{b_S}
=
\frac{
\lambda_{SC}^{(\boldsymbol{\Sigma}_x)}
\lambda_{SC}^{(\boldsymbol{\Sigma}_y)}
}{
\lambda_S^{(\boldsymbol{\Sigma}_x)}
\lambda_S^{(\boldsymbol{\Sigma}_y)}
}
=
\nu(\boldsymbol{\Sigma}_x)\nu(\boldsymbol{\Sigma}_y)
\]
From \autoref{eq:appendix-exact-implicit-eigenvalue},
\[
F(\tilde{\lambda}_S(t))
=
2\gamma b_S t+F(\tilde{\lambda}_{S, 0}),
\]
and
\[
F(\tilde{\lambda}_{SC}(t))
=
2\gamma b_{SC} t+F(\tilde{\lambda}_{SC, 0})
=
2\gamma \bar{\nu} b_S t+F(\tilde{\lambda}_{SC, 0})
\]
Eliminating \(t\) gives the exact conserved relation
\begin{align}
F(\tilde{\lambda}_{SC}(t))
-
\bar{\nu} F(\tilde{\lambda}_S(t))
=
F(\tilde{\lambda}_{SC, 0})
-
\bar{\nu} F(\tilde{\lambda}_{S, 0})
\label{eq:appendix-alpha-conserved-relation}
\end{align}
Now write
\[
q(t)
\coloneqq
\frac{\tilde{\lambda}_{SC}(t)}{\tilde{\lambda}_S(t)}
=
\frac{1-\alpha(t)}{1+\alpha(t)}
\]
Then \(\tilde{\lambda}_{SC}(t)=q(t)\tilde{\lambda}_S(t)\), and
\autoref{eq:appendix-alpha-conserved-relation} becomes
\begin{align}
F\!\left(
\tilde{\lambda}_{S}
\frac{1-\alpha}{1+\alpha}
\right)
-
\bar{\nu} F(\tilde{\lambda}_{S})
=
F(\tilde{\lambda}_{SC, 0})
-
\bar{\nu} F(\tilde{\lambda}_{S, 0})
\label{eq:appendix-alpha-implicit}
\end{align}
The corresponding time is
\begin{align}
t
=
\frac{
F(\tilde{\lambda}_{S})-F(\tilde{\lambda}_{S, 0})
}{
2\gamma b_S
}
\label{eq:appendix-alpha-parametric-time}
\end{align}
Equations \autoref{eq:appendix-alpha-implicit} and
\autoref{eq:appendix-alpha-parametric-time} give an exact implicit parametric
solution for \(\alpha(t)\), using \(\tilde{\lambda}_{S}\) as the clock variable.

\paragraph{Zero initialization.}

The formulas involving \(F\) assume \(\tilde{\lambda}_{m, 0}>0\). This is the generic case for
small random initialization. If a mode is initialized exactly at zero, then the
multiplicative ODE:
\[
\dot\lambda_m
=
2\gamma b_m
\frac{\lambda_m}{(1+\gamma\lambda_m)^2}
\]
keeps it at zero for all time. Thus exactly zero-initialized modes are invariant
boundary cases and should be handled separately. The terminal law below is the
generic interior statement, with the boundary cases obtained directly from the
same ODE.

\clearpage
\subsection{Asymptotics} \label{appendix:asymptotics}

\paragraph{Early-time asymptotics.} When \(\tilde\lambda_m\ll 1\),
\[
F(\tilde\lambda_m)
=
\log \tilde\lambda_m+2\tilde\lambda_m+\frac{1}{2}\tilde\lambda_m^2
=
\log \tilde\lambda_m+O(\tilde\lambda_m)
\]
Thus, while the mode remains in this regime,
\begin{align}
\log \tilde\lambda_m(t)
&=
\log \tilde{\lambda}_{m, 0}
+
2\gamma b_m t
+
O(\tilde\lambda_m(t)+\tilde{\lambda}_{m, 0}),
\end{align}
so to leading order
\begin{align}
\tilde\lambda_m(t)
\approx
\tilde{\lambda}_{m, 0}\exp(2\gamma b_m t)
\label{eq:appendix-early-mode-growth}
\end{align}
For the inverse-SNR \(q(t)=\tilde{\lambda}_{SC}(t)/\tilde{\lambda}_S(t)\), this gives
\begin{align}
q(t)
\approx
q_0
\exp\!\left(
2\gamma(b_{SC}-b_S)t
\right)
\label{eq:appendix-early-q-growth}
\end{align}
Thus, in the signal-dominant regime \(b_S>b_{SC}\), the noise-to-signal ratio
initially decreases exponentially, and abstraction initially increases toward
perfect abstraction.

\paragraph{Late-time asymptotics.}

For large \(a\),
\[
F(a)
=
\frac{1}{2}a^2+2a+\log a
=
\frac{1}{2}a^2\bigl(1+o(1)\bigr)
\]
Equivalently, for the dummy variable \(r\), as \(r \to\infty\),
\begin{align}
F^{-1}(r)
=
\sqrt{2r}
-
2
+
O\!\left(
\frac{\log r}{\sqrt{r}}
\right)
\label{eq:appendix-F-inverse-asymptotic}
\end{align}
Applying this to
\[
\tilde\lambda_m(t)
=
\Xi(F(\tilde{\lambda}_{m, 0})+2\gamma b_m t)
\]
gives, for any mode with \(b_m>0\),
\begin{align}
\tilde\lambda_m(t)
&=
2\sqrt{\gamma b_m t}
+
O(1)
\label{eq:appendix-late-r-asymptotic}
\end{align}
and equivalently, for the original eigenvalue variable,
\begin{align}
\lambda_m(t)
&=
2\sqrt{\frac{b_m t}{\gamma}}
+
O(\gamma^{-1})
\end{align}

\paragraph{Proof of \autoref{thm:fixed-point}.}

Assume \(b_S>0\), so that the shape signal mode grows. In the generic interior
case \(\tilde{\lambda}_{S, 0}>0\) and \(\tilde{\lambda}_{SC, 0}>0\), the exact solution gives
\[
F(\tilde{\lambda}_S(t))
=
2\gamma b_S t+O(1),
\qquad
F(\tilde{\lambda}_{SC}(t))
=
2\gamma b_{SC} t+O(1)
\]
If \(b_{SC}>0\), then by the late-time asymptotic
\autoref{eq:appendix-late-r-asymptotic},
\begin{align}
\frac{\tilde{\lambda}_{SC}(t)}{\tilde{\lambda}_S(t)}
&\longrightarrow
\frac{
2\sqrt{\gamma b_{SC}t}
}{
2\sqrt{\gamma b_S t}
}
=
\sqrt{\frac{b_{SC}}{b_S}}
\end{align}
If \(b_{SC}=0\), then the \(SC\) mode does not grow, while \(\tilde{\lambda}_S(t)\to\infty\), so the same formula holds with limiting ratio \(0\). Therefore, in nondegenerate cases we have:
\begin{align}
\lim_{t\to\infty}
\frac{\lambda_{SC}(t)}{\lambda_S(t)}
&=
\lim_{t\to\infty}
\frac{\tilde{\lambda}_{SC}(t)}{\tilde{\lambda}_S(t)}
\\
&=
\sqrt{
\frac{
\lambda_{SC}^{(\boldsymbol{\Sigma}_x)}
\lambda_{SC}^{(\boldsymbol{\Sigma}_y)}
}{
\lambda_S^{(\boldsymbol{\Sigma}_x)}
\lambda_S^{(\boldsymbol{\Sigma}_y)}
}
}
\\
&=
\sqrt{
\nu(\boldsymbol{\Sigma}_x)\nu(\boldsymbol{\Sigma}_y)
}
\label{eq:appendix-terminal-q}
\end{align}
Since
\[
\alpha(t)
=
\frac{1-q(t)}{1+q(t)},
\qquad
q(t)
=
\frac{\lambda_{SC}(t)}{\lambda_S(t)},
\]
Therefore we conclude that
\begin{align}
\alpha_\infty
&\coloneqq
\lim_{t\to\infty}\alpha(t)
\\
&=
\frac{
1-\sqrt{\nu(\boldsymbol{\Sigma}_x)\nu(\boldsymbol{\Sigma}_y)}
}{
1+\sqrt{\nu(\boldsymbol{\Sigma}_x)\nu(\boldsymbol{\Sigma}_y)}
}
\label{eq:appendix-terminal-alpha-law}
\end{align}
Which proves the expression for \autoref{thm:fixed-point}. The constants \(F(\tilde{\lambda}_{S, 0})\) and \(F(\tilde{\lambda}_{SC, 0})\) disappear in the large-time limit, which is why the terminal abstraction is independent of the initialization scale in the generic interior. Thus the initialization affects the trajectory, and in particular the maximum abstraction reached before the late-time regime, but not the terminal value in \autoref{eq:appendix-terminal-alpha-law}.

\clearpage
\subsection{Non-monotonicity and overshoot} \label{appendix:overshoot-initialization}

This subsection proves the overshoot claim discussed in \autoref{section:fixedpoints-overshoot}. The main text states that, in the signal-dominant regime, abstraction can rise above its terminal value before eventually returning to it. Here we prove a slightly stronger phase-plane statement: the inverse-SNR of the representation first decreases, crosses below its terminal value, reaches a unique minimum, and then increases back to its terminal value from below. Since abstraction is a decreasing function of this inverse-SNR, this means abstraction first increases past its terminal value, reaches a unique maximum, and then decreases back to the terminal value from above.

\paragraph{Reduction to the shape and interaction modes.}
Abstraction depends only on the \(S\) and \(SC\) eigenvalues of the feature kernel \(\mathbf{Q}(t)\). We therefore define
\begin{align}
s(t)
&\coloneqq
\tilde{\lambda}_{S}(t),
&
n(t)
&\coloneqq
\tilde{\lambda}_{SC}(t),
&
q(t)
&\coloneqq
\frac{n(t)}{s(t)}
\end{align}
Here \(s(t)\) is the ridge-normalized shape eigenvalue, \(n(t)\) is the ridge-normalized interaction eigenvalue, and \(q(t)\) is the inverse-SNR of \(\mathbf{Q}(t)\) for shape abstraction. Since ridge normalization multiplies both eigenvalues by the same factor \(\gamma\), it does not change their ratio. Hence
\begin{align}
\alpha(t)
=
\frac{\lambda_S(t)-\lambda_{SC}(t)}
{\lambda_S(t)+\lambda_{SC}(t)}
=
\frac{s(t)-n(t)}
{s(t)+n(t)}
=
\frac{1-q(t)}{1+q(t)}
\label{eq:appendix-overshoot-alpha-q}
\end{align}
Thus increasing abstraction is equivalent to decreasing \(q(t)\), because
\begin{align}
\frac{d}{dq}
\left(
\frac{1-q}{1+q}
\right)
=
-\frac{2}{(1+q)^2}
<
0
\label{eq:appendix-overshoot-alpha-decreasing}
\end{align}

Define the following helper variables to simplify notation:
\begin{align}
b_S
&\coloneqq
\lambda_S^{(\boldsymbol{\Sigma}_x)}
\lambda_S^{(\boldsymbol{\Sigma}_y)},
\\
\bar{\nu}
&\coloneqq
\nu(\boldsymbol{\Sigma}_x)\nu(\boldsymbol{\Sigma}_y)
=
\frac{
\lambda_{SC}^{(\boldsymbol{\Sigma}_x)}
\lambda_{SC}^{(\boldsymbol{\Sigma}_y)}
}{
\lambda_S^{(\boldsymbol{\Sigma}_x)}
\lambda_S^{(\boldsymbol{\Sigma}_y)}
},
\\
\rho
&\coloneqq
\sqrt{\bar{\nu}}
\end{align}
The terminal abstraction law in \autoref{thm:fixed-point} can then be written as
\begin{align}
q_\infty
=
\rho,
\qquad
\alpha_\infty
=
\frac{1-\rho}{1+\rho}
\label{eq:appendix-overshoot-terminal-values}
\end{align}
Throughout this subsection we assume the non-degenerate signal-dominant regime
\begin{align}
0<\bar{\nu}<1,
\qquad
\text{equivalently}
\qquad
0<\rho<1
\end{align}
The degenerate case \(\bar{\nu}=0\) has \(\alpha_\infty=1\), so strict overshoot above the terminal abstraction is impossible because \(\alpha(t)\leq 1\). We also work in the interior \(s_0>0\), \(n_0>0\), where
\begin{align}
s_0\coloneqq s(0),
\qquad
n_0\coloneqq n(0),
\qquad
q_0\coloneqq q(0)=\frac{n_0}{s_0}
\end{align}
The assumption \(\alpha_0<\alpha_\infty\) is equivalent to
\begin{align}
q_0>\rho
\label{eq:appendix-overshoot-q0-above-rho}
\end{align}
We can see that the map \(q\mapsto (1-q)/(1+q)\) is strictly decreasing by \autoref{eq:appendix-overshoot-alpha-decreasing}.

\paragraph{Exact two-mode dynamics.}
By \autoref{theorem:implicit-solution-eigenvalue}, the ridge-normalized eigenvalues satisfy
\begin{align}
F(s(t))
&=
F(s_0)
+
2\gamma b_S t,
\\
F(n(t))
&=
F(n_0)
+
2\gamma \bar{\nu} b_S t,
\label{eq:appendix-overshoot-implicit-sn}
\end{align}
where
\begin{align}
F(a)
=
\frac{1}{2}a^2+2a+\log a,
\qquad
a>0
\end{align}
Notice the factor \(2\gamma b_S t\). This follows from the ridge-normalized ODE in \autoref{proposition:linear-odes}; differentiating \(F(\tilde{\lambda}_m(t))\) gives \(2\gamma \lambda_m^{(\boldsymbol{\Sigma}_x)}\lambda_m^{(\boldsymbol{\Sigma}_y)}\), not \(2\lambda_m^{(\boldsymbol{\Sigma}_x)}\lambda_m^{(\boldsymbol{\Sigma}_y)}/\gamma\).

Since
\begin{align}
F'(a)
=
a+2+\frac{1}{a}
=
\frac{(1+a)^2}{a}
>
0,
\end{align}
the function \(F\) is strictly increasing on \((0,\infty)\). Differentiating \autoref{eq:appendix-overshoot-implicit-sn} gives the explicit scalar ODEs
\begin{align}
\dot{s}
&=
2\gamma b_S
\frac{s}{(1+s)^2},
\\
\dot{n}
&=
2\gamma \bar{\nu} b_S
\frac{n}{(1+n)^2}
\label{eq:appendix-overshoot-s-n-odes}
\end{align}
In particular, \(s(t)\) and \(n(t)\) are strictly increasing, and \(s(t)\) may be used as a clock variable.

\begin{proposition}[Overshoot and non-monotonicity of abstraction]
\label{proposition:overshoot-appendix}
Assume \(0<\bar{\nu}<1\) and \(\alpha_0<\alpha_\infty\). Equivalently, assume \(0<\rho<1\) and \(q_0>\rho\). Then:
\begin{enumerate}[label=(\roman*)]
    \item There is a unique finite time \(T_{\mathrm{cross}}>0\) such that
    \begin{align}
    q(T_{\mathrm{cross}})=\rho,
    \qquad
    \alpha(T_{\mathrm{cross}})=\alpha_\infty
    \end{align}
    Moreover,
    \begin{align}
    q(t)>\rho
    \quad \text{for } 0\leq t<T_{\mathrm{cross}},
    \qquad
    q(t)<\rho
    \quad \text{for every finite } t>T_{\mathrm{cross}}
    \end{align}
    Equivalently,
    \begin{align}
    \alpha(t)<\alpha_\infty
    \quad \text{for } 0\leq t<T_{\mathrm{cross}},
    \qquad
    \alpha(t)>\alpha_\infty
    \quad \text{for every finite } t>T_{\mathrm{cross}}
    \end{align}

    \item There is a unique finite time \(t_*>T_{\mathrm{cross}}\) at which \(q(t)\) attains its global minimum. Consequently, \(\alpha(t)\) attains its unique global maximum at \(t_*\), and this maximum is strictly larger than \(\alpha_\infty\).
\end{enumerate}
Thus abstraction overshoots its terminal value and is non-monotone along training.
\end{proposition}

\begin{proof}
\medskip
\noindent\textbf{Step 1: the conserved trajectory in the \((s,n)\)-plane.}
Eliminating time from \autoref{eq:appendix-overshoot-implicit-sn} gives
\begin{align}
F(n(t))-\rho^2 F(s(t))
=
F(n_0)-\rho^2F(s_0)
\eqqcolon
C
\label{eq:appendix-overshoot-trajectory-constant}
\end{align}
This equation describes the trajectory in the \((s,n)\)-plane.

\medskip
\noindent\textbf{Step 2: the unique crossing of the terminal ratio.}
The condition \(\alpha(t)=\alpha_\infty\) is equivalent to \(q(t)=\rho\), or
\begin{align}
n(t)=\rho s(t)
\end{align}
Substituting this relation into the conserved trajectory equation
\autoref{eq:appendix-overshoot-trajectory-constant}, define
\begin{align}
H_\infty(s)
&\coloneqq
F(\rho s)-\rho^2F(s)
\end{align}
Expanding \(F\) gives
\begin{align}
H_\infty(s)
&=
2\rho(1-\rho)s
+
\log \rho
+
(1-\rho^2)\log s
\label{eq:appendix-overshoot-H-infty}
\end{align}
Therefore
\begin{align}
H_\infty'(s)
&=
2\rho(1-\rho)
+
\frac{1-\rho^2}{s}
>
0
\end{align}
Also,
\begin{align}
\lim_{s\downarrow 0}H_\infty(s)
=
-\infty,
\qquad
\lim_{s\to\infty}H_\infty(s)
=
+\infty
\end{align}
Hence \(H_\infty\) is a strictly increasing bijection from \((0,\infty)\) to \(\mathbb{R}\). It follows that the equation
\begin{align}
H_\infty(s)=C
\end{align}
has a unique positive solution. Call it \(s_{\mathrm{cross}}\).

We next show that this crossing happens after initialization. Since \(q_0>\rho\), we have
\begin{align}
n_0=q_0s_0>\rho s_0
\end{align}
Because \(F\) is strictly increasing,
\begin{align}
C
=
F(n_0)-\rho^2F(s_0)
>
F(\rho s_0)-\rho^2F(s_0)
=
H_\infty(s_0)
\end{align}
Since \(H_\infty\) is strictly increasing, this implies
\begin{align}
s_{\mathrm{cross}}>s_0
\end{align}
The crossing time is therefore finite and positive:
\begin{align}
T_{\mathrm{cross}}
=
\frac{
F(s_{\mathrm{cross}})-F(s_0)
}{
2\gamma b_S
}
>
0
\label{eq:appendix-overshoot-crossing-time}
\end{align}

It remains to identify which side of the terminal ratio the trajectory lies on before and after this crossing. For any time \(t\),
\begin{align}
q(t)>\rho
&\Longleftrightarrow
n(t)>\rho s(t)
\\
&\Longleftrightarrow
F(n(t))>F(\rho s(t))
\\
&\Longleftrightarrow
C>H_\infty(s(t))
\end{align}
Because \(s(t)\) is strictly increasing and \(H_\infty\) is strictly increasing, this inequality holds exactly when \(s(t)<s_{\mathrm{cross}}\), i.e. exactly before \(T_{\mathrm{cross}}\). Similarly, for every finite time \(t>T_{\mathrm{cross}}\), we have \(s(t)>s_{\mathrm{cross}}\), hence
\begin{align}
C<H_\infty(s(t)),
\end{align}
and therefore
\begin{align}
q(t)<\rho
\end{align}
Using the monotone relationship \autoref{eq:appendix-overshoot-alpha-q}, this proves the first claim.

\medskip
\noindent\textbf{Step 3: the unique minimum of \(q(t)\).}
We now show that the trajectory is non-monotone. From
\autoref{eq:appendix-overshoot-s-n-odes},
\begin{align}
\frac{\dot{q}}{q}
&=
\frac{\dot{n}}{n}
-
\frac{\dot{s}}{s}
\\
&=
2\gamma b_S
\left[
\frac{\rho^2}{(1+n)^2}
-
\frac{1}{(1+s)^2}
\right].
\label{eq:appendix-overshoot-qdot}
\end{align}
Since \(q>0\), the sign of \(\dot{q}\) is the sign of the bracketed quantity. Thus
\begin{align}
\dot{q}=0
&\Longleftrightarrow
\frac{\rho}{1+n}
=
\frac{1}{1+s}
\\
&\Longleftrightarrow
n
=
\rho(1+s)-1
\label{eq:appendix-overshoot-nullcline}
\end{align}
This is the nullcline of the ratio \(q(t)\). It is the curve in the \((s,n)\)-plane along which the inverse-SNR stops decreasing and starts increasing.

The nullcline has positive \(n\) only when
\begin{align}
s>s_{\mathrm{c}},
\qquad
s_{\mathrm{c}}
\coloneqq
\frac{1-\rho}{\rho}
\end{align}
On this domain, substitute the nullcline relation
\autoref{eq:appendix-overshoot-nullcline} into the conserved trajectory
\autoref{eq:appendix-overshoot-trajectory-constant} and define
\begin{align}
G(s)
&\coloneqq
F\!\left(\rho(1+s)-1\right)
-
\rho^2F(s),
\qquad
s>s_{\mathrm{c}}
\end{align}
We claim that \(G\) is strictly increasing. Differentiating gives
\begin{align}
G'(s)
&=
\rho
F'\!\left(\rho(1+s)-1\right)
-
\rho^2F'(s)
\end{align}
Using \(F'(a)=(1+a)^2/a\) and
\begin{align}
1+\rho(1+s)-1
=
\rho(1+s),
\end{align}
we obtain
\begin{align}
G'(s)
&=
\rho^2(1+s)^2
\left[
\frac{\rho}{\rho(1+s)-1}
-
\frac{1}{s}
\right]
\\
&=
\frac{
\rho^2(1+s)^2(1-\rho)
}{
s\left(\rho(1+s)-1\right)
}
>
0
\label{eq:appendix-overshoot-G-prime}
\end{align}
Moreover,
\begin{align}
\lim_{s\downarrow s_{\mathrm{c}}}G(s)
=
-\infty,
\end{align}
because the argument \(\rho(1+s)-1\) of \(F\) tends to \(0\) from above, and
\begin{align}
\lim_{a\downarrow 0}F(a)=-\infty
\end{align}
At the other endpoint,
\begin{align}
G(s)
=
\rho(1-\rho)s
+
(1-\rho^2)\log s
+
O(1),
\qquad
s\to\infty,
\end{align}
so
\begin{align}
\lim_{s\to\infty}G(s)
=
+\infty
\end{align}
Therefore, for every real value of the trajectory constant \(C\), there is a unique solution
\begin{align}
G(s_*)=C
\label{eq:appendix-overshoot-sstar}
\end{align}
This unique \(s_*\) is the unique intersection between the trajectory and the nullcline.

We now check that this intersection occurs after initialization. If \(s_0\leq s_{\mathrm{c}}\), then \(s_*>s_{\mathrm{c}}\geq s_0\). If \(s_0>s_{\mathrm{c}}\), then the nullcline value at initialization is positive, and
\begin{align}
n_0-\left(\rho(1+s_0)-1\right)
&=
s_0(q_0-\rho)
+
(1-\rho)
>
0,
\end{align}
where we used \(q_0>\rho\) and \(0<\rho<1\). Since \(F\) is strictly increasing,
\begin{align}
C
=
F(n_0)-\rho^2F(s_0)
>
F\!\left(\rho(1+s_0)-1\right)-\rho^2F(s_0)
=
G(s_0)
\end{align}
Because \(G\) is strictly increasing, this implies \(s_*>s_0\). Thus in all cases the nullcline intersection occurs after initialization. The corresponding time is
\begin{align}
t_*
=
\frac{
F(s_*)-F(s_0)
}{
2\gamma b_S
}
>
0
\label{eq:appendix-overshoot-tstar}
\end{align}

Finally, we determine the sign of \(\dot{q}\). From
\autoref{eq:appendix-overshoot-qdot},
\begin{align}
\dot{q}<0
&\Longleftrightarrow
\frac{\rho}{1+n}
<
\frac{1}{1+s}
\\
&\Longleftrightarrow
n>\rho(1+s)-1
\label{eq:appendix-overshoot-qdot-negative}
\end{align}
If \(s(t)<s_*\), then either \(s(t)\leq s_{\mathrm{c}}\), in which case \(\rho(1+s(t))-1\leq 0<n(t)\), or \(s_{\mathrm{c}}<s(t)<s_*\), in which case \(G(s(t))<C\). In the second case, the conserved trajectory gives
\begin{align}
F(n(t))
=
C+\rho^2F(s(t))
>
G(s(t))+\rho^2F(s(t))
=
F\!\left(\rho(1+s(t))-1\right)
\end{align}
Since \(F\) is strictly increasing,
\begin{align}
n(t)>\rho(1+s(t))-1
\end{align}
Thus \(\dot{q}(t)<0\) whenever \(s(t)<s_*\), equivalently whenever \(t<t_*\).

Similarly, if \(s(t)>s_*\), then \(G(s(t))>C\), so
\begin{align}
F(n(t))
<
F\!\left(\rho(1+s(t))-1\right),
\end{align}
and hence
\begin{align}
n(t)<\rho(1+s(t))-1
\end{align}
Therefore \(\dot{q}(t)>0\) whenever \(t>t_*\).

Thus \(q(t)\) strictly decreases on \([0,t_*)\), reaches its unique global minimum at \(t_*\), and strictly increases on \((t_*,\infty)\). At the minimum, the trajectory lies on the nullcline, so
\begin{align}
q(t_*)
&=
\frac{\rho(1+s_*)-1}{s_*}
\\
&=
\rho
-
\frac{1-\rho}{s_*}
\label{eq:appendix-overshoot-qstar}
\end{align}
Since \(s_*>s_{\mathrm{c}}=(1-\rho)/\rho\), this minimum is positive. Since \(s_*<\infty\), it is also strictly below \(\rho\):
\begin{align}
0<q(t_*)<\rho
\end{align}
Consequently \(q(t)\) crosses the terminal ratio \(\rho\) before it reaches its minimum, so \(T_{\mathrm{cross}}<t_*\).

Using the strictly decreasing relationship between \(\alpha\) and \(q\) from
\autoref{eq:appendix-overshoot-alpha-decreasing}, \(\alpha(t)\) strictly increases on \([0,t_*)\), reaches its unique global maximum at \(t_*\), and strictly decreases on \((t_*,\infty)\). Since \(q(t_*)<\rho\), this maximum satisfies
\begin{align}
\alpha(t_*)
=
\frac{1-q(t_*)}{1+q(t_*)}
>
\frac{1-\rho}{1+\rho}
=
\alpha_\infty
\end{align}
By \autoref{thm:fixed-point}, \(\alpha(t)\to\alpha_\infty\) as \(t\to\infty\). Combining this convergence with the crossing result above, \(\alpha(t)\) approaches \(\alpha_\infty\) from above for all sufficiently large finite times. This proves the proposition.
\end{proof}

\paragraph{Interpretation.}
The derivations above analyze the mechanism behind overshoot. In the early time regime, \autoref{appendix:asymptotics} shows that
\begin{align}
q(t)
\approx
q_0
\exp\!\left(
2\gamma b_S(\bar{\nu}-1)t
\right),
\end{align}
so \(q(t)\) initially decreases when \(\bar{\nu}<1\). This is the initial abstraction-improving phase: the signal mode grows faster than the interaction mode. However, the late-time dynamics do not preserve this initial exponential growth-rate comparison. Instead, the late-time asymptotics force
\begin{align}
q(t)
=
\frac{n(t)}{s(t)}
\longrightarrow
\sqrt{\bar{\nu}}
=
\rho,
\end{align}
as proved in \autoref{thm:fixed-point}. The trajectory therefore undershoots the terminal inverse-SNR \(\rho\) before returning to it. Since abstraction is a decreasing function of inverse-SNR, this undershoot of \(q(t)\) is exactly the overshoot of \(\alpha(t)\) above \(\alpha_\infty\).

\clearpage
\subsection{Initialization scale} \label{appendix:initialization-scale}

This section characterizes the time at which the abstraction trajectory peaks and proves \autoref{theorem:initialization-overshoot}. Define
\begin{align}
s(t)
&\coloneqq
\tilde{\lambda}_{S}(t),
&
n(t)
&\coloneqq
\tilde{\lambda}_{SC}(t),
&
q(t)
&\coloneqq
\frac{n(t)}{s(t)}
\end{align}
We also write
\begin{align}
\bar{\nu}
&\coloneqq
\nu(\boldsymbol{\Sigma}_x)\nu(\boldsymbol{\Sigma}_y),
&
\rho
&\coloneqq
\sqrt{\bar{\nu}}
\end{align}
Thus \(\rho\) is the terminal value of \(q(t)\), and
\begin{align}
\alpha_\infty
=
\frac{1-\rho}{1+\rho}
\end{align}
Throughout this subsection, we work in the non-degenerate signal-dominated case
\begin{align}
0<\rho<1
\end{align}

\paragraph{Characterizing the peak abstraction time.}
Fix the initial inverse-SNR
\begin{align}
q_0
\coloneqq
\nu_{\mathbf{Q},0}
=
\frac{n_0}{s_0}
\end{align}
The initialization family in \autoref{theorem:initialization-overshoot} is
\begin{align}
s_0=\kappa,
\qquad
n_0=\kappa q_0
\end{align}
The assumption \(\alpha_0<\alpha_\infty\) is equivalent to
\begin{align}
q_0>\rho,
\end{align}
because \(q\mapsto (1-q)/(1+q)\) is strictly decreasing on \((0,\infty)\).

By \autoref{theorem:implicit-solution-eigenvalue}, the ridge-normalized eigenvalues satisfy
\begin{align}
F(s(t))
&=
2\gamma b_S t
+
F(\kappa),
\\
F(n(t))
&=
2\gamma \rho^2 b_S t
+
F(\kappa q_0),
\end{align}
where
\begin{align}
b_S
&\coloneqq
\lambda_S^{(\boldsymbol{\Sigma}_x)}\lambda_S^{(\boldsymbol{\Sigma}_y)},
&
F(a)
&=
\frac{1}{2}a^2+2a+\log a
\end{align}
The factor \(2\gamma b_S\) follows from the ridge-normalized variable \(\tilde\lambda_m=\gamma\lambda_m\) under the ridge convention used in \autoref{proposition:linear-odes}.
Since
\begin{align}
F'(a)
=
a+2+\frac{1}{a}
=
\frac{(1+a)^2}{a},
\end{align}
differentiating the implicit solutions gives
\begin{align}
\dot{s}
&=
2\gamma b_S
\frac{s}{(1+s)^2},
\\
\dot{n}
&=
2\gamma \rho^2 b_S
\frac{n}{(1+n)^2}
\end{align}
Therefore
\begin{align}
\frac{\dot{q}}{q}
&=
\frac{\dot{n}}{n}
-
\frac{\dot{s}}{s}
\\
&=
2\gamma b_S
\left[
\frac{\rho^2}{(1+n)^2}
-
\frac{1}{(1+s)^2}
\right].
\label{eq:qdot-initialization-proof}
\end{align}
Because \(s,n>0\), the condition \(\dot{q}=0\) is equivalent to
\begin{align}
\frac{\rho}{1+n}
=
\frac{1}{1+s},
\end{align}
or
\begin{align}
n
=
\rho(1+s)-1
\label{eq:q-critical-curve}
\end{align}
This curve is the nullcline of \(q(t)\). It is the curve along which the ratio \(q(t)=n(t)/s(t)\) stops decreasing and starts increasing.

Next, subtracting \(\rho^2\) times the implicit equation for \(s(t)\) from the implicit equation for \(n(t)\) eliminates time:
\begin{align}
F(n(t))-\rho^2F(s(t))
=
C_\kappa,
\label{eq:trajectory-constant-kappa}
\end{align}
where
\begin{align}
C_\kappa
\coloneqq
F(\kappa q_0)-\rho^2F(\kappa)
\end{align}
This is the trajectory equation in the \((s,n)\)-plane. To find where the trajectory intersects the nullcline \autoref{eq:q-critical-curve}, substitute
\begin{align}
n=\rho(1+s)-1
\end{align}
into \autoref{eq:trajectory-constant-kappa}. This gives the scalar function
\begin{align}
G(s)
\coloneqq
F\!\left(\rho(1+s)-1\right)
-
\rho^2F(s),
\qquad
s>\frac{1-\rho}{\rho}
\end{align}
The lower endpoint is the point at which the nullcline first has positive \(n\).

We now show that \(G\) is strictly increasing on this domain. Using
\begin{align}
1+\left(\rho(1+s)-1\right)
=
\rho(1+s),
\end{align}
we obtain
\begin{align}
G'(s)
&=
\rho
F'\!\left(\rho(1+s)-1\right)
-
\rho^2F'(s)
\\
&=
\rho^2(1+s)^2
\left[
\frac{\rho}{\rho(1+s)-1}
-
\frac{1}{s}
\right]
\\
&=
\frac{
\rho^2(1+s)^2(1-\rho)
}{
s\left(\rho(1+s)-1\right)
}
>
0
\end{align}
Moreover,
\begin{align}
\lim_{s\downarrow (1-\rho)/\rho}G(s)
=
-\infty,
\qquad
\lim_{s\to\infty}G(s)
=
+\infty
\end{align}
The second limit follows, for example, from the expansion
\begin{align}
G(s)
=
\rho(1-\rho)s
+
(1-\rho^2)\log s
+
O(1),
\qquad
s\to\infty
\end{align}
Therefore, for every \(\kappa>0\), there is a unique solution \(s_*(\kappa)\) to
\begin{align}
G(s_*(\kappa))=C_\kappa
\label{eq:sstar-definition}
\end{align}
At this point, the corresponding interaction coordinate is
\begin{align}
n_*(\kappa)
=
\rho(1+s_*(\kappa))-1,
\end{align}
and hence
\begin{align}
q_*(\kappa)
\coloneqq
\frac{n_*(\kappa)}{s_*(\kappa)}
=
\rho
-
\frac{1-\rho}{s_*(\kappa)}
\label{eq:qstar-kappa}
\end{align}

It remains to justify that this point is indeed the unique global minimum of \(q(t)\). From \autoref{eq:qdot-initialization-proof},
\begin{align}
\dot{q}<0
&\Longleftrightarrow
\frac{\rho}{1+n}<\frac{1}{1+s}
\\
&\Longleftrightarrow
n>\rho(1+s)-1
\label{eq:qdot-negative-condition}
\end{align}
At initialization,
\begin{align}
1+n_0-\rho(1+s_0)
&=
1+\kappa q_0-\rho(1+\kappa)
\\
&=
1-\rho+\kappa(q_0-\rho)
>
0,
\end{align}
so \(q(t)\) initially decreases.

The unique solution \(s_*(\kappa)\) occurs after initialization. To see this, if \(\kappa\leq (1-\rho)/\rho\), then \(s_*(\kappa)>(1-\rho)/\rho\geq \kappa\). If instead \(\kappa>(1-\rho)/\rho\), then the nullcline value \(\rho(1+\kappa)-1\) is positive, and the inequality above implies
\begin{align}
n_0>\rho(1+\kappa)-1
\end{align}
Since \(F\) is strictly increasing,
\begin{align}
C_\kappa
&=
F(n_0)-\rho^2F(\kappa)
\\
&>
F\!\left(\rho(1+\kappa)-1\right)-\rho^2F(\kappa)
\\
&=
G(\kappa)
\end{align}
Because \(G\) is strictly increasing and \(G(s_*(\kappa))=C_\kappa\), this also implies \(s_*(\kappa)>\kappa\).

Since \(s(t)\) is strictly increasing in time, the unique value \(s_*(\kappa)\) corresponds to the unique time
\begin{align}
t_*(\kappa)
=
\frac{1}{2\gamma b_S}
\left[
F(s_*(\kappa))-F(\kappa)
\right]
\label{eq:tstar-kappa}
\end{align}
This is the unique time at which the trajectory intersects the nullcline \(\dot q=0\).

Finally, the sign of \(\dot q\) changes from negative to positive at this time. If \(s(t)<s_*(\kappa)\), then either \(s(t)\leq (1-\rho)/\rho\), in which case \(\rho(1+s(t))-1\leq 0<n(t)\), or \(s(t)>(1-\rho)/\rho\), in which case \(G(s(t))<C_\kappa\). In both cases,
\begin{align}
n(t)>\rho(1+s(t))-1,
\end{align}
and therefore \(\dot q(t)<0\). Similarly, if \(s(t)>s_*(\kappa)\), then \(G(s(t))>C_\kappa\), which implies
\begin{align}
n(t)<\rho(1+s(t))-1,
\end{align}
and therefore \(\dot q(t)>0\). Thus \(q(t)\) decreases for \(t<t_*(\kappa)\) and increases for \(t>t_*(\kappa)\). Consequently,
\begin{align}
q_{\min}(\kappa)=q_*(\kappa)
\end{align}
Because
\begin{align}
\frac{d}{dq}
\frac{1-q}{1+q}
=
-\frac{2}{(1+q)^2}
<
0,
\end{align}
the abstraction \(\alpha(t)=(1-q(t))/(1+q(t))\) is maximized at the same time \(t_*(\kappa)\). Therefore
\begin{align}
\alpha_{\max}(\kappa)
=
\frac{1-q_*(\kappa)}{1+q_*(\kappa)}
\label{eq:alphamax-qstar}
\end{align}

\paragraph{Proof of \autoref{theorem:initialization-overshoot}.}
We now prove the theorem using the peak characterization above. Recall from \autoref{eq:alphamax-qstar} that
\begin{align}
\alpha_{\max}(\kappa)
=
\frac{1-q_*(\kappa)}{1+q_*(\kappa)},
\end{align}
where
\begin{align}
q_*(\kappa)
=
\rho
-
\frac{1-\rho}{s_*(\kappa)}
\end{align}
and \(s_*(\kappa)\) is the unique solution of
\begin{align}
G(s_*(\kappa))=C_\kappa,
\end{align}
with
\begin{align}
C_\kappa
=
F(\kappa q_0)-\rho^2F(\kappa)
\end{align}

We first show that \(\alpha_{\max}(\kappa)\) is strictly decreasing in \(\kappa\). Differentiate \(C_\kappa\):
\begin{align}
\frac{dC_\kappa}{d\kappa}
&=
q_0F'(\kappa q_0)-\rho^2F'(\kappa)
\\
&=
\kappa(q_0^2-\rho^2)
+
2(q_0-\rho^2)
+
\frac{1-\rho^2}{\kappa}
>
0
\end{align}
The strict positivity follows from \(q_0>\rho\) and \(0<\rho<1\). Since \(G\) is strictly increasing, the identity
\begin{align}
G(s_*(\kappa))=C_\kappa
\end{align}
implies that \(s_*(\kappa)\) is strictly increasing in \(\kappa\). Hence
\begin{align}
q_*(\kappa)
=
\rho
-
\frac{1-\rho}{s_*(\kappa)}
\end{align}
is also strictly increasing in \(\kappa\). Finally, since
\begin{align}
\frac{d}{dq}
\frac{1-q}{1+q}
=
-\frac{2}{(1+q)^2}
<
0,
\end{align}
it follows that
\begin{align}
\alpha_{\max}(\kappa)
=
\frac{1-q_*(\kappa)}{1+q_*(\kappa)}
\end{align}
is strictly decreasing in \(\kappa\).

Next we compute the limiting behavior as we take the rich \((\kappa\downarrow0)\) and lazy \((\kappa\to\infty)\) limits. As \(\kappa\downarrow0\), we have
\begin{align}
C_\kappa
&=
F(\kappa q_0)-\rho^2F(\kappa)
\\
&=
(1-\rho^2)\log \kappa
+
\log q_0
+
o(1)
\to
-\infty
\end{align}
Since \(G\) is strictly increasing and tends to \(-\infty\) at the lower endpoint \((1-\rho)/\rho\), this implies
\begin{align}
s_*(\kappa)
\downarrow
\frac{1-\rho}{\rho}
\end{align}
Therefore
\begin{align}
q_*(\kappa)
=
\rho
-
\frac{1-\rho}{s_*(\kappa)}
\to
0,
\end{align}
and hence
\begin{align}
\lim_{\kappa\downarrow0}
\alpha_{\max}(\kappa)
=
\frac{1-0}{1+0}
=
1
\end{align}

On the other hand, as \(\kappa\to\infty\),
\begin{align}
C_\kappa
&=
F(\kappa q_0)-\rho^2F(\kappa)
\\
&=
\frac{1}{2}\kappa^2(q_0^2-\rho^2)
+
O(\kappa)
\to
+\infty
\end{align}
Since \(G(s)\to+\infty\) only as \(s\to\infty\), we have
\begin{align}
s_*(\kappa)
\to
\infty
\end{align}
Therefore
\begin{align}
q_*(\kappa)
=
\rho
-
\frac{1-\rho}{s_*(\kappa)}
\to
\rho,
\end{align}
and thus
\begin{align}
\lim_{\kappa\to\infty}
\alpha_{\max}(\kappa)
=
\frac{1-\rho}{1+\rho}
=
\alpha_\infty
\end{align}
This completes the proof of all the statements in \autoref{theorem:initialization-overshoot}.
\qed

\paragraph{Interpretation.}
The proof reveals why initialization scale controls the peak abstraction but not the terminal abstraction. The terminal value is determined by the asymptotic ratio
\begin{align}
q(t)
\to
\rho
=
\sqrt{\bar{\nu}}
=
\sqrt{\nu(\boldsymbol{\Sigma}_x)\nu(\boldsymbol{\Sigma}_y)},
\end{align}
which depends only on input and target geometry. By contrast, the maximum abstraction is controlled by the lowest value reached by \(q(t)\), and this depends on the trajectory constant
\begin{align}
C_\kappa
=
F(\kappa q_0)-\rho^2F(\kappa)
\end{align}
For small \(\kappa\), the logarithmic part of \(F\) dominates, forcing the minimum of \(q(t)\) close to zero and hence pushing \(\alpha_{\max}(\kappa)\) close to perfect abstraction. For large \(\kappa\), the quadratic part of \(F\) dominates, and the minimum of \(q(t)\) occurs only after \(q(t)\) has dipped by a vanishing amount below its terminal ratio \(\rho\). This is the lazy-limit behavior in which overshoot above the terminal abstraction disappears asymptotically.

\clearpage
\subsection{Rich-limit residence-time metastability} \label{appendix:rich-limit-residence}

In this section, we analyze the residence time of the trajectory in the near-perfect abstraction band and prove \autoref{proposition:rich-limit-metastability}. As before, define
\begin{align}
s(t)
&\coloneqq
\tilde{\lambda}_S(t),
&
n(t)
&\coloneqq
\tilde{\lambda}_{SC}(t),
&
q(t)
&\coloneqq
\frac{n(t)}{s(t)}
\end{align}
Recall that shape abstraction is
\begin{align}
\alpha(t)=\frac{1-q(t)}{1+q(t)}
\end{align}
Let
\begin{align}
\bar{\nu}
&\coloneqq
\nu(\boldsymbol{\Sigma}_x)\nu(\boldsymbol{\Sigma}_y)
=
\frac{\lambda_{SC}^{(\boldsymbol{\Sigma}_x)}\lambda_{SC}^{(\boldsymbol{\Sigma}_y)}}
{\lambda_S^{(\boldsymbol{\Sigma}_x)}\lambda_S^{(\boldsymbol{\Sigma}_y)}},
\\
\rho
&\coloneqq
\sqrt{\bar{\nu}},
\\
b_S
&\coloneqq
\lambda_S^{(\boldsymbol{\Sigma}_x)}\lambda_S^{(\boldsymbol{\Sigma}_y)}
\end{align}
Under \autoref{setting:signal-dominated}, \(0<\rho<1\). The initialization family from \autoref{theorem:initialization-overshoot} is
\begin{align}
s(0)=\kappa,
\qquad
n(0)=\kappa q_0,
\qquad
q_0\coloneqq\nu_{\mathbf{Q},0}
\end{align}
Moreover, \(\alpha_0<\alpha_\infty\) is equivalent to \(q_0>\rho\), since
\begin{align}
\alpha_\infty=\frac{1-\rho}{1+\rho}
\end{align}

Fix \(0<\varepsilon<1-\alpha_\infty\), and define
\begin{align}
r_\varepsilon
\coloneqq
\frac{\varepsilon}{2-\varepsilon}
\end{align}
The condition \(\alpha(t)\ge 1-\varepsilon\) is equivalent to
\begin{align}
q(t)\le r_\varepsilon
\end{align}
Also, \(0<\varepsilon<1-\alpha_\infty\) is equivalent to \(r_\varepsilon<\rho\). Therefore
\begin{align}
q_0>\rho>r_\varepsilon,
\end{align}
so the trajectory begins outside the near-perfect band. Since \(q(t)\to\rho>r_\varepsilon\), it also eventually leaves the band.

By \autoref{theorem:implicit-solution-eigenvalue}, the ridge-normalized eigenvalues satisfy
\begin{align}
F(s(t))
&=
2\gamma b_S t+F(\kappa),
\\
F(n(t))
&=
2\gamma \rho^2 b_S t+F(\kappa q_0),
\end{align}
where
\begin{align}
F(a)=\frac{1}{2}a^2+2a+\log a
\end{align}
At a crossing of the threshold \(q(t)=r_\varepsilon\), we have \(n=r_\varepsilon s\). Eliminating \(t\) gives
\begin{align}
F(r_\varepsilon s)-\rho^2 F(s)
=
F(\kappa q_0)-\rho^2 F(\kappa)
\end{align}
Define
\begin{align}
H_\varepsilon(s)
\coloneqq
F(r_\varepsilon s)-\rho^2 F(s)
\end{align}
Then threshold crossings are exactly the positive solutions of
\begin{align}
H_\varepsilon(s)=C_\kappa,
\qquad
C_\kappa\coloneqq F(\kappa q_0)-\rho^2F(\kappa)
\end{align}

We first note that, for sufficiently small \(\kappa\), there are exactly two threshold crossings. Expanding \(H_\varepsilon\),
\begin{align}
H_\varepsilon(s)
&=
\frac{1}{2}(r_\varepsilon^2-\rho^2)s^2
+2(r_\varepsilon-\rho^2)s
+\log r_\varepsilon
+(1-\rho^2)\log s
\end{align}
Since \(r_\varepsilon<\rho\), we have \(H_\varepsilon(s)\to-\infty\) as \(s\to\infty\). Since \(1-\rho^2>0\), we also have \(H_\varepsilon(s)\to-\infty\) as \(s\downarrow0\). Furthermore,
\begin{align}
s H_\varepsilon'(s)
=
-(\rho^2-r_\varepsilon^2)s^2
+2(r_\varepsilon-\rho^2)s
+(1-\rho^2)
\end{align}
This quadratic has exactly one positive root: its discriminant is positive, and the product of its two roots is negative because its leading coefficient is negative while its constant term is positive. Hence \(H_\varepsilon\) increases once and then decreases once. Also,
\begin{align}
C_\kappa
=
(1-\rho^2)\log \kappa+
\log q_0+
o(1)
\to
-\infty
\qquad
\text{as }\kappa\downarrow0
\end{align}
Therefore, for sufficiently small \(\kappa\), the equation \(H_\varepsilon(s)=C_\kappa\) has two positive solutions. Denote them by
\begin{align}
s_\varepsilon^{\mathrm{in}}(\kappa)
\qquad\text{and}\qquad
s_\varepsilon^{\mathrm{out}}(\kappa),
\end{align}
with \(s_\varepsilon^{\mathrm{in}}(\kappa)<s_\varepsilon^{\mathrm{out}}(\kappa)\). Because \(F\) is strictly increasing and \(q_0>r_\varepsilon\), we have \(H_\varepsilon(\kappa)<C_\kappa\), so the first root occurs after initialization. Since \(s(t)\) is increasing in \(t\), these two roots correspond to the entrance and exit times from the near-perfect abstraction band. Thus
\begin{align}
T_\varepsilon^{\mathrm{in}}(\kappa)
&=
\frac{1}{2\gamma b_S}
\left[
F\!\left(s_\varepsilon^{\mathrm{in}}(\kappa)\right)-F(\kappa)
\right]
\\
T_\varepsilon^{\mathrm{out}}(\kappa)
&=
\frac{1}{2\gamma b_S}
\left[
F\!\left(s_\varepsilon^{\mathrm{out}}(\kappa)\right)-F(\kappa)
\right]
\end{align}
The residence time is
\begin{align}
\mathbf{R}_\varepsilon(\kappa)
=
T_\varepsilon^{\mathrm{out}}(\kappa)
-
T_\varepsilon^{\mathrm{in}}(\kappa)
\end{align}

We now estimate the entrance time. The entrance root satisfies \(s_\varepsilon^{\mathrm{in}}(\kappa)\to0\) as \(\kappa\downarrow0\). Using \(F(a)=\log a+O(a)\) as \(a\downarrow0\), the crossing equation gives
\begin{align}
(1-\rho^2)\log s_\varepsilon^{\mathrm{in}}(\kappa)
+
\log r_\varepsilon
+
o(1)
=
(1-\rho^2)\log \kappa
+
\log q_0
+
o(1)
\end{align}
Therefore
\begin{align}
\log\frac{s_\varepsilon^{\mathrm{in}}(\kappa)}{\kappa}
=
\frac{1}{1-\rho^2}
\log\frac{q_0}{r_\varepsilon}
+
o(1)
\end{align}
Substituting into the expression for \(T_\varepsilon^{\mathrm{in}}(\kappa)\), again using \(F(a)=\log a+O(a)\), yields
\begin{align}
T_\varepsilon^{\mathrm{in}}(\kappa)
=
\frac{1}{2\gamma b_S(1-\rho^2)}
\log\frac{q_0}{r_\varepsilon}
+
o(1)
\end{align}
In particular,
\begin{align}
T_\varepsilon^{\mathrm{in}}(\kappa)=O(1)
\end{align}

We next estimate the exit time. The exit root satisfies \(s_\varepsilon^{\mathrm{out}}(\kappa)\to\infty\). Let \(L_\kappa\coloneqq\log(1/\kappa)\). Since
\begin{align}
C_\kappa
=
F(\kappa q_0)-\rho^2 F(\kappa)
=
-(1-\rho^2)L_\kappa+O(1),
\end{align}
the crossing equation gives
\begin{align}
-\frac{1}{2}(\rho^2-r_\varepsilon^2)
\left(s_\varepsilon^{\mathrm{out}}(\kappa)\right)^2
+
O\!\left(s_\varepsilon^{\mathrm{out}}(\kappa)+\log s_\varepsilon^{\mathrm{out}}(\kappa)\right)
=
-(1-\rho^2)L_\kappa+O(1)
\end{align}
It follows that
\begin{align}
\left(s_\varepsilon^{\mathrm{out}}(\kappa)\right)^2
=
\frac{2(1-\rho^2)}{\rho^2-r_\varepsilon^2}L_\kappa
+
O\!\left(\sqrt{L_\kappa}\right)
\end{align}
Substituting this into
\begin{align}
T_\varepsilon^{\mathrm{out}}(\kappa)
=
\frac{1}{2\gamma b_S}
\left[
F\!\left(s_\varepsilon^{\mathrm{out}}(\kappa)\right)-F(\kappa)
\right],
\end{align}
and using
\begin{align}
F\!\left(s_\varepsilon^{\mathrm{out}}(\kappa)\right)
&=
\frac{1}{2}
\left(s_\varepsilon^{\mathrm{out}}(\kappa)\right)^2
+
O\!\left(\sqrt{L_\kappa}\right),
\\
-F(\kappa)
&=
L_\kappa+O(1),
\end{align}
we obtain
\begin{align}
T_\varepsilon^{\mathrm{out}}(\kappa)
=
\frac{1}{2\gamma b_S}
\frac{1-r_\varepsilon^2}{\rho^2-r_\varepsilon^2}
\log\frac{1}{\kappa}
+
O\!\left(\sqrt{\log\frac{1}{\kappa}}\right)
\end{align}
Since the entrance time is \(O(1)\), the residence time obeys
\begin{align}
\mathbf{R}_\varepsilon(\kappa)
&=
T_\varepsilon^{\mathrm{out}}(\kappa)
-
T_\varepsilon^{\mathrm{in}}(\kappa)
\\
&=
\frac{1}{2\gamma b_S}
\frac{1-r_\varepsilon^2}{\rho^2-r_\varepsilon^2}
\log\frac{1}{\kappa}
+
O\!\left(\sqrt{\log\frac{1}{\kappa}}\right)
\end{align}
Thus
\begin{align}
\mathbf{R}_\varepsilon(\kappa)=\Theta\!\left(\log\frac{1}{\kappa}\right),
\end{align}
and in particular \(\mathbf{R}_\varepsilon(\kappa)\to\infty\) as \(\kappa\downarrow0\). This proves \autoref{proposition:rich-limit-metastability}.
\qed

\clearpage
\subsection{Depth dynamics} \label{appendix:depth-dynamics}

This appendix derives the depth results stated in \autoref{section:depth}. The main goal is to prove \autoref{theorem:layerwise-interp}: first, that balanced deep linear networks interpolate layerwise between the data and the final hidden layer in \(\arctanh\)-space, and second, that the terminal abstraction of the final hidden layer interpolates between the data and the targets in \(\arctanh\)-space.
We then prove that terminal abstraction of the final hidden layer is monotonically increasing with network depth, as discussed in \autoref{eq:deep-terminal-abstraction}.

\paragraph{Preliminaries.}
Define the following for an \(L\)-hidden-layer linear network, where \(\ell\in\{1,\ldots,L\}\) indexes the hidden layers:
\begin{align}
\mathbf{Q}^{(0)} \coloneqq \boldsymbol{\Sigma}_x,
\qquad
\mathbf{Q}^{(\ell)}(t) \coloneqq \mathbf{Z}^{(\ell)}(t)^\top \mathbf{Z}^{(\ell)}(t),
\qquad
\ell \in \{1,\ldots,L\}
\end{align}
When no layer superscript is shown, \(\lambda_m(t)\) denotes the final hidden-layer eigenvalue \(\lambda_m^{(L)}(t)\). We also write
\begin{align}
x_m \coloneqq \lambda_m^{(\boldsymbol{\Sigma}_x)}=\lambda_m^{(0)},
\qquad
y_m \coloneqq \lambda_m^{(\boldsymbol{\Sigma}_y)}
\end{align}
For any 2FS kernel \(A\), recall that
\begin{align}
\nu(A)
=
\frac{\lambda_{SC}^{(A)}}{\lambda_S^{(A)}},
\qquad
\alpha_A
=
\frac{1-\nu(A)}{1+\nu(A)}
\end{align}
Whenever \(0<\nu(A)<\infty\), a useful reparameterization is obtained by applying \(\arctanh\):
\begin{align}
\arctanh \alpha_A
&=
\frac{1}{2}
\log
\frac{1+\alpha_A}{1-\alpha_A}
\\
&=
\frac{1}{2}
\log
\frac{1}{\nu(A)}
\\
&=
\frac{1}{2}
\log
\frac{\lambda_S^{(A)}}{\lambda_{SC}^{(A)}}
\label{eq:appendix-arctanh-log-snr}
\end{align}
Thus \(\arctanh \alpha\) is one half of the log SNR. This is why the depth result takes a particularly simple linear form in \(\arctanh\alpha\)-space. Boundary cases with \(\nu(A)=0\) or \(\nu(A)=\infty\) are understood by continuity whenever the corresponding limit is finite or infinite in the extended real line.

\paragraph{Deriving the balanced deep eigenvalue ODE.}
We first derive the scalar eigenvalue dynamics for the final hidden layer. Let \(g_{j,m}(t)\) denote the scalar gain applied by hidden layer \(j\) to mode \(m \in\{I,S,C,SC,G\}\). Then the final-layer eigenvalue is
\begin{align}
\lambda_m(t)
=
x_m
\prod_{j=1}^{L} g_{j,m}(t)^2 
\label{eq:appendix-final-eigenvalue-product}
\end{align}
The reduced loss from \autoref{assumption:readout} is
\begin{align}
\Lc^*(\mathbf{Q})
=
\frac{1}{2}
\Tr\!
\left[
\boldsymbol{\Sigma}_y (\mathbf{I}+\gamma \mathbf{Q})^{-1}
\right]
\end{align}
Therefore, in the common 2FS eigenbasis,
\begin{align}
\frac{\partial \Lc^*}{\partial \lambda_m}
=
-
\frac{1}{2}
\frac{\gamma y_m}{(1+\gamma\lambda_m)^2}
\end{align}
It is convenient to define the positive effective target gain
\begin{align}
\mu_m(\lambda_m)
\coloneqq
-2
\frac{\partial \Lc^*}{\partial \lambda_m}
=
\frac{
\gamma y_m
}{
(1+\gamma\lambda_m)^2
}
\label{eq:appendix-depth-effective-target-gain}
\end{align}
This is the same spectral filter as in \autoref{eq:appendix-M-filter}. In particular,
\begin{align}
\mu_m(\lambda_m)
\sim
\frac{y_m}{\gamma\lambda_m^2}
\qquad
\text{as}
\qquad
\lambda_m\to\infty
\end{align}
Differentiating \autoref{eq:appendix-final-eigenvalue-product} with respect to \(g_{j,m}\) gives
\begin{align}
\frac{\partial \lambda_m}{\partial g_{j,m}}
=
\frac{2\lambda_m}{g_{j,m}}
\end{align}
Therefore gradient flow gives
\begin{align}
\dot g_{j,m}
&=
-
\frac{\partial \Lc^*}{\partial g_{j,m}}
\\
&=
-
\frac{\partial \Lc^*}{\partial \lambda_m}
\frac{\partial \lambda_m}{\partial g_{j,m}}
\\
&=
\mu_m(\lambda_m)
\frac{\lambda_m}{g_{j,m}}
\end{align}
Now differentiate \(\lambda_m\):
\begin{align}
\dot \lambda_m
&=
\sum_{j=1}^{L}
\frac{\partial \lambda_m}{\partial g_{j,m}}
\dot g_{j,m}
\\
&=
\sum_{j=1}^{L}
\frac{2\lambda_m}{g_{j,m}}
\mu_m(\lambda_m)
\frac{\lambda_m}{g_{j,m}}
\\
&=
2\mu_m(\lambda_m)\lambda_m^2
\sum_{j=1}^{L}
\frac{1}{g_{j,m}^2}
\label{eq:appendix-depth-ode-before-balancing}
\end{align}
Under \autoref{setting:balanced-depth}, all hidden-layer gains for mode \(m\) have the same squared magnitude:
\begin{align}
g_{1,m}^2
=
\cdots
=
g_{L,m}^2
=
u_m^2
\end{align}
Since \(\lambda_m=x_m u_m^{2L}\), we have
\begin{align}
u_m^2
=
\left(
\frac{\lambda_m}{x_m}
\right)^{1/L}
\end{align}
Substituting this into \autoref{eq:appendix-depth-ode-before-balancing} yields
\begin{align}
\dot \lambda_m
&=
2L
\mu_m(\lambda_m)
\lambda_m^2
\left(
\frac{x_m}{\lambda_m}
\right)^{1/L}
\\
&=
2L
\gamma
y_m
x_m^{1/L}
\frac{
\lambda_m^{2-\frac{1}{L}}
}{
(1+\gamma\lambda_m)^2
}
\label{eq:appendix-deep-scalar-ode}
\end{align}
Equivalently, restoring the original notation,
\begin{align}
\dot \lambda_m^{(L)}
=
2L\gamma
\lambda_m^{(\boldsymbol{\Sigma}_y)}
\left(
\lambda_m^{(0)}
\right)^{1/L}
\frac{
\left(\lambda_m^{(L)}\right)^{2-\frac{1}{L}}
}{
\left(1+\gamma\lambda_m^{(L)}\right)^2
}
\end{align}
For \(L=1\), this reduces to the shallow scalar ODE in \autoref{proposition:linear-odes}.

\paragraph{Exact implicit solution of the deep scalar ODE.}
Although the terminal abstraction law can be obtained directly from the large-\(\lambda_m\) asymptotics, it is helpful to record the exact implicit solution. Define the ridge-normalized eigenvalue
\begin{align}
a_m(t)
\coloneqq
\gamma\lambda_m(t)
\end{align}
Then \autoref{eq:appendix-deep-scalar-ode} becomes
\begin{align}
\dot a_m
=
2L
\gamma^{1/L}
x_m^{1/L}
y_m
\frac{
a_m^{2-\frac{1}{L}}
}{
(1+a_m)^2
}
\label{eq:appendix-deep-normalized-ode}
\end{align}
For \(L\geq 1\), define \(G_L:(0,\infty)\to\mathbb{R}\) by
\begin{align}
G_L(a)
\coloneqq
\begin{cases}
\frac{1}{2}a^2+2a+\log a,
&
L=1,
\\[6pt]
\frac{L}{L+1}a^{1+\frac{1}{L}}
+
2L a^{\frac{1}{L}}
-
\frac{L}{L-1}a^{-\frac{L-1}{L}},
&
L>1
\end{cases}
\label{eq:appendix-depth-implicit-antiderivative}
\end{align}
A direct differentiation gives
\begin{align}
G_L'(a)
=
\frac{(1+a)^2}{a^{2-\frac{1}{L}}}
\end{align}
Therefore
\begin{align}
\frac{d}{dt}G_L(a_m(t))
&=
G_L'(a_m(t))\dot a_m(t)
\\
&=
2L
\gamma^{1/L}
x_m^{1/L}
y_m
\end{align}
Integrating in time gives the exact implicit solution
\begin{align}
G_L(a_m(t))
=
2L
\gamma^{1/L}
x_m^{1/L}
y_m t
+
G_L(a_m(0))
\label{eq:appendix-depth-implicit-solution}
\end{align}
For \(L=1\), this reduces to the shallow-network implicit solution from \autoref{theorem:implicit-solution}, with \(G_1(a)=\frac{1}{2}a^2+2a+\log a\) and \(a_m=\gamma\lambda_m\). The implicit solution above assumes \(a_m(0)>0\). If \(a_m(0)=0\), then the corresponding mode remains on the invariant boundary \(a_m(t)=0\) under the scalar ODE.

\paragraph{Layerwise interpolation.}
We now prove the first statement of \autoref{theorem:layerwise-interp}. Under \autoref{setting:balanced-depth},
\begin{align}
\lambda_m^{(\ell)}(t)
=
x_m u_m(t)^{2\ell}
\end{align}
For the signal and interaction modes, this gives
\begin{align}
\nu^{(\ell)}(t)
&\coloneqq
\nu(\mathbf{Q}^{(\ell)}(t))
=
\frac{\lambda_{SC}^{(\ell)}(t)}{\lambda_S^{(\ell)}(t)}
\\
&=
\frac{x_{SC}}{x_S}
\left(
\frac{u_{SC}(t)^2}{u_S(t)^2}
\right)^{\ell}
\\
&=
\nu_X
\left(
\frac{u_{SC}(t)^2}{u_S(t)^2}
\right)^{\ell},
\end{align}
where \(\nu_X\coloneqq\nu(\boldsymbol{\Sigma}_x)\). At the final hidden layer,
\begin{align}
\nu^{(L)}(t)
=
\nu_X
\left(
\frac{u_{SC}(t)^2}{u_S(t)^2}
\right)^L
\end{align}
Eliminating the ratio \(u_{SC}^2/u_S^2\), we obtain
\begin{align}
\nu^{(\ell)}(t)
=
\nu_X^{1-\frac{\ell}{L}}
\left(
\nu^{(L)}(t)
\right)^{\frac{\ell}{L}}
\label{eq:appendix-layerwise-nu-interp}
\end{align}
Taking \(-\frac{1}{2}\log\) of both sides and using \autoref{eq:appendix-arctanh-log-snr} gives
\begin{align}
\arctanh\alpha^{(\ell)}(t)
&=
\left(
1-\frac{\ell}{L}
\right)
\arctanh\alpha^{(\boldsymbol{\Sigma}_x)}
+
\frac{\ell}{L}
\arctanh\alpha^{(L)}(t)
\end{align}
This proves \autoref{eq:layerwise-interp}. The key point is that balancing makes the log inverse-SNR interpolate linearly across depth. Since \(\arctanh\alpha\) is one half of the log SNR, the interpolation is linear in \(\arctanh\alpha\)-space rather than in \(\alpha\)-space itself.

\paragraph{Terminal final-layer abstraction.}
We next prove the second statement of \autoref{theorem:layerwise-interp}. Assume first that the relevant signal and interaction eigenvalues are strictly positive:
\begin{align}
x_S,x_{SC},y_S,y_{SC}>0
\end{align}
From \autoref{eq:appendix-depth-implicit-solution}, and using
\begin{align}
G_L(a)
\sim
\frac{L}{L+1}a^{1+\frac{1}{L}}
\qquad
\text{as}
\qquad
a\to\infty,
\end{align}
we have
\begin{align}
a_m(t)^{1+\frac{1}{L}}
\sim
2(L+1)
\gamma^{1/L}
x_m^{1/L}
y_m t
\end{align}
Because \(a_m=\gamma\lambda_m\), ratios of \(a_m\)'s and \(\lambda_m\)'s are identical. Hence
\begin{align}
\left(
\frac{
\lambda_{SC}^{(L)}(t)
}{
\lambda_S^{(L)}(t)
}
\right)^{1+\frac{1}{L}}
&\longrightarrow
\frac{
x_{SC}^{1/L}y_{SC}
}{
x_S^{1/L}y_S
}
\\
&=
\nu_X^{1/L}\nu_Y,
\end{align}
where
\begin{align}
\nu_Y\coloneqq\nu(\boldsymbol{\Sigma}_y)
=
\frac{y_{SC}}{y_S}
\end{align}
Therefore the terminal final-layer inverse-SNR is
\begin{align}
\nu_{\infty}^{(L)}
\coloneqq
\lim_{t\to\infty}
\nu(\mathbf{Q}^{(L)}(t))
=
\nu_X^{\frac{1}{L+1}}
\nu_Y^{\frac{L}{L+1}}
\label{eq:appendix-deep-terminal-nu}
\end{align}
Applying \(\arctanh\alpha=-\frac{1}{2}\log\nu\), we obtain
\begin{align}
\arctanh\alpha_{\infty}^{(L)}
&=
-
\frac{1}{2}
\log
\nu_{\infty}^{(L)}
\\
&=
-
\frac{1}{2}
\left[
\frac{1}{L+1}\log\nu_X
+
\frac{L}{L+1}\log\nu_Y
\right]
\\
&=
\frac{1}{L+1}
\arctanh\alpha^{(\boldsymbol{\Sigma}_x)}
+
\frac{L}{L+1}
\arctanh\alpha^{(\boldsymbol{\Sigma}_y)}
\label{eq:appendix-deep-terminal-arctanh}
\end{align}
This proves \autoref{eq:deep-terminal-abstraction}. The same formula extends to the boundary cases \(\nu_X=0\), \(\nu_Y=0\), \(\nu_X=\infty\), or \(\nu_Y=\infty\) by taking the corresponding extended-real limit, provided the signal mode used to define the inverse-SNR is nonzero.

Combining the layerwise interpolation with the terminal final-layer law gives the terminal abstraction at every hidden layer:
\begin{align}
\arctanh\alpha_{\infty}^{(\ell;L)}
&=
\left(
1-\frac{\ell}{L}
\right)
\arctanh\alpha^{(\boldsymbol{\Sigma}_x)}
+
\frac{\ell}{L}
\left[
\frac{1}{L+1}\arctanh\alpha^{(\boldsymbol{\Sigma}_x)}
+
\frac{L}{L+1}\arctanh\alpha^{(\boldsymbol{\Sigma}_y)}
\right]
\\
&=
\left(
1-\frac{\ell}{L+1}
\right)
\arctanh\alpha^{(\boldsymbol{\Sigma}_x)}
+
\frac{\ell}{L+1}
\arctanh\alpha^{(\boldsymbol{\Sigma}_y)}
\label{eq:appendix-terminal-layerwise-depth-interp}
\end{align}
Thus, at convergence, the hidden layers occupy equally spaced locations in \(\arctanh\alpha\)-space between the input geometry and the target geometry, with the final hidden layer sitting at fraction \(L/(L+1)\) of the way from inputs to targets. Next, we analytically continue the network depth to positive reals \(L \in \mathbb{R}_+\) and show that terminal abstraction is monotonically increasing with \(L\) whenever \(\nu(\boldsymbol{\Sigma}_y)<\nu(\boldsymbol{\Sigma}_x)\).

\begin{corollary}[Depth monotonicity of terminal abstraction]
\label{corollary:depth-monotonicity}
Assume \(0<\nu(\boldsymbol{\Sigma}_x),\nu(\boldsymbol{\Sigma}_y)<\infty\). Treating the depth \(L\) as a continuous variable, the terminal final-layer abstraction satisfies
\begin{align}
\frac{\partial}{\partial L}
\arctanh\alpha_{\infty}^{(L)}
=
\frac{1}{2(L+1)^2}
\log
\frac{\nu(\boldsymbol{\Sigma}_x)}{\nu(\boldsymbol{\Sigma}_y)}
\label{eq:appendix-depth-derivative}
\end{align}
Consequently, terminal abstraction increases with depth whenever
\begin{align}
\nu(\boldsymbol{\Sigma}_y)<\nu(\boldsymbol{\Sigma}_x),
\end{align}
or equivalently whenever
\begin{align}
\alpha^{(\boldsymbol{\Sigma}_y)}>\alpha^{(\boldsymbol{\Sigma}_x)}
\end{align}
\end{corollary}

\begin{proof}
Let
\begin{align}
A_{\mathbf{X}}
\coloneqq
\arctanh\alpha^{(\boldsymbol{\Sigma}_x)},
\qquad
A_{\mathbf{Y}}
\coloneqq
\arctanh\alpha^{(\boldsymbol{\Sigma}_y)}
\end{align}
From \autoref{eq:deep-terminal-abstraction},
\begin{align}
\arctanh\alpha_{\infty}^{(L)}
=
\frac{A_{\mathbf{X}}+L A_{\mathbf{Y}}}{L+1}
\end{align}
Differentiating with respect to \(L\) gives
\begin{align}
\frac{\partial}{\partial L}
\arctanh\alpha_{\infty}^{(L)}
&=
\frac{
A_{\mathbf{Y}}(L+1)-(A_{\mathbf{X}}+L A_{\mathbf{Y}})
}{
(L+1)^2
}
\\
&=
\frac{A_{\mathbf{Y}}-A_{\mathbf{X}}}{(L+1)^2}
\end{align}
Using \autoref{eq:appendix-arctanh-log-snr},
\begin{align}
A_{\mathbf{Y}}-A_{\mathbf{X}}
&=
-
\frac{1}{2}\log\nu(\boldsymbol{\Sigma}_y)
+
\frac{1}{2}\log\nu(\boldsymbol{\Sigma}_x)
\\
&=
\frac{1}{2}
\log
\frac{\nu(\boldsymbol{\Sigma}_x)}{\nu(\boldsymbol{\Sigma}_y)}
\end{align}
Therefore
\begin{align}
\frac{\partial}{\partial L}
\arctanh\alpha_{\infty}^{(L)}
=
\frac{1}{2(L+1)^2}
\log
\frac{\nu(\boldsymbol{\Sigma}_x)}{\nu(\boldsymbol{\Sigma}_y)}
\end{align}
Since \(\arctanh\) is strictly increasing on \((-1,1)\), the sign of this derivative is also the sign of the change in \(\alpha_{\infty}^{(L)}\) itself. Thus terminal abstraction strictly increases with continuous depth whenever
\begin{align}
\nu(\boldsymbol{\Sigma}_x)>\nu(\boldsymbol{\Sigma}_y)
\end{align}
Finally, the equivalence with \(\alpha^{(\boldsymbol{\Sigma}_y)}>\alpha^{(\boldsymbol{\Sigma}_x)}\) follows because
\begin{align}
\alpha
=
\frac{1-\nu}{1+\nu}
\end{align}
is strictly decreasing in \(\nu\):
\begin{align}
\frac{d}{d\nu}
\frac{1-\nu}{1+\nu}
=
-
\frac{2}{(1+\nu)^2}
<
0
\end{align}
This concludes the proof.
\end{proof}
Therefore terminal abstraction improves monotonically with integer depth whenever the target kernel is more abstract than the input kernel. Note that in actual implementations of neural networks, \(L\) is an integer and so the derivative \(\partial/\partial L\) should be interpreted with some care. The corresponding finite difference has the same sign:
\begin{align}
\arctanh\alpha_{\infty}^{(L+1)}
-
\arctanh\alpha_{\infty}^{(L)}
=
\frac{
\arctanh\alpha^{(\boldsymbol{\Sigma}_y)}-\arctanh\alpha^{(\boldsymbol{\Sigma}_x)}
}{
(L+1)(L+2)
}
\label{eq:appendix-depth-finite-difference}
\end{align}
Therefore the depth-monotonicity conclusion is not an artifact of differentiating with respect to a continuous relaxation of depth.

\paragraph{Discussion on \autoref{setting:balanced-depth}.}
We now discuss and motivate the balanced depth setting. This assumption is the mode-wise analogue of the standard balancedness condition for deep linear networks. Write the gain of hidden layer \(j\) on 2FS mode \(m\) as \(g_{j,m}(t)\), and define the squared gain
\begin{align}
    h_{j,m}(t)
    \coloneqq
    g_{j,m}(t)^2 
\end{align}
Then the eigenvalue of the representation kernel after layer \(\ell\) can be written as
\begin{align}
    \lambda_m^{(\ell)}(t)
    =
    \lambda_m^{(0)}
    \prod_{j=1}^{\ell} h_{j,m}(t)
    \label{eq:appendix-layerwise-general-gains}
\end{align}
Thus \autoref{setting:balanced-depth} is the special case
\begin{align}
    h_{1,m}(t)=h_{2,m}(t)=\cdots=h_{L,m}(t)
    \qquad
    \text{for every mode } m
    \label{eq:appendix-mode-balanced-gains}
\end{align}
Equivalently, for each mode \(m\), the representation growth is distributed equally across hidden layers. This is closely related to the usual notion of balancedness in the deep linear network literature. In \cite{aroraConvergenceAnalysis19}, adjacent weight matrices are called \(\delta\)-balanced when
\begin{align}
    \left\|
    \mathbf{W}_{j+1}^{\top}\mathbf{W}_{j+1}
    -
    \mathbf{W}_j \mathbf{W}_j^{\top}
    \right\|_F
    \leq
    \delta
    \qquad
    \text{for all adjacent layers } j
    \label{eq:appendix-arora-balancedness}
\end{align}
The perfectly balanced case \(\delta=0\) forces adjacent layers to have the same nonzero singular values. This is essentially the case studied in \cite{kuninGetRich24}. If the singular directions are aligned with the 2FS modes relevant to our reduced dynamics, then \autoref{eq:appendix-arora-balancedness} reduces exactly to \autoref{setting:balanced-depth}, i.e.:
\begin{align}
    h_{j+1,m}=h_{j,m}
    \qquad
    \text{for every } j,m
\end{align}

There are several natural ways to initialize inside \autoref{setting:balanced-depth}. In the reduced mode-wise description, one can simply choose, for each mode \(m\), a positive initial per-layer squared gain \(\rho_m>0\) and set
\begin{align}
    h_{1,m}(0)
    =
    h_{2,m}(0)
    =
    \cdots
    =
    h_{L,m}(0)
    =
    \rho_m 
\end{align}
Then
\begin{align}
    \lambda_m^{(\ell)}(0)
    =
    \lambda_m^{(0)}
    \rho_m^{\ell}
\end{align}
Equivalently, if one wants to prescribe an initial final-layer eigenvalue \(\lambda_{m,\mathrm{init}}^{(L)}\) for a mode with \(\lambda_m^{(0)}>0\), one should choose
\begin{align}
    \rho_m
    =
    \left(
    \frac{
        \lambda_{m,\mathrm{init}}^{(L)}
    }{
        \lambda_m^{(0)}
    }
    \right)^{1/L}
    \label{eq:appendix-balanced-init-rho}
\end{align}
This is the scalar-mode version of the SVD-balanced initialization procedure used for deep linear networks: instead of placing all of an end-to-end singular value in one layer, one distributes its \(L\)-th root evenly across the \(L\) hidden factors. At the matrix level, a concrete exact construction is to choose hidden factors whose singular vectors align with the relevant mode subspaces and whose singular values are distributed equally across depth. For \autoref{fig:linear-theory}A--B, we use the orthogonal initialization such that \(\mathbf{W}_\ell(0)=\sigma O_\ell\) for an orthogonal matrix \(O_\ell\) and the same \(\sigma\) for every hidden layer.

Furthermore, if \autoref{setting:balanced-depth} holds at initialization, then it remains balanced throughout training:

\begin{lemma}[Persistence of mode-wise balancedness]
\label{lemma:balanced-gains-persist}
Assume the reduced loss depends on the hidden factors only through the final hidden-layer kernel eigenvalues
\begin{align}
    \lambda_m^{(L)}(t)
    =
    \lambda_m^{(0)}
    \prod_{j=1}^{L} h_{j,m}(t)
\end{align}
Under hidden-layer gradient flow with a shared learning rate, the adjacent squared-gain differences
\begin{align}
    h_{j+1,m}(t)-h_{j,m}(t)
\end{align}
are conserved for every mode \(m\) and every adjacent pair \(j,j+1\). In particular, if the hidden gains are balanced at initialization, then they remain balanced for all training time.
\end{lemma}

\begin{proof}
Fix a mode \(m\) and write
\begin{align}
    \lambda_m(t)
    \equiv
    \lambda_m^{(L)}(t)
    =
    \lambda_m^{(0)}
    \prod_{j=1}^{L} g_{j,m}(t)^2
\end{align}
Since \(\Lc^*\) depends on \(g_{j,m}\) only through \(\lambda_m\), the chain rule gives
\begin{align}
    \frac{\partial \Lc^*}{\partial g_{j,m}}
    =
    \frac{\partial \Lc^*}{\partial \lambda_m}
    \frac{\partial \lambda_m}{\partial g_{j,m}}
\end{align}
But
\begin{align}
    \frac{\partial \lambda_m}{\partial g_{j,m}}
    =
    \frac{2\lambda_m}{g_{j,m}}
\end{align}
Therefore gradient flow gives
\begin{align}
    \dot g_{j,m}
    =
    -
    \frac{\partial \Lc^*}{\partial g_{j,m}}
    =
    -
    \frac{\partial \Lc^*}{\partial \lambda_m}
    \frac{2\lambda_m}{g_{j,m}}
\end{align}
Now differentiate the squared gain \(h_{j,m}=g_{j,m}^2\):
\begin{align}
    \dot h_{j,m}
    &=
    2g_{j,m}\dot g_{j,m}
    \\
    &=
    -4\lambda_m
    \frac{\partial \Lc^*}{\partial \lambda_m}
    \label{eq:appendix-hdot-independent-of-layer}
\end{align}
The right-hand side of \autoref{eq:appendix-hdot-independent-of-layer} depends on the mode \(m\) and on the final eigenvalue \(\lambda_m\), but it does not depend on the layer index \(j\). Hence, for every adjacent pair,
\begin{align}
    \frac{d}{dt}
    \left(
    h_{j+1,m}-h_{j,m}
    \right)
    =
    \dot h_{j+1,m}-\dot h_{j,m}
    =
    0
\end{align}
Thus all adjacent squared-gain differences are conserved. If
\begin{align}
    h_{j+1,m}(0)=h_{j,m}(0)
\end{align}
at initialization, then
\begin{align}
    h_{j+1,m}(t)=h_{j,m}(t)
\end{align}
for all \(t\geq0\). This proves the claim.
\end{proof}

\clearpage
\section{Derivations and proofs for results in \autoref{section:nonlinear} (nonlinear networks)}

\subsection{Preservation of 2FS in large width} \label{appendix:2fs-preserved}

This section proves the claim used in \autoref{section:nonlinear}: in the infinite-width limit, if the preactivation kernel \(\mathbf{Q}\) is 2FS, then applying either the erf or leaky-ReLU nonlinearity produces a feature kernel \(\mathbf{K}\) that is also 2FS. We also derive the entrywise kernel maps in \autoref{eq:kernel-maps}. 

\paragraph{2FS as invariance under sample permutations.}

Recall from \autoref{assumption:2fs} that a kernel \(A\) is 2FS if
\begin{align}
\boldsymbol{\Pi}_g A \boldsymbol{\Pi}_g^\top = A,
\qquad
\forall g\in \mathcal{G},
\end{align}
where \(\mathcal{G}\cong (S_n)^4\rtimes(\mathbb{Z}_2)^2\) acts by arbitrary within-class permutations together with global relabelings \(s\mapsto -s\) and \(c\mapsto -c\). Equivalently, \(A\) is constant on the five pair orbits induced by this group. Writing sample \(i\) as \(i=(s_i,c_i,r_i)\), where \(r_i\in\{1,\ldots,n\}\) indexes the exemplar inside the fine class, these five entry types are
\begin{align}
a_{ij}
=
\begin{cases}
a_d, & i=j,\\
a_2, & i\neq j,\ s_i=s_j,\ c_i=c_j,\\
a_{1s}, & s_i=s_j,\ c_i\neq c_j,\\
a_{1c}, & s_i\neq s_j,\ c_i=c_j,\\
a_0, & s_i\neq s_j,\ c_i\neq c_j
\end{cases}
\end{align}
Thus, to prove that \(\mathbf{K}\) is 2FS, it is enough to show that \(\mathbf{K}_{ij}\) depends only on which of these five categories the pair \((i,j)\) belongs to.

\paragraph{Infinite-width feature kernel.}

Let \(\mathbf{z}^{(a)}\in\mathbb{R}^N\) denote the vector of preactivations of hidden unit \(a\) across the \(N\) samples. We assume that, in the infinite-width limit the rows are iid draws from:
\begin{align}
\mathbf{z}^{(a)} \stackrel{\mathrm{iid}}{\sim} \mathcal{N}(0,\mathbf{Q}),
\qquad
a=1,\ldots,D
\end{align}
For a coordinatewise nonlinearity \(\phi\), define
\begin{align}
\mathbf{H}_{ai}=\phi(z_i^{(a)}),
\qquad
\mathbf{K} \coloneqq \frac{1}{D}\mathbf{H}^\top \mathbf{H}
\end{align}
Then, entrywise,
\begin{align}
\mathbf{K}_{ij}
=
\frac{1}{D}\sum_{a=1}^D
\phi(z_i^{(a)})\phi(z_j^{(a)})
\end{align}
Since \(N\) is fixed and \(\phi(z_i)\phi(z_j)\) has finite expectation for both erf and leaky ReLU, the law of large numbers gives
\begin{align}
\mathbf{K}_{ij}
\xrightarrow[D\to\infty]{\mathrm{a.s.}}
\mathbb{E}_{z\sim\mathcal{N}(0,\mathbf{Q})}
\left[
\phi(z_i)\phi(z_j)
\right]
\end{align}
Thus the infinite-width feature kernel is the deterministic matrix
\begin{align}
\mathbf{K}_\phi(\mathbf{Q})
=
\mathbb{E}_{z\sim\mathcal{N}(0,\mathbf{Q})}
\left[
\phi(\mathbf{z})\phi(\mathbf{z})^\top
\right]
\end{align}
At finite width, \(\mathbf{K}\) has \(D^{-1/2}\)-scale empirical fluctuations around this limit, so \(\mathbf{K}\) is only exactly 2FS in the infinite-width limit.

\begin{proposition}[Coordinatewise nonlinearities preserve 2FS in the Gaussian infinite-width limit]
\label{proposition:2fs-preserved-infinite-width}
Let \(\mathbf{Q}\) be 2FS, so \(\boldsymbol{\Pi}_g \mathbf{Q} \boldsymbol{\Pi}_g^\top=\mathbf{Q}\) for every \(g\in\mathcal{G}\). Let \(z\sim\mathcal{N}(0,\mathbf{Q})\), and let
\begin{align}
\mathbf{K}_\phi(\mathbf{Q})
\coloneqq
\mathbb{E}
\left[
\phi(\mathbf{z})\phi(\mathbf{z})^\top
\right],
\end{align}
where \(\phi\) acts coordinatewise. Then
\begin{align}
\boldsymbol{\Pi}_g \mathbf{K}_\phi(\mathbf{Q}) \boldsymbol{\Pi}_g^\top = \mathbf{K}_\phi(\mathbf{Q}),
\qquad
\forall g\in\mathcal{G}
\end{align}
Therefore \(\mathbf{K}_\phi(\mathbf{Q})\) is 2FS. In particular, erf and leaky ReLU preserve 2FS in the infinite-width kernel limit.
\end{proposition}

\begin{proof}
Because \(\mathbf{z} \sim\mathcal{N}(0,\mathbf{Q})\) and \(\boldsymbol{\Pi}_g \mathbf{Q} \boldsymbol{\Pi}_g^\top=\mathbf{Q}\), we have
\begin{align}
\boldsymbol{\Pi}_g \mathbf{z}
\sim
\mathcal{N}(0,\boldsymbol{\Pi}_g \mathbf{Q} \boldsymbol{\Pi}_g^\top)
=
\mathcal{N}(0,\mathbf{Q})
\end{align}
Thus \(\boldsymbol{\Pi}_g \mathbf{z}\stackrel{d}{=}\mathbf{z}\). Also, since \(\phi\) is applied coordinatewise, it commutes with sample permutations:
\begin{align}
\phi(\boldsymbol{\Pi}_g \mathbf{z})=\boldsymbol{\Pi}_g\phi(\mathbf{z})
\end{align}
Using these two facts,
\begin{align}
\boldsymbol{\Pi}_g \mathbf{K}_\phi(\mathbf{Q})\boldsymbol{\Pi}_g^\top
&=
\mathbb{E}
\left[
\boldsymbol{\Pi}_g\phi(\mathbf{z})\phi(\mathbf{z})^\top \boldsymbol{\Pi}_g^\top
\right]
\\
&=
\mathbb{E}
\left[
\phi(\boldsymbol{\Pi}_g \mathbf{z})\phi(\boldsymbol{\Pi}_g \mathbf{z})^\top
\right]
\\
&=
\mathbb{E}
\left[
\phi(\mathbf{z})\phi(\mathbf{z})^\top
\right]
\\
&=
\mathbf{K}_\phi(\mathbf{Q})
\end{align}
This proves that \(\mathbf{K}_\phi(\mathbf{Q})\) is invariant under every permutation in \(\mathcal{G}\), hence is 2FS.
\end{proof}


\paragraph{Entrywise form of the infinite-width map.} We now derive the explicit maps from the 2FS entries of \(\mathbf{Q}\) to the 2FS entries of \(\mathbf{K}\), which are the maps used in \autoref{section:nonlinear}.

Assume \(q_d>0\). If \(q_d=0\), positive semidefiniteness forces \(\mathbf{Q}=0\), so \(z=0\) almost surely and both erf and leaky-ReLU feature kernels are the zero matrix; the claim then follows trivially. Since \(\mathbf{Q}\) is 2FS, all diagonal entries equal \(q_d\). For any pair \((i,j)\), the pair \((z_i,z_j)\) is centered jointly Gaussian with covariance
\begin{align}
\begin{pmatrix}
q_d & q_{ij}\\
q_{ij} & q_d
\end{pmatrix}
\end{align}
Therefore
\begin{align}
\mathbf{K}_{ij}
=
\mathbb{E}
\left[
\phi(z_i)\phi(z_j)
\right]
=
\psi_\phi(q_{ij},q_d)
\end{align}
for a scalar function \(\psi_\phi\) depending only on the nonlinearity. Hence, for each 2FS entry type \(\mu\in\{d,2,1s,1c,0\}\),
\begin{align}
k_\mu = \psi_\phi(q_\mu,q_d)
\end{align}
This immediately implies that \(\mathbf{K}\) has the same five-entry 2FS structure as \(\mathbf{Q}\). The rest of this appendix derives \(\psi_\phi\) for erf and leaky ReLU.

\paragraph{Erf kernel map.}

Let
\begin{align}
\begin{pmatrix}
u\\
v
\end{pmatrix}
\sim
\mathcal{N}
\left(
0,
\begin{pmatrix}
q_d & q\\
q & q_d
\end{pmatrix}
\right)
\end{align}
We use the convention
\begin{align}
\phi_\beta(x)
=
\mathrm{erf}(\beta x),
\qquad
\beta>0
\end{align}
For a scalar \(x\), the error function has the Gaussian sign representation
\begin{align}
\mathrm{erf}(x)
=
\mathbb{E}_{\epsilon}
\left[
\operatorname{sign}(x+\epsilon)
\right],
\qquad
\epsilon\sim\mathcal{N}\!\left(0,\frac{1}{2}\right)
\end{align}
Indeed,
\begin{align}
\mathbb{E}_{\epsilon}
\left[
\operatorname{sign}(x+\epsilon)
\right]
=
2\Phi(\sqrt{2}x)-1
=
\mathrm{erf}(x)
\end{align}
Let \(\epsilon_1,\epsilon_2\stackrel{\mathrm{iid}}{\sim}\mathcal{N}(0,1/2)\), independent of \((u,v)\), and define
\begin{align}
A
=
\beta u+\epsilon_1,
\qquad
B
=
\beta v+\epsilon_2
\end{align}
Then \((A,B)\) is centered jointly Gaussian with
\begin{align}
\operatorname{Var}(A)
=
\operatorname{Var}(B)
=
\beta^2 q_d+\frac{1}{2},
\qquad
\operatorname{Cov}(A,B)
=
\beta^2 q
\end{align}
Therefore the correlation of \((A,B)\) is
\begin{align}
r
=
\frac{\beta^2 q}{\beta^2 q_d+\frac{1}{2}}
=
\frac{2\beta^2 q}{1+2\beta^2 q_d}
\end{align}
For a centered bivariate Gaussian pair with correlation \(r\), the Gaussian sign identity gives
\begin{align}
\mathbb{E}
\left[
\operatorname{sign}(A)\operatorname{sign}(B)
\right]
=
\frac{2}{\pi}\arcsin(r)
\end{align}
Combining these identities,
\begin{align}
\mathbb{E}
\left[
\mathrm{erf}(\beta u)\mathrm{erf}(\beta v)
\right]
&=
\mathbb{E}
\left[
\operatorname{sign}(A)\operatorname{sign}(B)
\right]
\\
&=
\frac{2}{\pi}
\arcsin\!\left(
\frac{2\beta^2 q}{1+2\beta^2 q_d}
\right)
\end{align}
Therefore, if \(\mathbf{Q}\) is 2FS, the erf feature-kernel entries are
\begin{align}
k_\mu
=
\psi_\beta(q_\mu,q_d)
=
\frac{2}{\pi}
\arcsin\!\left(
\frac{2\beta^2 q_\mu}{1+2\beta^2 q_d}
\right),
\qquad
\mu\in\{d,2,1s,1c,0\}
\end{align}

\paragraph{Leaky-ReLU kernel map.}

Now let
\begin{align}
\phi_\omega(x)
=
\max\{x,\omega x\}
=
\omega x+(1-\omega)x_+,
\qquad
x_+\coloneqq \max\{x,0\},
\qquad
\omega\in[0,1]
\end{align}
Using the same centered Gaussian pair \((u,v)\) with variance \(q_d\) and covariance \(q\), we expand
\begin{align}
\mathbb{E}
\left[
\phi_\omega(u)\phi_\omega(v)
\right]
&=
\omega^2\mathbb{E}[uv]
+
\omega(1-\omega)
\left(
\mathbb{E}[u v_+]
+
\mathbb{E}[u_+ v]
\right)
+
(1-\omega)^2
\mathbb{E}[u_+v_+]
\end{align}
The mixed terms are simple. Since \(\mathbb{E}[u\mid v]=(q/q_d)v\),
\begin{align}
\mathbb{E}[u v_+]
&=
\mathbb{E}
\left[
\mathbb{E}[u\mid v]\,v\mathbf{1}_{\{v>0\}}
\right]
\\
&=
\frac{q}{q_d}
\mathbb{E}
\left[
v^2\mathbf{1}_{\{v>0\}}
\right]
\\
&=
\frac{q}{2}
\end{align}
By symmetry, \(\mathbb{E}[u_+v]=q/2\). Hence
\begin{align}
\mathbb{E}
\left[
\phi_\omega(u)\phi_\omega(v)
\right]
=
\omega q
+
(1-\omega)^2
\mathbb{E}[u_+v_+]
\end{align}
Now define
\begin{align}
\rho\coloneqq \frac{q}{q_d},
\qquad
\theta_\rho \coloneqq \arccos(\rho)
\end{align}
For a centered bivariate Gaussian pair with equal variances, the standard ReLU kernel integral is
\begin{align}
\mathbb{E}[u_+v_+]
=
\frac{1}{2\pi}
\left[
\sqrt{q_d^2-q^2}
+
q(\pi-\theta_\rho)
\right]
\end{align}
One quick derivation is by Price's theorem: if \(F(q)\coloneqq \mathbb{E}[u_+v_+]\) with \(q_d\) fixed, then
\begin{align}
\frac{\partial F}{\partial q}
=
\mathbb{P}(u>0,v>0)
=
\frac{\pi-\arccos(q/q_d)}{2\pi},
\end{align}
and \(F(0)=\mathbb{E}[u_+]\mathbb{E}[v_+]=q_d/(2\pi)\). Integrating from \(0\) to \(q\) gives the displayed expression.

Substituting the ReLU-ReLU term into the leaky-ReLU expansion gives
\begin{align}
\mathbb{E}
\left[
\phi_\omega(u)\phi_\omega(v)
\right]
&=
\omega q
+
\frac{(1-\omega)^2}{2\pi}
\left[
\sqrt{q_d^2-q^2}
+
q(\pi-\theta_\rho)
\right]
\\
&=
\left[
\omega+\frac{(1-\omega)^2}{2}
\right]q
+
\frac{(1-\omega)^2}{2\pi}
\left[
\sqrt{q_d^2-q^2}
-
q\theta_\rho
\right]
\\
&=
\frac{1+\omega^2}{2}q
+
\frac{(1-\omega)^2}{2\pi}
\left[
\sqrt{q_d^2-q^2}
-
q\arccos\left(\frac{q}{q_d}\right)
\right]
\end{align}
Therefore the leaky-ReLU kernel map is
\begin{align}
\psi_\omega(q,q_d)
=
\frac{1+\omega^2}{2}q
+
\frac{(1-\omega)^2}{2\pi}
\left[
\sqrt{q_d^2-q^2}
-
q\arccos\left(\frac{q}{q_d}\right)
\right]
\end{align}
Thus, if \(\mathbf{Q}\) is 2FS,
\begin{align}
k_\mu
=
\psi_\omega(q_\mu,q_d)
=
\frac{1+\omega^2}{2}q_\mu
+
\frac{(1-\omega)^2}{2\pi}
\left[
\sqrt{q_d^2-q_\mu^2}
-
q_\mu\arccos\left(\frac{q_\mu}{q_d}\right)
\right],
\qquad
\mu\in\{d,2,1s,1c,0\}
\end{align}
Hence the leaky-ReLU feature kernel is 2FS whenever the preactivation kernel is 2FS.

\clearpage
\subsection{Nonlinear network preactivation dynamics} \label{appendix:nonlinear-riccati}

Under readout optimality, the reduced objective depends on the feature kernel only through
\begin{align}
L^*(\mathbf{K})
&=
\frac{1}{2}\operatorname{Tr}\!\big[\boldsymbol{\Sigma}_y (\mathbf{I}+\gamma \mathbf{K})^{-1}\big],
&
\nabla_{\mathbf{H}} L^*
&=
- \mathbf{H} \mathbf{M},
\end{align}
with
\begin{align}
\mathbf{M} = \gamma (\mathbf{I}+\gamma \mathbf{K})^{-1}\boldsymbol{\Sigma}_y(\mathbf{I}+\gamma \mathbf{K})^{-1}
\end{align}
For one neuron row \( \mathbf{z} \in \mathbb{R}^N \), let \( \mathbf{h} = \phi(\mathbf{z}) \). Since \( \phi \) acts coordinatewise,
\begin{align}
\dot{\mathbf{z}} = \boldsymbol{\Sigma}_x \big(\mathbf{M}\mathbf{h} \odot \phi'(\mathbf{z})\big)
\end{align}
Therefore, with \( \mathbf{Q} = \mathbb{E}[\mathbf{z}\mathbf{z}^\top] \),
\begin{align}
\dot{\mathbf{Q}}
=
\mathbb{E}[\mathbf{z}\dot{\mathbf{z}}^\top] + \mathbb{E}[\dot{\mathbf{z}} \mathbf{z}^\top]
\end{align}
In both nonlinearities considered here, Stein identities reduce the mixed expectation to
\begin{align}
\mathbb{E}[\mathbf{z}\dot{\mathbf{z}}^\top] = \boldsymbol{\Sigma}_x \mathbf{C},
\qquad
\mathbf{C} = \mathbf{R}\mathbf{Q},
\end{align}
which gives the nonlinear Riccati equation
\begin{align}
\dot{\mathbf{Q}} = \boldsymbol{\Sigma}_x \mathbf{R}\mathbf{Q} + \mathbf{Q}\mathbf{R}\boldsymbol{\Sigma}_x
\label{eq:nonlinear-riccati}
\end{align}
Since \( \mathbf{R} \) is 2FS, \autoref{eq:nonlinear-riccati} diagonalizes in the projector basis, yielding
\begin{align}
\dot{\lambda}_\bullet^Q &= 2\lambda_\bullet^Q \lambda_\bullet^X \lambda_\bullet^R,
\\
\dot{\alpha}_{S,\mathbf{Q}} &= (1-\alpha_Q^2)\lambda_S^X \lambda_S^R(1-\nu_X\nu_R)
\end{align}

\paragraph{Leaky ReLU case.}
For \( \phi_\omega(z)=\max\{z,\omega z\} \), we have
\begin{align}
\phi_\omega'(z)=\omega + (1-\omega)\mathbf{1}_{\{z>0\}}
\end{align}
If \( (u,v) \) are centered jointly Gaussian with correlation \( \rho \), then
\begin{align}
\mathbf{D}_\omega(\rho)
&\coloneqq
(1-\omega)\,\mathbb{E}\!\big[\delta(u)\phi_\omega(v)\big]
=
\frac{(1-\omega)^2}{2\pi}\sqrt{1-\rho^2},
\\
\mathbf{E}_\omega(\rho)
&\coloneqq
\mathbb{E}\!\big[\phi_\omega'(u)\phi_\omega'(v)\big]
=
\omega + \frac{(1-\omega)^2}{2\pi}\big(\pi-\arccos \rho\big)
\end{align}
Substituting these into the Stein expansion gives
\begin{align}
\mathbf{R}_\omega
=
\operatorname{Diag}\!\big((\mathbf{M}\odot \mathbf{D}_\omega)\mathbf{1}\big)
+
(\mathbf{M}\odot \mathbf{E}_\omega)
\end{align}
Moreover, since
\begin{align}
\mathbf{D}_\omega = (1-\omega)^2 \mathbf{D}_{\mathrm{ReLU}},
\qquad
\mathbf{E}_\omega = \omega \mathbf{1}\mathbf{1}^\top + (1-\omega)^2 \mathbf{E}_{\mathrm{ReLU}},
\end{align}
we obtain the exact decomposition
\begin{align}
\mathbf{R}_\omega = \omega \mathbf{M} + (1-\omega)^2 \mathbf{R}_{\mathrm{ReLU}}
\label{eq:leaky-gain-decomposition}
\end{align}

\paragraph{Erf case.}

For
\begin{align}
\phi_\beta(z)
=
\mathrm{erf}(\beta z),
\end{align}
the infinite-width kernel map is
\begin{align}
\mathbf{K}_{ij}
=
\frac{2}{\pi}
\arcsin\!\left(
\frac{2\beta^2 \mathbf{Q}_{ij}}{1+2\beta^2 q_d}
\right)
\end{align}
Define the erf derivative-gain kernel
\begin{align}
(\mathbf{G}_\beta)_{ij}
&\coloneqq
\mathbb{E}
\left[
\phi_\beta'(z_i)\phi_\beta'(z_j)
\right]
\\
&=
\frac{4\beta^2/\pi}
{
\sqrt{
(1+2\beta^2 q_d)^2
-
4\beta^4 \mathbf{Q}_{ij}^2
}
}
\label{eq:erf-gain}
\end{align}
The second derivative satisfies
\begin{align}
\phi_\beta''(z)
=
-2\beta^2 z\,\phi_\beta'(z)
\end{align}
For a centered Gaussian pair \((z_i,z_j)\) with common variance \(q_d\), set
\begin{align}
T_{ij}
\coloneqq
\mathbb{E}
\left[
\phi_\beta(z_j)\phi_\beta''(z_i)
\right]
\end{align}
By Stein's lemma,
\begin{align}
\mathbb{E}
\left[
z_i\phi_\beta'(z_i)\phi_\beta(z_j)
\right]
=
q_d T_{ij}
+
\mathbf{Q}_{ij}(\mathbf{G}_\beta)_{ij}
\end{align}
Using \(\phi_\beta''(z_i)=-2\beta^2 z_i\phi_\beta'(z_i)\), this gives
\begin{align}
T_{ij}
=
-
\frac{2\beta^2}{1+2\beta^2 q_d}
\mathbf{Q}_{ij}
(\mathbf{G}_\beta)_{ij}
\end{align}
Hence, with
\begin{align}
\zeta_\beta
\coloneqq
\frac{2\beta^2}{1+2\beta^2 q_d},
\end{align}
the Stein expansion of the nonlinear preactivation dynamics yields
\begin{align}
\mathbf{R}_\beta
=
\mathbf{M}\odot \mathbf{G}_\beta
-
\zeta_\beta
\operatorname{Diag}\!\big(
(\mathbf{M}\odot \mathbf{G}_\beta\odot \mathbf{Q})\mathbf{1}
\big)
\end{align}

\clearpage
\subsection{Erf terminal abstraction} \label{appendix:erf-fixed-point}

Throughout this subsection, write
\begin{align}
\nu_X
\coloneqq
\frac{\lambda_{SC}^{\mathbf{X}}}{\lambda_S^{\mathbf{X}}},
\qquad
\nu_Y
\coloneqq
\frac{\lambda_{SC}^{\mathbf{Y}}}{\lambda_S^{\mathbf{Y}}},
\qquad
\nu_Q
\coloneqq
\frac{\lambda_{SC}^{\mathbf{Q}}}{\lambda_S^{\mathbf{Q}}},
\qquad
\nu_K
\coloneqq
\frac{\lambda_{SC}^{\mathbf{K}}}{\lambda_S^{\mathbf{K}}}
\end{align}
For the erf dynamics, define
\begin{align}
\mathbf{A}
\coloneqq
\mathbf{M}\odot \mathbf{G}_\beta,
\qquad
\mathbf{B}
\coloneqq
\mathbf{M}\odot \mathbf{G}_\beta\odot \mathbf{Q}
\end{align}
Because \(\mathbf{B}\) is 2FS, every row sum is identical. Hence
\begin{align}
\operatorname{Diag}(\mathbf{B}\mathbf{1})
=
b\mathbf{I}
\end{align}
for a scalar \(b\). From the corrected erf Stein term,
\begin{align}
\mathbf{R}_\beta
=
\mathbf{A}-s_\beta \mathbf{I},
\qquad
s_\beta
\coloneqq
\zeta_\beta b,
\qquad
\zeta_\beta
=
\frac{2\beta^2}{1+2\beta^2 q_d}
\end{align}
At the level of 2FS eigenvalues,
\begin{align}
\lambda_\bullet^{\mathbf{R}_\beta}
=
\lambda_\bullet^{\mathbf{A}}
-
s_\beta,
\qquad
\bullet\in\{I,S,C,SC,G\}
\end{align}
Therefore the gain inverse-SNR is
\begin{align}
\nu_{R_\beta}
&=
\frac{
\lambda_{SC}^{\mathbf{A}}-s_\beta
}{
\lambda_S^{\mathbf{A}}-s_\beta
}
\\
&=
\nu_A
\frac{
1-s_\beta/\lambda_{SC}^{\mathbf{A}}
}{
1-s_\beta/\lambda_S^{\mathbf{A}}
}
\end{align}
Define
\begin{align}
E_{\mathrm{Stein}}
&\coloneqq
\frac{
1-s_\beta/\lambda_{SC}^{\mathbf{A}}
}{
1-s_\beta/\lambda_S^{\mathbf{A}}
},
&
E_J
&\coloneqq
\frac{\nu_A}{\nu_M}
\end{align}
Since
\begin{align}
\mathbf{M}
=
\gamma
(\mathbf{I}+\gamma \mathbf{K})^{-1}
\boldsymbol{\Sigma}^{\mathbf{Y}}
(\mathbf{I}+\gamma \mathbf{K})^{-1},
\end{align}
we have
\begin{align}
\nu_M
&=
\nu_Y
\left(
\frac{
1+\gamma\lambda_S^{\mathbf{K}}
}{
1+\gamma\lambda_{SC}^{\mathbf{K}}
}
\right)^2
\\
&=
\frac{\nu_Y}{\nu_K^2}
\left(
\frac{
1+(\gamma\lambda_S^{\mathbf{K}})^{-1}
}{
1+(\gamma\lambda_{SC}^{\mathbf{K}})^{-1}
}
\right)^2
\end{align}
Now define
\begin{align}
E_\gamma
&\coloneqq
\left(
\frac{
1+(\gamma\lambda_S^{\mathbf{K}})^{-1}
}{
1+(\gamma\lambda_{SC}^{\mathbf{K}})^{-1}
}
\right)^2,
&
E_{\mathbf{Q}\to \mathbf{K}}
&\coloneqq
\left(
\frac{\nu_Q}{\nu_K}
\right)^2
\end{align}
Combining the previous displays gives
\begin{align}
\nu_{R_\beta}
=
\frac{\nu_Y}{\nu_Q^2}
\underbrace{
E_{\mathbf{Q}\to \mathbf{K}}
E_J
E_{\mathrm{Stein}}
E_\gamma
}_{E_{\mathrm{erf}}}
\label{eq:erf-nuR}
\end{align}
The nontrivial fixed point of \(\dot{\alpha}_{S,\mathbf{Q}}\) is characterized by
\begin{align}
1-\nu_X\nu_{R_\beta}
=
0
\end{align}
Substituting \autoref{eq:erf-nuR} gives
\begin{align}
(\nu_Q^*)^2
=
\nu_X\nu_Y E_{\mathrm{erf}},
\qquad
\alpha_{\mathbf{Q},\infty}^{\mathrm{erf}}
=
\frac{
1-\sqrt{\nu_X\nu_Y E_{\mathrm{erf}}}
}{
1+\sqrt{\nu_X\nu_Y E_{\mathrm{erf}}}
}
\end{align}

To obtain the small-\(\beta\) expansion, define the corrected normalized erf correlation parameter
\begin{align}
c_\beta
\coloneqq
\frac{2\beta^2 q_d}{1+2\beta^2 q_d},
\qquad
u_\mu
\coloneqq
c_\beta\rho_\mu,
\qquad
u_*
\coloneqq
\max_\mu |u_\mu|
\end{align}
If \(q_d\) and the off-diagonal \(q_\mu\) remain \(O(1)\), then \(u_*=O(\beta^2)\). Since
\begin{align}
\arcsin u
=
u+O(u^3),
\end{align}
the corrected erf map gives
\begin{align}
\nu_K
=
\nu_Q
\bigl(1+O(u_*^2)\bigr),
\qquad
E_{\mathbf{Q}\to \mathbf{K}}
=
1+O(\beta^4)
\end{align}
Similarly, from the corrected \autoref{eq:erf-gain},
\begin{align}
\mathbf{G}_\beta
=
\frac{
4\beta^2
}{
\pi(1+2\beta^2 q_d)
}
\mathbf{1}\mathbf{1}^{\top}
+
O(\beta^6),
\end{align}
entrywise on the finite set of 2FS entry types. Therefore
\begin{align}
\nu_A
=
\nu_M
\bigl(1+O(\beta^4)\bigr),
\qquad
E_J
=
1+O(\beta^4)
\end{align}
Moreover, in the nondegenerate regime where the ratios
\(b/\lambda_S^{\mathbf{A}}\) and \(b/\lambda_{SC}^{\mathbf{A}}\) remain \(O(1)\),
\begin{align}
\frac{s_\beta}{\lambda_\bullet^{\mathbf{A}}}
=
O(\beta^2),
\qquad
\bullet\in\{S,SC\},
\end{align}
and hence
\begin{align}
E_{\mathrm{Stein}}
=
1+O(\beta^2)
\end{align}
Finally, since \(\mathbf{K}=O(\beta^2)\) in the small-\(\beta\) regime, the nonzero signal and interaction eigenvalues satisfy \(\lambda_\bullet^{\mathbf{K}}\asymp \beta^2\). Thus, when \(\gamma\beta^2\to\infty\),
\begin{align}
E_\gamma
=
1+
O\!\left(
\frac{1}{\gamma\beta^2}
\right)
\end{align}
Combining these estimates,
\begin{align}
E_{\mathrm{erf}}
=
1+
O\!\left(
\beta^2
+
\frac{1}{\gamma\beta^2}
\right)
\end{align}
Consequently,
\begin{align}
\alpha_{\mathbf{Q},\infty}^{\mathrm{erf}}
=
\alpha_{\infty}^{\mathrm{linear}}
+
O\!\left(
\beta^2
+
\frac{1}{\gamma\beta^2}
\right),
\end{align}
where
\begin{align}
\alpha_{\infty}^{\mathrm{linear}}
=
\frac{
1-\sqrt{\nu_X\nu_Y}
}{
1+\sqrt{\nu_X\nu_Y}
}
\end{align}

\clearpage
\subsection{Leaky ReLU terminal abstraction} \label{appendix:relu-fixedpoint}

We derive the scalar terminal equation for leaky-ReLU preactivations and its closed form along the noiseless-target edge. Throughout this appendix, we work in the class-level, centered, shape--color balanced 2FS reduction. Namely, we assume that the global target mode has been removed,
\begin{align}
\lambda_G^{\mathbf{Y}}=0,
\end{align}
that shape and color are signal-balanced,
\begin{align}
\lambda_S^{\mathbf{X}}=\lambda_C^{\mathbf{X}},
\qquad
\lambda_S^{\mathbf{Y}}=\lambda_C^{\mathbf{Y}},
\qquad
\lambda_S^{\mathbf{Q}(0)}=\lambda_C^{\mathbf{Q}(0)},
\end{align}
and that the initialized preactivation kernel is centered,
\begin{align}
\lambda_G^{\mathbf{Q}(0)}=0
\end{align}
These conditions are preserved by the dynamics. We write
\begin{align}
\nu_X
=
\frac{\lambda_{SC}^{\mathbf{X}}}{\lambda_S^{\mathbf{X}}},
\qquad
\nu_Y
=
\frac{\lambda_{SC}^{\mathbf{Y}}}{\lambda_S^{\mathbf{Y}}},
\qquad
\nu_Q
=
\frac{\lambda_{SC}^{\mathbf{Q}}}{\lambda_S^{\mathbf{Q}}}
\end{align}
The within-class \(I\)-mode is not relevant for the class-mean abstraction calculation below and is suppressed from the notation.

\paragraph{Balanced 2FS parameterization of the preactivation kernel.}

Let the two factors be \(s,c\in\{-1,+1\}\). In the centered balanced reduction, the class-level preactivation kernel can be written as
\begin{align}
\mathbf{Q}_{(s,c),(s',c')}
=
\frac{q}{4}
\left(
ss'
+
cc'
+
\nu_Q ss'cc'
\right),
\end{align}
where \(q=\lambda_S^{\mathbf{Q}}=\lambda_C^{\mathbf{Q}}>0\). The diagonal entry is
\begin{align}
q_d
=
\frac{q}{4}(2+\nu_Q)
\end{align}
There are two off-diagonal correlations. If the two classes differ in both factors, then
\begin{align}
\rho_0
=
\frac{\mathbf{Q}_{(s,c),(-s,-c)}}{q_d}
=
\frac{\nu_Q-2}{\nu_Q+2}
\end{align}
If they differ in exactly one factor, then
\begin{align}
\rho_1
=
\frac{\mathbf{Q}_{(s,c),(s,-c)}}{q_d}
=
-\frac{\nu_Q}{\nu_Q+2}
\end{align}

Define the angular coordinate
\begin{align}
\theta
=
\arccos
\left(
\frac{\nu_Q-2}{\nu_Q+2}
\right)
\in[0,\pi]
\end{align}
Then
\begin{align}
\rho_0=\cos\theta
\end{align}
Solving for \(\nu_Q\), we get
\begin{align}
\nu_Q
=
\frac{2(1+\cos\theta)}{1-\cos\theta}
\end{align}
Therefore
\begin{align}
\rho_1
=
-\frac{\nu_Q}{\nu_Q+2}
=
-\frac{1+\cos\theta}{2}
\end{align}
Finally, since
\begin{align}
\alpha_Q
=
\frac{1-\nu_Q}{1+\nu_Q},
\end{align}
substituting the expression for \(\nu_Q\) gives
\begin{align}
\alpha_Q(\theta)
=
-\frac{1+3\cos\theta}{3+\cos\theta}
\end{align}
Thus \(\theta=\pi\) corresponds to \(\nu_Q=0\) and \(\alpha_Q=1\), while \(\theta=0\) corresponds to \(\nu_Q=\infty\) and \(\alpha_Q=-1\).

\paragraph{Feature inverse-SNR induced by leaky ReLU.}

Let
\begin{align}
\phi_\omega(z)=\max\{z,\omega z\},
\qquad
\omega\in[0,1]
\end{align}
For a centered Gaussian pair \((U,V)\) with equal variance and correlation \(\rho\), the normalized leaky-ReLU kernel map is
\begin{align}
g_\omega(\rho)
=
\rho
+
\frac{(1-\omega)^2}{\pi(1+\omega^2)}
\left(
\sqrt{1-\rho^2}
-
\rho\arccos\rho
\right)
\end{align}
Equivalently,
\begin{align}
g_\omega(\rho)
=
\frac{2\omega}{1+\omega^2}\rho
+
\frac{(1-\omega)^2}{1+\omega^2}
\frac{
\sqrt{1-\rho^2}
+
(\pi-\arccos\rho)\rho
}{\pi}
\end{align}
The two forms are identical. The second form makes clear that leaky ReLU interpolates between the linear normalized kernel \(g_1(\rho)=\rho\) and the pure-ReLU normalized kernel.

The feature kernel \(\mathbf{K}\) has diagonal entry normalized to \(1\), one-factor-flip entry \(g_\omega(\rho_1)\), and two-factor-flip entry \(g_\omega(\rho_0)\), up to an irrelevant positive scale. For a \(2\times2\) balanced class kernel with diagonal entry \(a\), one-factor-flip entry \(b\), and two-factor-flip entry \(d\), the \(S\)-mode eigenvalue is proportional to \(a-d\), while the \(SC\)-mode eigenvalue is proportional to \(a-2b+d\). Hence the feature inverse-SNR induced by a preactivation angle \(\theta\) is
\begin{align}
\nu_K(\theta,\omega)
=
\frac{
1-2g_\omega(\rho_1)+g_\omega(\rho_0)
}{
1-g_\omega(\rho_0)
},
\end{align}
where
\begin{align}
\rho_0=\cos\theta,
\qquad
\rho_1=-\frac{1+\cos\theta}{2}
\end{align}

\paragraph{Effective target ratio in the terminal regime.}

The effective label kernel is
\begin{align}
\mathbf{M}(\mathbf{K})
=
\gamma
(\mathbf{I}+\gamma \mathbf{K})^{-1}
\boldsymbol{\Sigma}_y
(\mathbf{I}+\gamma \mathbf{K})^{-1}
\end{align}
Since \(\mathbf{K}\) and \(\boldsymbol{\Sigma}_y\) are simultaneously diagonal in the 2FS basis, its eigenvalues are
\begin{align}
\lambda_m^{\mathbf{M}}
=
\gamma
\frac{\lambda_m^{\mathbf{Y}}}{(1+\gamma\lambda_m^{\mathbf{K}})^2}
\end{align}
In the terminal large-feature-norm regime, \(\gamma\lambda_m^{\mathbf{K}}\gg1\) for the active modes, so
\begin{align}
\lambda_m^{\mathbf{M}}
\sim
\frac{\lambda_m^{\mathbf{Y}}}{\gamma(\lambda_m^{\mathbf{K}})^2}
\end{align}
Therefore the effective target interaction-to-signal ratio is
\begin{align}
\mu(\theta)
\coloneqq
\frac{\lambda_{SC}^{\mathbf{M}}}{\lambda_S^{\mathbf{M}}}
=
\frac{\nu_Y}{\nu_K(\theta,\omega)^2}
\end{align}

\paragraph{Gain inverse-SNR.}

Here we  compute the inverse-SNR of the leaky-ReLU gain kernel \(\mathbf{R}_\omega\). For pure ReLU, define
\begin{align}
D(\rho)
=
\frac{\sqrt{1-\rho^2}}{2\pi},
\qquad
E(\rho)
=
\frac{\pi-\arccos\rho}{2\pi}
\end{align}
Set
\begin{align}
D_0=D(\rho_0),
\qquad
E_0=E(\rho_0),
\qquad
D_1=D(\rho_1),
\qquad
E_1=E(\rho_1)
\end{align}

Let
\begin{align}
m=\lambda_S^{\mathbf{M}}=\lambda_C^{\mathbf{M}},
\qquad
\lambda_{SC}^{\mathbf{M}}=\mu m
\end{align}
Then the balanced class-level entries of \(\mathbf{M}\) are
\begin{align}
\mathbf{M}_d=\frac{m}{4}(2+\mu),
\qquad
\mathbf{M}_1=-\frac{m\mu}{4},
\qquad
\mathbf{M}_0=\frac{m}{4}(\mu-2),
\end{align}
where \(\mathbf{M}_d\) is the diagonal entry, \(\mathbf{M}_1\) is the one-factor-flip entry, and \(\mathbf{M}_0\) is the two-factor-flip entry.

For pure ReLU, the gain has entries
\begin{align}
\mathbf{R}_d^{+}
=
\mathbf{M}_dE(1)
+
2\mathbf{M}_1D_1
+
\mathbf{M}_0D_0,
\end{align}
\begin{align}
\mathbf{R}_1^{+}
=
\mathbf{M}_1E_1,
\qquad
\mathbf{R}_0^{+}
=
\mathbf{M}_0E_0
\end{align}
Here the superscript \(+\) denotes the pure-ReLU contribution. Since \(E(1)=1/2\), we get
\begin{align}
\frac{\mathbf{R}_d^{+}}{m}
=
\frac{2+\mu}{8}
-
\frac{\mu}{2}D_1
+
\frac{\mu-2}{4}D_0
\end{align}
The \(S\)-mode eigenvalue is proportional to \(\mathbf{R}_d^{+}-\mathbf{R}_0^{+}\), while the \(SC\)-mode eigenvalue is proportional to \(\mathbf{R}_d^{+}-2\mathbf{R}_1^{+}+\mathbf{R}_0^{+}\). Thus define
\begin{align}
A_S(\theta,\mu)
=
\frac{2+\mu}{8}
-
\frac{\mu}{2}D_1
+
\frac{\mu-2}{4}D_0
-
\frac{\mu-2}{4}E_0,
\end{align}
and
\begin{align}
A_{SC}(\theta,\mu)
=
\frac{2+\mu}{8}
-
\frac{\mu}{2}D_1
+
\frac{\mu-2}{4}D_0
+
\frac{\mu}{2}E_1
+
\frac{\mu-2}{4}E_0
\end{align}
These are the pure-ReLU gain eigenvalues normalized by \(m\):
\begin{align}
\lambda_S^{\mathbf{R}_{\mathrm{ReLU}}}=mA_S(\theta,\mu),
\qquad
\lambda_{SC}^{\mathbf{R}_{\mathrm{ReLU}}}=mA_{SC}(\theta,\mu)
\end{align}

For leaky ReLU,
\begin{align}
\mathbf{R}_\omega
=
\omega \mathbf{M}
+
(1-\omega)^2\mathbf{R}_{\mathrm{ReLU}}
\end{align}
Therefore
\begin{align}
\lambda_S^{\mathbf{R}_\omega}
=
\omega m
+
(1-\omega)^2mA_S(\theta,\mu),
\end{align}
and
\begin{align}
\lambda_{SC}^{\mathbf{R}_\omega}
=
\omega\mu m
+
(1-\omega)^2mA_{SC}(\theta,\mu)
\end{align}
Thus the gain inverse-SNR is
\begin{align} \label{eq:inverse-snr-relu}
\nu_R(\theta;\omega,\nu_Y)
=
\frac{
\omega\mu(\theta)
+
(1-\omega)^2A_{SC}(\theta,\mu(\theta))
}{
\omega
+
(1-\omega)^2A_S(\theta,\mu(\theta))
}
\end{align}

\paragraph{Scalar terminal condition.}

The preactivation abstraction dynamics have the form
\begin{align}
\dot{\alpha}_{\mathbf{Q}}
=
(1-\alpha_Q^2)
\lambda_S^{\mathbf{X}}
\lambda_S^{\mathbf{R}_\omega}
\left(
1-\nu_X\nu_R
\right)
\end{align}
Therefore any interior terminal point must satisfy
\begin{align}
\nu_X\nu_R(\theta;\omega,\nu_Y)=1
\end{align}
Once \(\theta\) is found, the corresponding terminal preactivation abstraction is
\begin{align}
\alpha_{\mathbf{Q},\infty}
=
-\frac{1+3\cos\theta}{3+\cos\theta}
\end{align}

When \(\omega=1\), the ReLU gate term disappears. Then \(\mathbf{K}=\mathbf{Q}\), so \(\nu_K=\nu_Q\), and
\begin{align}
\nu_R
=
\mu
=
\frac{\nu_Y}{\nu_Q^2}
\end{align}
The terminal equation becomes
\begin{align}
\nu_X\frac{\nu_Y}{\nu_Q^2}=1,
\end{align}
which gives the linear-network law
\begin{align}
\nu_{\mathbf{Q},\infty}
=
\sqrt{\nu_X\nu_Y}
\end{align}

\paragraph{Noiseless-target edge.}

Now set
\begin{align}
\nu_Y=0
\end{align}
Then
\begin{align}
\mu(\theta)=0
\end{align}
The dependence on the feature inverse-SNR \(\nu_K\) drops out entirely.

Since
\begin{align}
\rho_0=\cos\theta,
\end{align}
we have
\begin{align}
D_0
=
\frac{\sin\theta}{2\pi},
\qquad
E_0
=
\frac{\pi-\theta}{2\pi}
\end{align}
Substituting \(\mu=0\) into \(A_S\) and \(A_{SC}\) gives
\begin{align}
A_S(\theta,0)
=
\frac14
-
\frac12D_0
+
\frac12E_0
=
\frac{2\pi-\theta-\sin\theta}{4\pi},
\end{align}
and
\begin{align}
A_{SC}(\theta,0)
=
\frac14
-
\frac12D_0
-
\frac12E_0
=
\frac{\theta-\sin\theta}{4\pi}
\end{align}
Therefore
\begin{align}
\nu_R(\theta;\omega,0)
=
\frac{
(1-\omega)^2(\theta-\sin\theta)
}{
4\pi\omega
+
(1-\omega)^2(2\pi-\theta-\sin\theta)
}
\end{align}
The interior terminal equation \(\nu_X\nu_R=1\) is equivalently
\begin{align}
\nu_X
=
F_\omega(\theta)
\coloneqq
\frac{
4\pi\omega
+
(1-\omega)^2(2\pi-\theta-\sin\theta)
}{
(1-\omega)^2(\theta-\sin\theta)
}
\end{align}

We now analyze this curve. Let
\begin{align}
b=(1-\omega)^2
\end{align}
Then
\begin{align}
F_\omega(\theta)
=
\frac{
4\pi\omega
+
b(2\pi-\theta-\sin\theta)
}{
b(\theta-\sin\theta)
}
\end{align}
For \(\theta\in(0,\pi)\),
\begin{align}
\theta-\sin\theta>0
\end{align}
Also,
\begin{align}
\lim_{\theta\downarrow0}F_\omega(\theta)=+\infty,
\end{align}
because
\begin{align}
\theta-\sin\theta
=
\frac{\theta^3}{6}
+
O(\theta^5)
\end{align}
At the perfect-abstraction endpoint \(\theta=\pi\),
\begin{align}
F_\omega(\pi)
=
\frac{
4\pi\omega+b\pi
}{
b\pi
}
=
1+\frac{4\omega}{(1-\omega)^2}
\end{align}
Therefore define
\begin{align}
\nu_{\mathrm{crit}}(\omega)
=
1+\frac{4\omega}{(1-\omega)^2}
\end{align}

The function \(F_\omega\) is strictly decreasing. Indeed,
\begin{align}
F_\omega'(\theta)
=
\frac{
-b(1+\cos\theta)(\theta-\sin\theta)
-
\left[
4\pi\omega
+
b(2\pi-\theta-\sin\theta)
\right]
(1-\cos\theta)
}{
b(\theta-\sin\theta)^2
}
\end{align}
Every term in the numerator is nonpositive, and for \(\theta\in(0,\pi)\) the numerator is strictly negative. Hence \(F_\omega\) decreases from \(+\infty\) to \(\nu_{\mathrm{crit}}(\omega)\).

It follows that:

\begin{align}
\alpha_{\mathbf{Q},\infty}=1
\quad
\text{if}
\quad
\nu_X\le \nu_{\mathrm{crit}}(\omega),
\end{align}
and for
\begin{align}
\nu_X>\nu_{\mathrm{crit}}(\omega),
\end{align}
there is a unique interior terminal point determined by
\begin{align}
\nu_X
=
\frac{
4\pi\omega
+
(1-\omega)^2(2\pi-\theta-\sin\theta)
}{
(1-\omega)^2(\theta-\sin\theta)
}
\end{align}

In terms of
\begin{align}
\eta_X
=
\frac{\nu_X}{1+\nu_X},
\end{align}
the critical value is
\begin{align}
\eta_{\mathrm{crit}}(\omega)
=
\frac{\nu_{\mathrm{crit}}(\omega)}
{1+\nu_{\mathrm{crit}}(\omega)}
=
\frac{
4\omega+(1-\omega)^2
}{
4\omega+2(1-\omega)^2
}
\end{align}
For pure ReLU, \(\omega=0\), this gives
\begin{align}
\nu_{\mathrm{crit}}(0)=1,
\qquad
\eta_{\mathrm{crit}}(0)=\frac12
\end{align}

\paragraph{High-input-noise asymptotic on the noiseless edge.}

We now analyze the bottom-right corner of the phase diagram:
\begin{align}
\nu_Y=0,
\qquad
\nu_X\to\infty
\end{align}
This corresponds to \(\theta\downarrow0\). Using
\begin{align}
\sin\theta
=
\theta-\frac{\theta^3}{6}+O(\theta^5),
\end{align}
we get
\begin{align}
\theta-\sin\theta
=
\frac{\theta^3}{6}+O(\theta^5),
\end{align}
and
\begin{align}
2\pi-\theta-\sin\theta
=
2\pi-2\theta+\frac{\theta^3}{6}+O(\theta^5)
\end{align}
Therefore
\begin{align}
F_\omega(\theta)
=
\frac{
12\pi(1+\omega^2)
}{
(1-\omega)^2\theta^3
}
-
\frac{12}{\theta^2}
+
O(\theta^{-1})
\end{align}
Define
\begin{align}
A_\omega
=
\frac{
12\pi(1+\omega^2)
}{
(1-\omega)^2
}
\end{align}
Inverting the expansion gives
\begin{align}
\theta
=
A_\omega^{1/3}\nu_X^{-1/3}
-
4A_\omega^{-1/3}\nu_X^{-2/3}
+
O(\nu_X^{-1})
\end{align}
Next, near \(\theta=0\),
\begin{align}
\alpha_Q(\theta)
=
-\frac{1+3\cos\theta}{3+\cos\theta}
=
-1+\frac{\theta^2}{4}+O(\theta^4)
\end{align}
Substituting the expansion for \(\theta\), we obtain
\begin{align}
\alpha_{\mathbf{Q},\infty}
=
-1
+
\frac14
\left(
\frac{
12\pi(1+\omega^2)
}{
(1-\omega)^2\nu_X
}
\right)^{2/3}
-
\frac{2}{\nu_X}
+
O(\nu_X^{-4/3})
\end{align}
In particular,
\begin{align}
\alpha_{\mathbf{Q},\infty}
\to
-1
\qquad
\text{as}
\qquad
\nu_X\to\infty
\end{align}
For pure ReLU, \(\omega=0\), this becomes
\begin{align}
\alpha_{\mathbf{Q},\infty}
=
-1
+
\frac14
\left(
\frac{12\pi}{\nu_X}
\right)^{2/3}
-
\frac{2}{\nu_X}
+
O(\nu_X^{-4/3})
\end{align}

\clearpage
\subsection{Attenuation law} \label{appendix:attenuation}

\paragraph{Preliminaries.}
\autoref{section:attenuation} compares the abstraction measured before and after the coordinatewise nonlinearity. The preactivation kernel is
\begin{align}
\mathbf{Q} \coloneqq \frac{1}{D}\mathbf{Z}^\top \mathbf{Z},
\end{align}
and the feature kernel is
\begin{align}
\mathbf{K} \coloneqq \frac{1}{D}\mathbf{H}^\top \mathbf{H},
\qquad
\mathbf{H}=\phi(\mathbf{Z})
\end{align}
The aim of this subsection is to prove the formal version of \autoref{thm:nonlinear-attenuation}. The result is \emph{instantaneous}: we fix a time \(t\), and ask how abstraction is changed from \(\mathbf{Q}\) to \(\mathbf{K}\) by applying \(\phi\) to a fixed preactivation geometry. Note this is an independent question from how the nonlinearities affect the training dynamics, which are discussed in \autoref{section:nonlinear-dynamics} and \autoref{appendix:nonlinear-riccati}. Throughout, assume the fixed preactivation kernel \(\mathbf{Q}\) is 2FS and positive semidefinite with common diagonal entry \(q_d>0\). In the infinite-width limit, \autoref{appendix:2fs-preserved} shows that the feature kernel is also 2FS and its five entry types are deterministic functions of the five entry types of \(\mathbf{Q}\). We write the 2FS entries of a generic kernel \(A\) as follows: (also see \autoref{fig:2fs-ansatz}D):
\begin{align}
 a_d,
 \qquad
 a_2,
 \qquad
 a_{1s},
 \qquad
 a_{1c},
 \qquad
 a_0,
\end{align}
where the subscripts mean diagonal, same fine class, same shape but different color, same color but different shape, and different in both factors, respectively.

\paragraph{Entrywise abstraction formula.}
The derivation below translates the geometric cosine definition from \autoref{section:min-model} into a formula involving only 2FS kernel entries. This is the entrywise analogue of \autoref{proposition:2fs-abstraction-eigenvalues}.

For a 2FS kernel \(A\), define the within-class centroid self-inner-product
\begin{align}
\bar a
\coloneqq
\frac{a_d+(n-1)a_2}{n}
\end{align}
Let \(\boldsymbol{\mu}_{s,c}^{A}\) denote the class centroid in the representation whose Gram matrix is \(A\). The two context-specific shape directions are:
\begin{align}
\mathbf{u}_{+,A}
\coloneqq
\boldsymbol{\mu}_{+,+}^{A}-\boldsymbol{\mu}_{-,+}^{A}, \qquad 
\mathbf{u}_{-,A}
\coloneqq
\boldsymbol{\mu}_{+,-}^{A}-\boldsymbol{\mu}_{-,-}^{A}
\end{align}
By the five-entry 2FS structure,
\begin{align}
\|\mathbf{u}_{+,A}\|^2
=
\|\mathbf{u}_{-,A}\|^2
&=
2(\bar a-a_{1c}),
\\
\langle \mathbf{u}_{+,A},\mathbf{u}_{-,A}\rangle
&=
2(a_{1s}-a_0)
\end{align}
Therefore, whenever \(\bar a-a_{1c}>0\), the shape abstraction induced by \(A\) is
\begin{align}
\alpha_{S,A}
=
\frac{\langle \mathbf{u}_{+,A},\mathbf{u}_{-,A}\rangle}
{\|\mathbf{u}_{+,A}\|\,\|\mathbf{u}_{-,A}\|}
=
\frac{a_{1s}-a_0}{\bar a-a_{1c}}
\label{eq:appendix-attenuation-alpha-entry}
\end{align}
Equivalently,
\begin{align}
\alpha_{S,A}
=
\frac{\lambda_S^{(A)}-\lambda_{SC}^{(A)}}
{\lambda_S^{(A)}+\lambda_{SC}^{(A)}}
\end{align}
In the main text, \(\alpha_Q\) and \(\alpha_K\) refer to \(\alpha_Q\) and \(\alpha_K\), respectively.

\paragraph{Normalized NNGP maps.}
The nonlinearities considered in \autoref{section:nonlinear} act entrywise on the 2FS kernel in the infinite-width limit. It is useful to separate the overall variance scale from the correlation dependence. Define normalized preactivation correlations
\begin{align}
\rho_\mu
\coloneqq
\frac{q_\mu}{q_d},
\qquad
\mu\in\{d,2,1s,1c,0\},
\qquad
\rho_d=1,
\end{align}
and
\begin{align}
\bar\rho
\coloneqq
\frac{1+(n-1)\rho_2}{n}
\end{align}
If the infinite-width NNGP map has the form
\begin{align}
k_\mu=\psi(q_\mu,q_d),
\end{align}
then we define its normalized scalar map by:
\begin{align}
g(\rho)
\coloneqq
\frac{\psi(q_d\rho,q_d)}{\psi(q_d,q_d)}
\label{eq:appendix-attenuation-normalized-map}
\end{align}
Then \(g(1)=1\), \(k_\mu=k_d g(\rho_\mu)\) with \(k_d=\psi(q_d,q_d)>0\), and the common scale \(k_d\) cancels from the cosine. Applying \autoref{eq:appendix-attenuation-alpha-entry} to \(\mathbf{Q}\) and \(\mathbf{K}\) gives
\begin{align}
\alpha_Q
&=
\frac{\rho_{1s}-\rho_0}{\bar\rho-\rho_{1c}},
\label{eq:appendix-attenuation-alpha-q}
\\
\alpha_K
&=
\frac{g(\rho_{1s})-g(\rho_0)}{\bar g-g(\rho_{1c})},
\qquad
\bar g
\coloneqq
\frac{1+(n-1)g(\rho_2)}{n}
\label{eq:appendix-attenuation-alpha-k}
\end{align}

\paragraph{Exact expression for the attenuation factor.}
We first derive the exact multiplicative factor relating \(\alpha_K\) and \(\alpha_Q\). For a scalar function \(g\), define the secant slope
\begin{align}
s_g(x,y)
\coloneqq
\begin{cases}
\dfrac{g(x)-g(y)}{x-y}, & x\neq y,\\[1.25em]
g'(x), & x=y,
\end{cases}
\label{eq:appendix-attenuation-secant}
\end{align}
where the second line is used when \(g\) is differentiable at \(x\). Then
\begin{align}
g(\rho_{1s})-g(\rho_0)
=
s_g(\rho_{1s},\rho_0)(\rho_{1s}-\rho_0)
\end{align}
For the denominator,
\begin{align}
\bar g-g(\rho_{1c})
&=
\frac{1}{n}\bigl(g(1)-g(\rho_{1c})\bigr)
+
\frac{n-1}{n}\bigl(g(\rho_2)-g(\rho_{1c})\bigr)
\\
&=
\Bigl[
 w_d s_g(1,\rho_{1c})
 +
 w_2 s_g(\rho_2,\rho_{1c})
 \Bigr]
(\bar\rho-\rho_{1c}),
\end{align}
where
\begin{align}
w_d
&\coloneqq
\frac{1}{n}\frac{1-\rho_{1c}}{\bar\rho-\rho_{1c}},
&
w_2
&\coloneqq
\frac{n-1}{n}\frac{\rho_2-\rho_{1c}}{\bar\rho-\rho_{1c}}
\label{eq:appendix-attenuation-weights}
\end{align}
These weights satisfy \(w_d+w_2=1\). Substituting into \autoref{eq:appendix-attenuation-alpha-q} and \autoref{eq:appendix-attenuation-alpha-k} gives the exact identity
\begin{align}
\alpha_K
&=
\mathcal{A}_g(\mathbf{Q})\,\alpha_Q,
\label{eq:appendix-attenuation-exact-law}
\\
\mathcal{A}_g(\mathbf{Q})
&\coloneqq
\frac{s_g(\rho_{1s},\rho_0)}
{w_d s_g(1,\rho_{1c})+w_2 s_g(\rho_2,\rho_{1c})}
\label{eq:appendix-attenuation-factor}
\end{align}
Thus \autoref{eq:appendix-attenuation-factor} gives the exact expression for the attenuation factor. The next question is when \(\mathcal{A}_g(\mathbf{Q})\) lies in \([0,1]\). The reslt below proves this for the two nonlinearities used in the main text, erf and L-ReLU. 

\begin{theorem}[Attenuation law, formal version of \autoref{thm:nonlinear-attenuation}]
\label{theorem:attenuation-law-formal}
Fix a time \(t\). Suppose \autoref{assumption:2fs} holds and the preactivation kernel \(\mathbf{Q}\) is nonzero, PSD, and 2FS with diagonal entry \(q_d>0\). Assume \(\lambda_S^{(\mathbf{Q})}+\lambda_{SC}^{(\mathbf{Q})}>0\), so \(\alpha_Q\) is defined. Let \(\mathbf{K}_\phi(\mathbf{Q})\) be the feature kernel obtained by applying \(\phi\) coordinatewise to the preactivations in the infinite-width limit \(D \to \infty\).

For both nonlinearities
\begin{align}
\phi_\beta(z)
&=
\mathrm{erf}(\beta z),
\qquad
\beta>0,
\\
\phi_\omega(z)
&=
\max\{z,\omega z\},
\qquad
\omega\in[0,1],
\end{align}
the corresponding feature-space abstraction satisfies
\begin{align}
\alpha_{\mathbf{K}_\phi}
=
\mathcal{A}_\phi(\mathbf{Q})\,\alpha_Q,
\qquad
0\le \mathcal{A}_\phi(\mathbf{Q})\le 1,
\label{eq:appendix-attenuation-theorem-law}
\end{align}
where \(\mathcal{A}_\phi(\mathbf{Q})\) is the exact factor in \autoref{eq:appendix-attenuation-factor} with \(g=g_\phi\). Consequently,
\begin{align}
|\alpha_{\mathbf{K}_\phi}|
\le
|\alpha_Q|
\end{align}
In particular, if \(\alpha_Q>0\), then
\begin{align}
0\le \alpha_{\mathbf{K}_\phi}\le \alpha_Q,
\end{align}
so applying either nonlinearity cannot improve positive shape abstraction at that time.
\end{theorem}

\begin{proof}
The proof has three parts. First, we prove a general attenuation lemma for normalized scalar maps with nonnegative power-series coefficients and monotonicity. Second, we verify that the \(\mathrm{erf}\) NNGP map satisfies these conditions. Third, we verify the same conditions for leaky ReLU.

\paragraph{Part 1: Schur-power attenuation lemma.} Define \(\tilde{\mathbf{Q}} = \mathbf{Q} / q_d\) as the 2FS correlation kernel and denote \(\tilde{\mathbf{Q}}^{\circ p}\) as the \(p\)-th Schur power (i.e. elementwise power) of \(\tilde{\mathbf{Q}}\). At a high level, this lemma aims to show that every Schur power \(\tilde{\mathbf{Q}}^{\circ p}\), \(p \geq 1\), has abstraction no larger than \(\tilde{\mathbf{Q}}\), and that positive combinations of such powers also inherit this bound. 


For notational simplicity assume \(n\ge2\). The case \(n=1\) is obtained by deleting the within-class residual mode. Since \(\tilde{\mathbf{Q}} = \mathbf{Q}/q_d\) is PSD and 2FS, its normalized entries can be written as
\begin{align}
1
&=
\tau_G+\tau_S+\tau_C+\tau_{SC}+\tau_I,
\\
\rho_2
&=
\tau_G+\tau_S+\tau_C+\tau_{SC}
-
\frac{\tau_I}{n-1},
\\
\rho_{1s}
&=
\tau_G+\tau_S-\tau_C-\tau_{SC},
\\
\rho_{1c}
&=
\tau_G-\tau_S+\tau_C-\tau_{SC},
\\
\rho_0
&=
\tau_G-\tau_S-\tau_C+\tau_{SC},
\label{eq:appendix-attenuation-tau-decomp}
\end{align}
Where \(\tau_G,\tau_S,\tau_C,\tau_{SC}\) are the normalized class-level mode masses \(\tau_I\) is the normalized within-class residual mass. These are all proportional to the eigenvalues of \(\mathbf{Q}\) such that:
\begin{align}
\tau_m \coloneqq \frac{\lambda_m^{(\mathbf{Q})}}{4n q_d} \geq 0 \quad \text{for } m \in \{G,S,C,SC\},
\end{align}
Now define
\begin{align}
X
&\coloneqq
\tau_G+\tau_C,
&
U
&\coloneqq
\tau_G-\tau_C,
\\
Y
&\coloneqq
\tau_S+\tau_{SC},
&
V
&\coloneqq
\tau_S-\tau_{SC}
\end{align}
Note that:
\begin{align}
X\ge |U|,
\qquad
Y\ge |V|
\end{align}
Then we have the following simple relations:
\begin{align}
\bar\rho=X+Y,
\qquad
\rho_{1c}=X-Y,
\qquad
\rho_{1s}=U+V,
\qquad
\rho_0=U-V
\end{align}
Therefore
\begin{align}
\alpha_Q
=
\frac{\rho_{1s}-\rho_0}{\bar\rho-\rho_{1c}}
=
\frac{2V}{2Y}
=
\frac{V}{Y}
\label{eq:appendix-attenuation-alpha-vy}
\end{align}
The condition \(\lambda_S^{(\mathbf{Q})}+\lambda_{SC}^{(\mathbf{Q})}>0\) is exactly \(Y>0\). Now consider the monomial map \(g_p(\rho)=\rho^p\), with integer \(p\ge1\). Define
\begin{align}
N_p
&\coloneqq
\rho_{1s}^p-\rho_0^p
=
(U+V)^p-(U-V)^p,
\\
\Delta_p
&\coloneqq
\overline{\rho^p}-\rho_{1c}^p,
\qquad
\overline{\rho^p}
\coloneqq
\frac{1+(n-1)\rho_2^p}{n}
\end{align}
Now, we claim the following:
\begin{align}
|N_p|
\le
|\alpha_Q|\,\Delta_p 
\qquad
\text{for every integer }p \geq 1
\label{eq:appendix-attenuation-power-bound}
\end{align}
To prove this claim, our strategy is to (i) lower-bound the denominator; and (ii) upper-bound the numerator. For (i) we begin by introducing the scalar random variable
\begin{align}
\mathbf{W}
=
\begin{cases}
1, & \text{with probability }1/n,\\[0.25em]
-\dfrac{1}{n-1}, & \text{with probability }(n-1)/n
\end{cases}
\end{align}
Then
\begin{align}
\mathbb{E}[\mathbf{W}]=0,
\qquad
\mathbb{E}[\mathbf{W}^k]
=
\frac{1}{n}
\left(
1+\frac{(-1)^k}{(n-1)^{k-1}}
\right)
\ge0
\qquad
\text{for all }k\ge2,
\end{align}
and
\begin{align}
\overline{\rho^p}
=
\mathbb{E}\left[(X+Y+\tau_I \mathbf{W})^p\right]
\end{align}
Because \(X+Y\ge0\), expanding in powers of \(\mathbf{W}\) gives
\begin{align}
\overline{\rho^p}
\ge
(X+Y)^p
\end{align}
Thus we have obtained the lower bound (i):
\begin{align}
\Delta_p
\ge
(X+Y)^p-(X-Y)^p
\label{eq:appendix-attenuation-den-lower}
\end{align}
Now, we derive the upper bound (ii). The right-hand side is nonnegative because \(X,Y\ge0\) imply \(X+Y\ge |X-Y|\). If \(V=0\), then \(N_p=0\), and \autoref{eq:appendix-attenuation-power-bound} is immediately satisfied. Otherwise, using an odd-binomial expansion of \(|N_p|\) and then dividing by \(|V|\) gives:
\begin{align}
\frac{|N_p|}{|V|}
&=
\left|
2\sum_{\substack{1\le j\le p\\ j\ \mathrm{odd}}}
\binom{p}{j}U^{p-j}V^{j-1}
\right|
\\
&\le
2\sum_{\substack{1\le j\le p\\ j\ \mathrm{odd}}}
\binom{p}{j}|U|^{p-j}|V|^{j-1}
\\
&\le
2\sum_{\substack{1\le j\le p\\ j\ \mathrm{odd}}}
\binom{p}{j}X^{p-j}Y^{j-1}
\\
&=
\frac{(X+Y)^p-(X-Y)^p}{Y}
\\
&\le
\frac{\Delta_p}{Y}
\end{align}
Then, multiplying by \(|V|\) gives:
\begin{align}
|N_p| \leq \frac{|V|}{Y} \Delta_p = |\alpha_Q| \Delta_p
\end{align}

Which proves the claim in \autoref{eq:appendix-attenuation-power-bound}. Now suppose the normalized scalar map \(g\) has a power series on correlations \(\rho\) of the form:
\begin{align}
g(\rho)
=
a_0+
\sum_{p=1}^{\infty}a_p\rho^p,
\qquad
a_p\ge0
\quad
\text{for all }p\ge1,
\label{eq:appendix-attenuation-positive-series}
\end{align}
and suppose \(g\) is nondecreasing on the interval containing the five correlations. The constant term cancels in all abstraction differences, so
\begin{align}
g(\rho_{1s})-g(\rho_0)
&=
\sum_{p=1}^{\infty}a_p N_p,
\\
\bar g-g(\rho_{1c})
&=
\sum_{p=1}^{\infty}a_p \Delta_p
\end{align}
For the two maps verified below, \(a_1>0\). Since \(\Delta_1=\bar\rho-\rho_{1c}=2Y>0\), this implies
\begin{align}
\bar g-g(\rho_{1c})>0
\end{align}
Using \autoref{eq:appendix-attenuation-power-bound},
\begin{align}
\left|g(\rho_{1s})-g(\rho_0)\right|
&\le
\sum_{p=1}^{\infty}a_p |N_p|
\\
&\le
|\alpha_Q|
\sum_{p=1}^{\infty}a_p\Delta_p
\\
&=
|\alpha_Q|\,\bigl(\bar g-g(\rho_{1c})\bigr)
\end{align}
Because \(g\) is nondecreasing, the sign of \(g(\rho_{1s})-g(\rho_0)\) agrees with the sign of \(\rho_{1s}-\rho_0=2V\), hence agrees with the sign of \(\alpha_Q\). Dividing by the positive denominator \(\bar g-g(\rho_{1c})\) gives
\begin{align}
\operatorname{sign}(\alpha_K)
=
\operatorname{sign}(\alpha_Q),
\qquad
|\alpha_K|
\le
|\alpha_Q|
\label{eq:appendix-attenuation-general-bound}
\end{align}
Equivalently, when \(\alpha_Q\neq0\),
\begin{align}
0
\le
\frac{\alpha_K}{\alpha_Q}
\le
1
\end{align}
When \(\alpha_Q=0\), monotonicity gives \(\alpha_K=0\). For the differentiable maps below, the secant expression in \autoref{eq:appendix-attenuation-factor} is continuous as \(V\to0\), so the same bound \(0\le \mathcal{A}_g(\mathbf{Q})\le1\) follows by continuity from the nonzero-\(\alpha_Q\) case.

\paragraph{Part 2: explicit verification for erf.}
For
\begin{align}
\phi_\beta(z)=\mathrm{erf}(\beta z),
\qquad
\beta>0,
\end{align}
\autoref{appendix:2fs-preserved} derives the NNGP map
\begin{align}
\psi_\beta(q,q_d)
=
\frac{2}{\pi}
\arcsin\left(
\frac{2\beta^2 q}{1+2\beta^2 q_d}
\right)
\end{align}
Therefore the normalized scalar map is
\begin{align}
g_\beta(\rho)
=
\frac{\arcsin(c_\beta\rho)}{\arcsin(c_\beta)},
\qquad
c_\beta
\coloneqq
\frac{2\beta^2 q_d}{1+2\beta^2 q_d}
\in(0,1)
\label{eq:appendix-attenuation-erf-map}
\end{align}
It is nondecreasing, indeed strictly increasing, because
\begin{align}
g_\beta'(\rho)
=
\frac{c_\beta}
{\arcsin(c_\beta)\sqrt{1-c_\beta^2\rho^2}}
>
0
\end{align}
It also has a power series with nonnegative (nonconstant) coefficients:
\begin{align}
g_\beta(\rho)
=
\frac{1}{\arcsin(c_\beta)}
\sum_{j=0}^{\infty}
\frac{(2j)!}{4^j(j!)^2(2j+1)}
 c_\beta^{2j+1}\rho^{2j+1}
\end{align}
Thus \(g_\beta\) satisfies the conditions in Part 1, and hence
\begin{align}
\alpha_{\mathbf{K}_{\phi_\beta}}
=
\mathcal{A}_\beta(\mathbf{Q})\alpha_Q,
\qquad
0\le \mathcal{A}_\beta(\mathbf{Q})\le1
\end{align}
This proves the \(\mathrm{erf}\) case of the theorem. 


\paragraph{Part 3: explicit verification for leaky ReLU.}
For
\begin{align}
\phi_\omega(z)=\max\{z,\omega z\},
\qquad
\omega\in[0,1],
\end{align}
Recall in \autoref{appendix:2fs-preserved} we derived the unnormalized NNGP map:
\begin{align}
\psi_\omega(q,q_d)
=
\frac{1+\omega^2}{2}q
+
\frac{(1-\omega)^2}{2\pi}
\left[
\sqrt{q_d^2-q^2}
-
q\arccos\left(\frac{q}{q_d}\right)
\right]
\end{align}
Since \(\psi_\omega(q_d,q_d)=\frac{1+\omega^2}{2}q_d\), the normalized map is
\begin{align}
g_\omega(\rho)
&=
\rho+
\frac{(1-\omega)^2}{\pi(1+\omega^2)}
\left[
\sqrt{1-\rho^2}
-
\rho\arccos\rho
\right]
\\
&=
\frac{2\omega}{1+\omega^2}\rho
+
\frac{(1-\omega)^2}{1+\omega^2}g_{\mathrm{ReLU}}(\rho),
\label{eq:appendix-attenuation-lrelu-map}
\end{align}
where
\begin{align}
g_{\mathrm{ReLU}}(\rho)
=
\frac{1}{\pi}
\left[
\sqrt{1-\rho^2}
+
(\pi-\arccos\rho)\rho
\right]
\end{align}
The pure-ReLU normalized map satisfies
\begin{align}
g_{\mathrm{ReLU}}'(\rho)
=
1-\frac{\arccos\rho}{\pi}
=
\frac{1}{2}+\frac{1}{\pi}\arcsin\rho
\end{align}
Using
\begin{align}
\arcsin\rho
=
\sum_{j=0}^{\infty}
\frac{(2j)!}{4^j(j!)^2(2j+1)}\rho^{2j+1},
\end{align}
and integrating from \(0\), with \(g_{\mathrm{ReLU}}(0)=1/\pi\), gives
\begin{align}
g_{\mathrm{ReLU}}(\rho)
=
\frac{1}{\pi}
+
\frac{1}{2}\rho
+
\frac{1}{\pi}
\sum_{j=0}^{\infty}
\frac{(2j)!}{4^j(j!)^2(2j+1)(2j+2)}\rho^{2j+2}
\end{align}
Thus the nonconstant power series coefficients of \(g_{\mathrm{ReLU}}\) are also nonnegative. By \autoref{eq:appendix-attenuation-lrelu-map}, every power-series coefficient of \(g_\omega\) is also nonnegative for every \(\omega\in[0,1]\). Moreover,
\begin{align}
g_\omega'(\rho)
=
\frac{2\omega}{1+\omega^2}
+
\frac{(1-\omega)^2}{1+\omega^2}
\left(
1-\frac{\arccos\rho}{\pi}
\right)
\ge0
\end{align}
for \(\rho\in[-1,1]\). Therefore \(g_\omega\) satisfies the conditions in Part 1, and hence
\begin{align}
\alpha_{\mathbf{K}_{\phi_\omega}}
=
\mathcal{A}_\omega(\mathbf{Q})\alpha_Q,
\qquad
0\le \mathcal{A}_\omega(\mathbf{Q})\le1
\end{align}
This proves the leaky-ReLU case and completes the proof of \autoref{theorem:attenuation-law-formal}. Finally, leaky ReLU gives an additional intuition. Since
\begin{align}
\phi_\omega(z)=\omega z+(1-\omega)z_+,
\end{align}
and the mixed Gaussian terms satisfy \(\mathbb{E}[u v_+]=\mathbb{E}[u_+v]=q/2\), the unnormalized kernel obeys
\begin{align}
\mathbf{K}_\omega
=
\omega \mathbf{Q}+(1-\omega)^2 \mathbf{K}_{\mathrm{ReLU}},
\end{align}
where \(\mathbf{K}_{\mathrm{ReLU}}\) is the pure-ReLU feature kernel generated from the same \(\mathbf{Q}\). Since the numerator and denominator in \autoref{eq:appendix-attenuation-alpha-entry} are both linear in kernel entries,
\begin{align}
\alpha_{\mathbf{K}_\omega}
=
\frac{
\omega B_{\mathbf{Q}}\alpha_Q
+
(1-\omega)^2B_{\mathrm{ReLU}}\alpha_{\mathbf{K}_{\mathrm{ReLU}}}
}{
\omega B_{\mathbf{Q}}
+
(1-\omega)^2B_{\mathrm{ReLU}}
},
\end{align}
where
\begin{align}
B_{\mathbf{Q}}
\coloneqq
\bar q-q_{1c},
\qquad
B_{\mathrm{ReLU}}
\coloneqq
\bar k_{\mathrm{ReLU}}-k_{\mathrm{ReLU},1c}
\end{align}
Thus leaky ReLU is a denominator-weighted interpolation between the linear preactivation abstraction and pure-ReLU feature abstraction. In particular, as \(\omega\to1\),
\begin{align}
g_\omega(\rho)=\rho+O((1-\omega)^2),
\qquad
\alpha_{\mathbf{K}_{\phi_\omega}}
=
\alpha_Q+O((1-\omega)^2),
\qquad
\mathcal{A}_\omega(\mathbf{Q})=1+O((1-\omega)^2),
\end{align}
again assuming the abstraction denominator is bounded away from zero.
\end{proof}

\paragraph{When would the attenuation law hold for any general nonlinearity?} Let
\begin{align}
g_\phi^{(q_d)}(\rho)
\coloneqq
\frac{
\mathbb{E}_{(u,v)}[\phi(u)\phi(v)]
}{
\mathbb{E}_{u}[\phi(u)^2]
},
\qquad
\operatorname{Var}(u)=\operatorname{Var}(v)=q_d,
\qquad
\operatorname{Corr}(u,v)=\rho
\end{align}
For any nonlinearity whose infinite-width feature kernel is \(k_\mu=k_d g_\phi^{(q_d)}(\rho_\mu)\), the exact factor \autoref{eq:appendix-attenuation-factor} remains valid. Thus a locally sufficient condition is simply the secant-slope inequality
\begin{align}
0
\le
s_{g_\phi}(\rho_{1s},\rho_0)
\le
w_d s_{g_\phi}(1,\rho_{1c})+w_2 s_{g_\phi}(\rho_2,\rho_{1c}),
\label{eq:appendix-attenuation-local-secant-condition}
\end{align}
with positive feature denominator. A more interpretable sufficient condition is obtained when the correlations are naturally ordered, for example:
\begin{align}
\rho_0 \le \rho_{1s}, \rho_{1c} \le \rho_2 \le 1
\end{align}
and the weights in \autoref{eq:appendix-attenuation-weights} are nonnegative. In that case, if \(g_\phi\) is increasing and convex on the interval containing these correlations, then secant slopes increase as their intervals move to the right. Hence the numerator secant in \autoref{eq:appendix-attenuation-factor} is no larger than either denominator secant, which implies \(0\le\mathcal{A}_\phi(\mathbf{Q})\le1\). A stronger global sufficient condition, valid for all PSD 2FS kernels rather than just the naturally ordered ones, is the condition:
\begin{align}
g_\phi^{(q_d)}(\rho)
=
a_0+
\sum_{p=1}^{\infty}a_p\rho^p,
\qquad
a_p\ge0,
\qquad
\text{and}
\qquad
g_\phi^{(q_d)}\text{ is nondecreasing on }[-1,1]
\label{eq:appendix-attenuation-global-sufficient-condition}
\end{align}
\clearpage
\section{Experimental details}

\subsection{Experimental details for theory-simulation agreement figures} \label{appendix:theory-figures}

\subsubsection{\texorpdfstring{\autoref{fig:2fs-ansatz}D: synthetic ReLU network and probe generalization}{Figure 2D: synthetic ReLU network and probe generalization}} \label{appendix:synthetic-details}

For this figure we create a synthetic task reflecting the balanced \(2\times2\) factorial design, where the two factors are shape and color. This gives us four fine classes: \(\mathrm{RS},\mathrm{BS},\mathrm{RC},\mathrm{BC}\). Each example is generated from latent coordinates for shape, color, and interaction. The figure uses input dimension \(16\), one nuisance coordinate, latent strengths \(3.0\) for shape, \(2.0\) for color, and \(2.0\) for the interaction. To simulate realistic noisy data, the latent coordinates are mixed by a random matrix with normalized columns, corrupted with additive Gaussian noise of standard deviation \(2.0\), passed through a pointwise \(\tanh\) function, and then corrupted again with additive Gaussian noise of standard deviation \(1.0\). For each seed we generate 1000 training examples per class using the above procedure.

The model is a feedforward ReLU network with a bias-free linear readout. The network has \(2\) hidden layers with width \(32\), and ReLU after each hidden linear map. We train it to predict the two signed factor labels with mean-squared error using Adam, learning rate \(10^{-3}\), batch size \(512\), for 100 epochs. The plotted curves average over seeds \(0,1,2\), with shaded bands showing the SEM across seeds.

\subsubsection{\texorpdfstring{\autoref{fig:linear-theory}: deep linear networks}{Figure linear-theory: deep linear networks}}

The linear-network simulations use the simplified MSE objective described in \autoref{section:min-model} under \autoref{assumption:readout}. Each run builds a synthetic 2FS problem from the requested eigenvalues, initializes a deep linear network
Unless otherwise stated, the base configuration uses \(n=4\) examples per fine class, and inverse regularization parameter \(\gamma=1\).

Panel A uses input and label kernels are specified directly in the 2FS eigenbasis. The input has
\((\lambda_I,\lambda_S,\lambda_C,\lambda_G)=(0.10,0.16,0.12,0.03)\) and \(\lambda_{SC}=\nu_X\lambda_S\) with \(\nu_X=0.39\). The label kernel has \((\lambda_I,\lambda_S,\lambda_C,\lambda_G)=(0,0.16,0.12,0)\) and \(\lambda_{SC}=\nu_Y\lambda_S\) with \(\nu_Y=0.08\). We run hidden depth \(L=1\) with initialization scales \(10,1,10^{-1},10^{-10}\). The solid curves are the explicit weight-space gradient descent simulations; the dotted curves are the scalar theory predictions; the horizontal reference is the analytic fixed point.

Panel B uses a deep network with hidden depth \(L=8\) and \(\gamma=1\). The input and target kernels are specified in the 2FS eigenbasis. The input kernel has
\((\lambda_I,\lambda_S,\lambda_C,\lambda_{SC},\lambda_G)=(0.10,16.0,13.6,7.2,3.2)\), giving data-layer shape abstraction \(\alpha_S^{(0)}=(16.0-7.2)/(16.0+7.2)=0.3793\). The target kernel has
\((\lambda_I,\lambda_S,\lambda_C,\lambda_{SC},\lambda_G)=(0,16.8,12.0,2.4,0)\), giving target abstraction \(\alpha_S^{(\mathbf{Y})}=0.75\). The PE-aligned balanced initialization sets the final hidden kernel to
\((\lambda_I^{(\mathbf{Q}^{(L)}(0))},\lambda_S^{(\mathbf{Q}^{(L)}(0))},\lambda_C^{(\mathbf{Q}^{(L)}(0))},\lambda_{SC}^{(\mathbf{Q}^{(L)}(0))},\lambda_G^{(\mathbf{Q}^{(L)}(0))})=(0.05,0.32,0.22,0.58,0.08)\), with intermediate hidden-layer initial kernels chosen by the balanced depth factorization, so their mode eigenvalues geometrically interpolate between the input kernel and \(\mathbf{Q}^{(L)}(0)\). The solid layerwise curves are the empirical \(\alpha_S^{(\ell)}(t)\) values for the data layer and hidden layers. The dotted curves are the interpolation-law predictions computed from the data-layer abstraction and the contemporaneous final-layer abstraction. The red dotted line is the target abstraction.

Panel C uses a hidden depth sweep \(L=1,\ldots,40\). The input kernel has \(\lambda_I^{(\boldsymbol{\Sigma}_x)}=\lambda_G^{(\boldsymbol{\Sigma}_x)}=0\), \(\lambda_C^{(\boldsymbol{\Sigma}_x)}/\lambda_S^{(\boldsymbol{\Sigma}_x)}=0.90\), and \(\nu_X=\lambda_{SC}^{(\boldsymbol{\Sigma}_x)}/\lambda_S^{(\boldsymbol{\Sigma}_x)}=0.70\). The label kernel has \(\lambda_I^{(\boldsymbol{\Sigma}_y)}=\lambda_G^{(\boldsymbol{\Sigma}_y)}=0\), \(\lambda_C^{(\boldsymbol{\Sigma}_y)}/\lambda_S^{(\boldsymbol{\Sigma}_y)}=0.75\), and \(\nu_Y=\lambda_{SC}^{(\boldsymbol{\Sigma}_y)}/\lambda_S^{(\boldsymbol{\Sigma}_y)}=0.10\). The initialization uses eigenvalues \((\lambda_I^{(\mathbf{Q}(0))},\lambda_S^{(\mathbf{Q}(0))},\lambda_C^{(\mathbf{Q}(0))},\lambda_{SC}^{(\mathbf{Q}(0))},\lambda_G^{(\mathbf{Q}(0))})=(0.05,0.32,0.22,0.58,0.08)\). 

\subsubsection{\texorpdfstring{\autoref{fig:nonlinear}E--F: leaky-ReLU and erf nonlinear dynamics}{Figure nonlinear E-F: leaky-ReLU and erf nonlinear dynamics}}

For \autoref{fig:nonlinear}E (leaky-ReLU dynamics) we use the class-level \(N=4\) 2FS reduction with Walsh modes \(G,S,C,SC\) and leaky ReLU \(\phi(z)=\max(z,\omega z)\) with \(\omega=0.5\) and \(\gamma=100\). The input and label kernels are signal-balanced with \(\lambda_S=\lambda_C=1\), \(\nu_X=0.6\), \(\nu_Y=0.1\), and \(\lambda_G=0\). The initial kernel is also signal-balanced, with initial scale \(q(0)\coloneqq\lambda_S^{(\mathbf{Q}(0))}=\lambda_C^{(\mathbf{Q}(0))}=0.01\) and initial inverse-SNR \(\nu_{\mathbf{Q},0}\coloneqq\lambda_{SC}^{(\mathbf{Q}(0))}/\lambda_S^{(\mathbf{Q}(0))}=1\). The finite-width simulation uses hidden width \(512\) and empirical covariance \(\mathbf{Q}(0)\) for total time \(10^4\). The theory curve integrates the closed 2FS matrix ODE for \(\mathbf{Q}(t)\) over the same time interval; the red reference line is the analytic terminal fixed point \(\alpha_{\mathbf{Q},\infty}=0.6528\).

For \autoref{fig:nonlinear}F (erf dynamics) we use the row-dynamics simulator and the corresponding closed erf kernel ODE. The base problem has \(n=4\), \(\beta=0.1\), \(\gamma=1000\), \(\nu_X=1.0\), \(\nu_Y=0.1\), unit shape and color signal strengths, and zero \(I\)- and \(G\)-mode eigenvalues for the input and labels. The initial \(\mathbf{Q}(0)\) eigenvalues are \(G=3.0\), \(I=0.2\), \(S=0.5\), \(C=0.5\), and \(SC=0.3\). The finite-width simulation uses width \(2048\), and empirical covariance \(\mathbf{Q}(0)\) for total time \(10^4\). The theory curve integrates the closed 2FS matrix ODE with \texttt{DOP853}. The red reference line is the linear network terminal abstraction result from \autoref{thm:fixed-point}. 

The phase-diagram panels associated with the nonlinear comparison are computed separately. The leaky-ReLU phase diagrams evaluate a \(1000\times1000\) grid over \(\eta_X,\eta_Y\in[0,0.999999]\) for \(\omega=0.2\) using \(\omega=0\) and Brent root finder for the fixed point equation \autoref{eq:relufixedpointequation}. The erf phase diagram uses a \(10\times10\) \(\eta_X,\eta_Y\) grid with \(\beta=1\), \(\gamma=1000\), initial scale \(q(0)=0.1\), \(t_{\max}=10\). 

\subsection{Experimental details for 3dshapes experiments} \label{appendix:3dshapes-experiments}

\subsubsection{Description of task}

All 3dshapes experiments in this subsection use the red/blue cube/sphere subset of the 3dshapes dataset. It contains four fine classes, Red Sphere, Blue Sphere, Red Cube, and Blue Cube, with balanced class counts. The object hue is fixed to the nearest grid value matching red \(0.0\) or blue \(0.7\), and the object shape is matched exactly. Floor hue, wall hue, object scale, and object orientation are left free to vary. Unless otherwise stated, the processed dataset contains 1000 examples per fine class, for 4000 images total, with an 80/20 train/test split. Images use \(64\times64\) resolution and normalized channelwise by \((x-0.5)/0.5\) for the trained convolutional models.

The supervised targets are deterministic 2FS label features built from the balanced class manifest. We use normalized factor vectors for the constant, shape, color, and shape-color interaction modes. In all 3dshapes experiments below, \(\lambda_G=0\), \(\lambda_S=1\), \(\lambda_C=1\), and \(\lambda_I=0\); only the interaction parameter \(\lambda_{SC}=\nu_Y\lambda_S\) changes across experiments.

\subsubsection{\texorpdfstring{\autoref{fig:2fs-ansatz}C: feedforward convolutional network}{Figure 2C: feedforward convolutional network}} \label{appendix:feedforward-convnet}

The panel uses a feedforward convolutional network with 2 convolutional layers and 2 fully connected layers. It applies a \(3\times3\) convolution from 3 channels to 32 channels, ReLU, a second \(3\times3\) convolution from 32 to 32 channels, ReLU, adaptive average pooling, a width-256 fully connected ReLU layer, a width-64 fully connected ReLU layer, and a linear head to the label dimension. The final analyzed hidden representation is the width-64 \texttt{fc2} activation.

The labels use \(\nu_Y=0.125\). The model is initialized with \(\sigma_w^2=0.15\) and trained with AdamW, learning rate \(10^{-3}\), weight decay \(10^{-5}\), batch size \(128\), and a cosine learning-rate schedule with minimum learning rate \(10^{-5}\). Training runs for 100 epochs, and abstraction is measured every epoch. The plotted curve reports the mean and SEM across 5 seeds.

\subsubsection{\texorpdfstring{\autoref{fig:empirical-validation}: Small ResNets}{Figure empirical-validation: SmallResNet sweeps}}

\autoref{fig:empirical-validation} uses small ResNets with a \(3\times3\) convolutional stem with ReLU, followed by \(L\) residual blocks. Each block contains two \(3\times3\) convolutions with a ReLU between them, adds the residual connection, and applies a final ReLU. The representation at layer \(\ell\) is the adaptive-average-pooled post-block activation. A linear readout maps the final pooled representation to the output. 

The initialization-scale sweep in panel B uses \(n=1000\), \(\nu_Y=0.01\), depth \(L=2\), width \(32\), and initialization scales \(\sigma_w^2\in\{0.03,0.3,3.0,30.0\}\). For each scale we train seeds \(0,\ldots,19\) for 50 epochs with AdamW, learning rate \(10^{-3}\), weight decay \(10^{-5}\), batch size \(128\), and cosine decay to \(10^{-5}\). Panel B plots maximum shape abstraction attained during training for each scale averaged over seeds and with SEM error bars. 

The depth sweep in panel C uses \(n=100\), \(\nu_Y=0.25\), width \(32\), initialization scale \(\sigma_w^2=0.1\), and depths \(L\in\{2,4,6,8\}\). For each depth we train seeds \(0,\ldots,19\) for 50 epochs with AdamW, learning rate \(10^{-3}\), weight decay \(10^{-8}\), batch size \(128\), and cosine decay to \(10^{-5}\). Panel C plots the final layer's shape abstraction after the last epoch averaged across seeds with SEM error bars.

Panel D uses the same training run data as panel C, but uses just the depth \(L=8\) runs to plot the abstraction layerwise across \(\ell \in \{1, \ldots,8\}\). The plotted curves average over seeds \(0,\ldots,19\) with SEM error bars.

\subsubsection{\texorpdfstring{\autoref{fig:language-ablation}A: DINOv3 ViT-L/16 local GELU ablation}{Figure language-ablation A: DINOv3 ViT-L/16 local GELU ablation}}

This panel uses a pretrained DINOv3 ViT-L/16, loaded from the official \texttt{facebookresearch/dinov3} repository as \texttt{dinov3\_vitl16}. The run uses all 4000 examples from \texttt{rsbc\_3dshapes\_64}, batch size \(32\), and model input size \(256\). Representations are extracted from the post-block residual stream by forward hooks on each of the 24 transformer blocks. To obtain one embedding vector per sample we pool by averaging the captured token representations over the token dimension.

For the baseline, we run the unmodified DINOv3 model and compute shape abstraction from the post-block representation at each residual layer. For the ablation condition at layer \(\ell\), we temporarily replace the first activation site in \texttt{blocks.<layer>.mlp} with the identity map, run a forward pass, and compute abstraction from the post-block output of the same layer \(\ell\). The patch is local to that forward pass and is restored before the next layer is evaluated. In this model the patched object paths are \texttt{blocks.<layer>.mlp.act} for layers \(0,\ldots,23\). 

\subsection{Experimental details for Gemma 4 ablation experiments} \label{appendix:transformer-app-details}

This section gives details for the Gemma 4 ablation experiments in \autoref{fig:language-ablation}B-C. 

\paragraph{Gemma model and concept set.} \autoref{fig:language-ablation}B--C uses \texttt{google/gemma-4-E2B} from huggingface, loaded as a causal language model with 35 decoder blocks. User-facing layer index \(\ell = 0\) denotes the output of decoder block 0, and \(\ell = 34\) denotes the final residual stream. Following \cite{jiangOriginsLinear24}, We study bilingual concepts in the form of translation directions from one language to another. We use the same four language pairs French-Spanish, French-German, English-French and German-Spanish considered in \cite{jiangOriginsLinear24} as well as an additional 14 English--X translation directions (where X is Arabic, Chinese, Dutch, German, Indonesian, Italian, Japanese, Korean, Persian, Polish, Portuguese, Russian, Spanish, and Turkish). Each concept is instantiated by 80 ordered word pairs comprising common words. See \autoref{tab:gemma-english-spanish-pairs} for the English--Spanish word pairs.

\begin{table}[H]
\centering
\scriptsize
\setlength{\tabcolsep}{3pt}
\begin{tabular}{@{}ll@{\hspace{1.1em}}ll@{\hspace{1.1em}}ll@{\hspace{1.1em}}ll@{}}
\toprule
English & Spanish & English & Spanish & English & Spanish & English & Spanish \\
\midrule
water & agua & rain & lluvia & sister & hermana & one & uno \\
food & comida & snow & nieve & son & hijo & two & dos \\
bread & pan & dog & perro & daughter & hija & three & tres \\
fish & pez & cat & gato & red & rojo & four & cuatro \\
fruit & fruta & horse & caballo & green & verde & five & cinco \\
apple & manzana & bird & p\'ajaro & blue & azul & six & seis \\
house & casa & chicken & pollo & white & blanco & seven & siete \\
car & coche & day & d\'ia & black & negro & eight & ocho \\
money & dinero & night & noche & month & mes & nine & nueve \\
book & libro & big & grande & year & a\~no & ten & diez \\
pen & bol\'igrafo & small & peque\~no & gold & oro & music & m\'usica \\
road & carretera & good & bueno & silver & plata & mathematics & matem\'aticas \\
sun & sol & bad & malo & wood & madera & physics & f\'isica \\
moon & luna & new & nuevo & metal & metal & chemistry & qu\'imica \\
fire & fuego & old & viejo & stone & piedra & biology & biolog\'ia \\
hand & mano & hot & caliente & grass & hierba & engineering & ingenier\'ia \\
eye & ojo & cold & fr\'io & leaf & hoja & history & historia \\
tree & \'{a}rbol & man & hombre & sand & arena & geography & geograf\'ia \\
mountain & monta\~na & woman & mujer & river & r\'io & law & derecho \\
cloud & nube & brother & hermano & beach & playa & english & ingl\'es \\
\bottomrule
\end{tabular}
\caption{The 80 ordered English--Spanish word pairs used for the English--Spanish concept in the Gemma ablation experiments.}
\label{tab:gemma-english-spanish-pairs}
\end{table}

\paragraph{Tokenization and representation extraction.} 
Each word is represented by the hidden state of its final non-padding token. This lets the experiment use both single-token and multi-token lexical items while still attaching each representation to a single isolated word. The final run uses a fixed split seed of 1234 and exactly 80 valid pairs for every concept.

\paragraph{Abstraction metric.} For a concept with ordered pairs \(\{(u_i, v_i)\}_{i=1}^{m}\), let \(\mathbf{h}_i^{(\ell)}\) denote the residual-stream representation of word \(u_i\) at layer \(\ell\), and let \(\mathbf{g}_i^{(\ell)}\) denote the representation of word \(v_i\). We form ordered difference vectors
\begin{align}
 \mathbf{d}_i^{(\ell)} = \mathbf{h}_i^{(\ell)} - \mathbf{g}_i^{(\ell)}
\end{align}
The abstraction score is the mean pairwise cosine across all concept vectors:
\begin{align}
\alpha^{(\ell)} = \frac{2}{m(m-1)} \sum_{1 \le i < j \le m}
\frac{\mathbf{d}_i^{(\ell)} \cdot \mathbf{d}_j^{(\ell)}}{\|\mathbf{d}_i^{(\ell)}\| \, \|\mathbf{d}_j^{(\ell)}\|}
\end{align}
In the final Gemma run, \(m = 80\), so each concept contributes \(\binom{80}{2} = 3160\) pairwise cosine terms.

\paragraph{Local activation ablation.} In the ablated condition, the selected block's feed-forward activation is replaced only for that forward pass. For ordinary pointwise activations this replacement is the identity map. For fused gated activations, we preserve tensor shape and remove only the nonlinear part of the gate, which amounts to replacing the original gated activation by its linearized counterpart. The rest of the network, including attention and all other blocks, is unchanged. For the main-text result in \autoref{fig:language-ablation}B--C we ablate the final block and read out the final residual stream.

\paragraph{Probe generalization error.} To test the practical consequence predicted by \autoref{proposition:ccgp}, we evaluate whether a probe trained on one subset of lexical contexts transfers to unseen pairs from the same concept. Each concept's 80 ordered pairs are split into 40 train pairs and 40 test pairs. Because every concept has the same number of pairs, the same index split is reused across concepts. Each ordered pair contributes two labeled examples, \(u_i \mapsto 0\) and \(v_i \mapsto 1\). Train and test features are centered by the training mean, and we fit a logistic-regression probe with LBFGS, learning rate 1.0, 200 iterations, and tolerance \(10^{-7}\). Reported accuracies are test accuracies on the held-out 40 pairs.

\paragraph{Numerical summary.} In the final-layer diagonal comparison, abstraction increases for all 18 concepts, with concept-wise means moving from 0.207 in baseline to 0.242 after ablation. Held-out probe accuracy increases on average from 0.944 to 0.953, with 17 non-decreases out of 18 concepts. The only decrease occurs for French--Spanish, where accuracy changes slightly from 0.850 to 0.838. These results are the basis for \autoref{fig:language-ablation}B--C.

\subsection{Experimental details for macaque ventral-stream analysis} \label{appendix:majaj-app-details}

This section describes the analysis pipeline used to produce the results in \autoref{fig:neuro} and the associated robustness checks. 

\paragraph{Dataset and public-release caveat.} We use the public Brain-Score release \texttt{MajajHong2015.public}. The public release is close to, but not identical with, the original internal dataset used in the Majaj/Hong studies. In particular, it exposes only variation labels 0 and 3 rather than the explicit low/medium/high variation bins described in the original paper. All results reported here are computed directly from the public release.

\paragraph{Preprocessing.} Raw repeated presentations are first averaged to a single response vector per stimulus. Each neuroid is then z-scored across all 3200 averaged public stimuli before subsetting to the four focal categories. This yields a normalized response matrix that does not rely on blank-screen baselines, which are not exposed in the public release in the same form as in the original analyses.

\paragraph{Conceptual design.} We restrict attention to four categories: Boats, Tables, Fruits, and Animals. These define a \(2 \times 2\) factorial structure in which limbedness and naturalness act as the two abstract variables. Writing \(\boldsymbol{\mu}_{\mathrm{Boat}}, \boldsymbol{\mu}_{\mathrm{Table}}, \boldsymbol{\mu}_{\mathrm{Fruit}}, \boldsymbol{\mu}_{\mathrm{Animal}}\) for category centroids in a given neural population, we define
\begin{align}
 \mathbf{u}_{\mathrm{limbed},1} = \boldsymbol{\mu}_{\mathrm{Table}} - \boldsymbol{\mu}_{\mathrm{Boat}},
 \qquad
 \mathbf{u}_{\mathrm{limbed},2} = \boldsymbol{\mu}_{\mathrm{Animal}} - \boldsymbol{\mu}_{\mathrm{Fruit}},
\end{align}
\begin{align}
 \mathbf{u}_{\mathrm{natural},1} = \boldsymbol{\mu}_{\mathrm{Fruit}} - \boldsymbol{\mu}_{\mathrm{Boat}},
 \qquad
 \mathbf{u}_{\mathrm{natural},2} = \boldsymbol{\mu}_{\mathrm{Animal}} - \boldsymbol{\mu}_{\mathrm{Table}}
\end{align}
The abstraction score for each concept is the cosine similarity between its two context-specific vectors. The code variable \texttt{is\_organic} corresponds to what we call \emph{naturalness} in the paper.

\paragraph{Balanced centroid construction.} The public release contains 8 objects per focal category. In each bootstrap resample, we draw 8 objects per category with replacement. Within each sampled object and variation bin, images are resampled with replacement, averaged within the cell, averaged equally across the available variation bins for that object, and then averaged equally across sampled objects to form the category centroid. This hierarchical averaging ensures that the abstraction score is not dominated by categories, objects, or variation bins with more images.

\paragraph{Site matching and bootstrap design.} The public release contains 168 IT sites (58 from Chabo and 110 from Tito) and 88 V4 sites (70 from Chabo and 18 from Tito). Our primary analysis is the pooled site-matched comparison used in \autoref{fig:neuro}. On each bootstrap iteration, we retain all 88 V4 sites and sample 88 IT sites without replacement to remove the trivial dimensionality advantage that would arise from comparing 168 IT dimensions to only 88 V4 dimensions. The main figure uses 500 bootstrap resamples with random seed 0. In each resample, we jointly redraw objects, images, and the matched IT site subset.

\paragraph{Additional metrics and controls.} Besides the raw cosine score, the pipeline also records vector magnitude, a split-half cross-validated cosine, a permutation null obtained by randomly permuting the category labels, and a context-transfer decoder sanity check. The decoder control mirrors \autoref{proposition:ccgp}: for each concept, we train a linear direction in one context and test it on the opposite context, then average the two transfer directions.

\paragraph{Robustness summary.} The pooled site-matched IT--V4 gap reported in the main text is \(0.676\) for limbedness and \(0.682\) for naturalness, with both 95\% percentile bootstrap intervals excluding zero. The descriptive all-sites comparison gives essentially the same means and intervals. The Chabo-only matched analysis is even stronger in the same direction, while the Tito-only analysis is directionally similar but noisier because the public release contains only 18 Tito V4 sites. The permutation null is centered near zero in both areas, and the transfer-decoder control shows the same ordering as the cosine metric: in the pooled site-matched analysis, transfer accuracy is 0.599 versus 0.371 for limbedness and 0.655 versus 0.389 for naturalness in IT and V4 respectively.

\paragraph{Abstraction vs factorization.} Our findings are consistent with prior work showing greater factorization in IT than V4 \cite{lindseyFactorizedVisual24}. However, factorization and abstraction measure different things: factorization measures orthogonality between different concept's subspaces (sometimes called disentanglement), while abstraction measures the alignment of counterfactual concept vectors. As discussed in \cite{lindseyFactorizedVisual24}, the two measures are related in the sense that a fully factorized representation should be highly abstract. 

\section{Other}

\subsection{Compute resources} \label{appendix:compute-resources}

Neural network training experiments were run on a cluster with 8 NVIDIA H100 80GB HBM3 GPUs, 192 logical CPU cores from two Intel Xeon Platinum 8468 sockets, and approximately 1 TB of system RAM. Individual training runs used a single H100 GPU. All other analysis was run locally on a MacBook Air M2 with 8 GB of RAM. All code was implemented in Python 3.10-3.12. All neural network training experiments used stable 2.x versions of PyTorch.

\subsection{Assets and licenses} \label{appendix:assets-licenses}

We use the 3dshapes dataset under Apache-2.0, DINOv3 under the DINOv3 License, Gemma 4 E2B under Google’s Gemma usage terms, and the public Brain-Score MajajHong2015.public assembly under the public Brain-Score/DiCarlo Lab access terms. 


\end{document}